\PassOptionsToPackage{capitalize}{cleveref}
\documentclass{rtp_tech_report}

\usepackage{amsmath,amsfonts,bm}

\def\eqref#1{equation~\ref{#1}}

\def\1{\bm{1}}

\DeclareMathAlphabet{\mathsfit}{\encodingdefault}{\sfdefault}{m}{sl}
\SetMathAlphabet{\mathsfit}{bold}{\encodingdefault}{\sfdefault}{bx}{n}

\usepackage{url}
\usepackage{wrapfig}
\usepackage{float}
\usepackage{needspace}
\usepackage{tikz}
\usetikzlibrary{arrows.meta,positioning,calc,shapes.geometric,decorations.pathreplacing}
\usepackage{pgfplots}
\pgfplotsset{compat=1.18}
\usepgfplotslibrary{groupplots,colormaps}

\crefname{assumption}{Assumption}{Assumptions}
\providecommand{\swift}{\textnormal{\textsc{Swift}}}
\crefname{equation}{Eq.}{Eqs.}
\Crefname{equation}{Eq.}{Eqs.}
\crefname{algorithm}{Algorithm}{Algorithms}
\crefname{appendix}{Appendix}{Appendices}
\Crefname{appendix}{Appendix}{Appendices}
\crefname{section}{Section}{Sections}
\Crefname{section}{Section}{Sections}
\DeclareRobustCommand{\swiftfullref}[2]{\hyperref[#2]{#1~\ref*{#2}}}
\DeclareRobustCommand{\swiftpanelref}[2]{\hyperref[#1]{Figure~\ref*{#1}(#2)}}

\definecolor{cAGT}{HTML}{2258A6}   %
\definecolor{cTwo}{HTML}{E08A1E}   %
\definecolor{cThree}{HTML}{B8352E} %
\definecolor{cBase}{HTML}{6E6E6E}  %
\definecolor{cSeg}{HTML}{9DB4D6}   %
\definecolor{cGrid}{HTML}{DDDDDD}
\definecolor{cCard}{HTML}{FBFBFC}   %
\definecolor{cKeep}{HTML}{2F6B4F}   %
\definecolor{cAxis}{HTML}{444444}

\newcommand{\SG}{SWIFT\xspace}

\makeatletter
\@removefromreset{theorem}{section}
\makeatother
\renewcommand{\thetheorem}{\arabic{theorem}}
\renewcommand{\theHtheorem}{\arabic{theorem}}
\newcommand{\swiftboxbreakable}{false}
\tcolorboxenvironment{theorem}{
  enhanced, breakable=\swiftboxbreakable,
  colback=metabg, colframe=rtpblue!60!black,
  boxrule=0.5pt, arc=2pt,
  left=6pt, right=6pt, top=4pt, bottom=4pt,
  before skip=8pt, after skip=8pt
}

\tcolorboxenvironment{corollary}{
  enhanced, breakable,
  colback=metabg, colframe=rtpblue!60!black,
  boxrule=0.5pt, arc=2pt,
  left=6pt, right=6pt, top=4pt, bottom=4pt,
  before skip=8pt, after skip=8pt
}

\tcolorboxenvironment{proposition}{
  enhanced, breakable,
  colback=metabg, colframe=rtpblue!60!black,
  boxrule=0.5pt, arc=2pt,
  left=6pt, right=6pt, top=4pt, bottom=4pt,
  before skip=8pt, after skip=8pt
}

\theoremstyle{definition}
\makeatletter

\let\c@assumption\relax
\makeatother
\newtheorem{assumption}{Assumption}
\theoremstyle{plain}

\tcolorboxenvironment{assumption}{
  enhanced, breakable,
  colback=metabg, colframe=rtpblue!60!black,
  boxrule=0.5pt, arc=2pt,
  left=6pt, right=6pt, top=4pt, bottom=4pt,
  before skip=8pt, after skip=8pt
}

\title{\textbf{SWIFT}: \textbf{S}elf-guided \textbf{W}eak-to-strong 
d\textbf{IF}fusion models via \textbf{T}arget extrapolation}

\author{%
\parbox{\dimexpr\textwidth-1pt\relax}{\centering
\textbf{Yanjin He}$^{1,2,\sharp}$, Cuifeng Shen$^{2}$, Wanzhou Lei$^{2,3,\sharp}$, Maohua Li$^{2,4,\sharp}$\\[1mm]
Tao Lan$^{2}$, Kan Liu$^{2}$, Hanlin Tang$^{\dagger,2}$, Qing Qu$^{1}$
}}
\affiliation{%
\parbox{\dimexpr\textwidth-1pt\relax}{\centering\small
$^1$University of Michigan, Ann Arbor \quad $^2$Alibaba Group\\
$^3$University of California, Berkeley \quad $^4$Nanjing University
}}
\renewcommand{\contributionlist}{\vspace{-4mm}}
\newcommand{\authorfootnotes}{%
  \begingroup
  \renewcommand{\thefootnote}{}%
  \footnotetext{\textsuperscript{$\dagger$}Project lead \quad $^\sharp$Work done during internship at Alibaba}%
  \endgroup
}
\checkdata[E-mail]{yanjinhe@umich.edu, wanghaisheng.whs@alibaba-inc.com}
\hypersetup{
  pdfauthor={Yanjin He, Cuifeng Shen, Wanzhou Lei, Maohua Li, Tao Lan, Kan Liu, Hanlin Tang, Qing Qu},
  pdftitle={SWIFT: Self-guided Weak-to-strong dIFfusion models via Target extrapolation}
}

\newif\ifswiftshowtemplatelogos
\swiftshowtemplatelogostrue
\ifswiftshowtemplatelogos\else
  \fancypagestyle{firststyle}{\fancyhf{}\fancyfoot[C]{\thepage}}
\fi

\abstract{

 Training high-quality diffusion models is computationally expensive, and inference-time guidance often adds further sampling cost. In this work, we introduce \swift{}, a simple yet effective self-guided training method that uses a frozen weak denoiser to improve a stronger model without adding inference overhead. For each noisy input, \swift{} extrapolates the clean regression target away from the weak denoiser's prediction, amplifying components that remain underfit. The weak guide can be an early training checkpoint, requires no backpropagation during main-model training, and is discarded afterward. The resulting model retains the original sampler and supports both conditional and unconditional generation without label dropout or an unconditional branch. We explain this mechanism through a spectral analysis of finite-time population gradient flow in an RKHS. With an early-stopped weak guide, we prove that a nonempty range of extrapolation strengths yields strictly lower posterior-mean error than matched standard training, despite shifting the regression target away from the posterior mean. Experiments on CIFAR-10 and ImageNet-256 demonstrate improvements across model scales and representations. In one ImageNet-256 configuration, \swift{} surpasses the baseline's best FID using roughly half the total training compute, including guide preparation and training overhead. On RAEv2, it achieves FID $1.15$ without inference-time guidance and $\textbf{0.99}$ with guidance. These results establish target extrapolation as an effective approach to weak-to-strong diffusion training with efficient single-model inference.

}

\begin{document}
\newgeometry{top=1.95cm,bottom=2.1cm,left=2.5cm,right=2.5cm,headheight=48pt,headsep=-6.53583pt}
\maketitle
\authorfootnotes
\par
\noindent\begin{minipage}{\textwidth}
    \captionsetup{type=figure,skip=5pt}
    \centering
    \begin{tikzpicture}
\begin{axis}[width=0.5\linewidth, height=0.36\linewidth, tick label style={font=\scriptsize, /pgf/number format/assume math mode=true},
  label style={font=\scriptsize}, title style={font=\scriptsize\bfseries, yshift=-3pt},
  xmajorgrids, ymajorgrids, grid style={cGrid, line width=0.3pt},
  axis line style={cAxis!65, line width=0.45pt},
  tick style={cAxis!65, line width=0.45pt}, tick align=outside,
  xtick style={line width=0.4pt}, ytick style={line width=0.4pt},
  xmode=log, ymode=log, xmin=1.00, xmax=11.00, ymin=0.95, ymax=2.45,
  xtick={1,1.5,2,3,5,10}, xticklabels={1,1.5,2,3,5,10}, ytick={1,1.2,1.5,2}, yticklabels={1.0,1.2,1.5,2.0},
  minor xtick={}, minor ytick={},
  title={ImageNet-256, published results}, xlabel={FID-50k (no guidance)}, ylabel={FID-50k (with guidance)}]
\addplot[only marks, mark=*, mark size=1.6pt, mark options={fill=cTwo!65, draw=white, line width=0.3pt}] coordinates {(9.62,2.27)};
\node[font=\tiny, text=cAxis!90, anchor=east, inner sep=0.4pt, xshift=-2.5pt, yshift=0.0pt] at (axis cs:9.62,2.27) {DiT-XL/2};
\addplot[only marks, mark=*, mark size=1.6pt, mark options={fill=cTwo!65, draw=white, line width=0.3pt}] coordinates {(8.30,2.06)};
\node[font=\tiny, text=cAxis!90, anchor=east, inner sep=0.4pt, xshift=-2.5pt, yshift=0.0pt] at (axis cs:8.30,2.06) {SiT-XL/2};
\addplot[only marks, mark=*, mark size=1.6pt, mark options={fill=cTwo!65, draw=white, line width=0.3pt}] coordinates {(6.23,1.79)};
\node[font=\tiny, text=cAxis!90, anchor=east, inner sep=0.4pt, xshift=-2.5pt, yshift=0.0pt] at (axis cs:6.23,1.79) {MDT};
\addplot[only marks, mark=*, mark size=1.6pt, mark options={fill=cTwo!65, draw=white, line width=0.3pt}] coordinates {(5.90,1.42)};
\node[font=\tiny, text=cAxis!90, anchor=east, inner sep=0.4pt, xshift=-2.5pt, yshift=0.0pt] at (axis cs:5.90,1.42) {REPA};
\addplot[only marks, mark=*, mark size=1.6pt, mark options={fill=cTwo!65, draw=white, line width=0.3pt}] coordinates {(6.27,1.26)};
\node[font=\tiny, text=cAxis!90, anchor=east, inner sep=0.4pt, xshift=-2.5pt, yshift=0.0pt] at (axis cs:6.27,1.26) {DDT};
\addplot[only marks, mark=square*, mark size=1.4pt, mark options={fill=cTwo!65, draw=white, line width=0.3pt}] coordinates {(2.35,1.55)};
\node[font=\tiny, text=cAxis!90, anchor=west, inner sep=0.4pt, xshift=2.5pt, yshift=0.0pt] at (axis cs:2.35,1.55) {MAR-H};
\addplot[only marks, mark=*, mark size=1.6pt, mark options={fill=cTwo!65, draw=white, line width=0.3pt}] coordinates {(2.17,1.35)};
\node[font=\tiny, text=cAxis!90, anchor=west, inner sep=0.4pt, xshift=2.5pt, yshift=0.0pt] at (axis cs:2.17,1.35) {LightningDiT};
\addplot[only marks, mark=*, mark size=1.6pt, mark options={fill=cTwo!65, draw=white, line width=0.3pt}] coordinates {(1.80,1.36)};
\node[font=\tiny, text=cAxis!90, anchor=south, inner sep=0.4pt, xshift=0.0pt, yshift=2.5pt] at (axis cs:1.80,1.36) {REG};
\addplot[only marks, mark=*, mark size=1.6pt, mark options={fill=cTwo!65, draw=white, line width=0.3pt}] coordinates {(2.38,1.04)};
\node[font=\tiny, text=cAxis!90, anchor=west, inner sep=0.4pt, xshift=2.5pt, yshift=0.0pt] at (axis cs:2.38,1.04) {SFD};
\addplot[only marks, mark=*, mark size=1.6pt, mark options={fill=cTwo!65, draw=white, line width=0.3pt}] coordinates {(1.69,1.12)};
\node[font=\tiny, text=cAxis!90, anchor=west, inner sep=0.4pt, xshift=2.5pt, yshift=0.0pt] at (axis cs:1.69,1.12) {REPA-E};
\addplot[only marks, mark=*, mark size=1.6pt, mark options={fill=cTwo!65, draw=white, line width=0.3pt}] coordinates {(1.51,1.13)};
\node[font=\tiny, text=cAxis!90, anchor=south west, inner sep=0.4pt, xshift=0.5pt, yshift=2.0pt] at (axis cs:1.51,1.13) {RAE};
\addplot[only marks, mark=*, mark size=1.6pt, mark options={fill=cTwo!65, draw=white, line width=0.3pt}] coordinates {(1.31,1.13)};
\node[font=\tiny, text=cAxis!90, anchor=east, inner sep=0.4pt, xshift=-2.5pt, yshift=0.0pt] at (axis cs:1.31,1.13) {GAE};
\addplot[only marks, mark=*, mark size=1.6pt, mark options={fill=cTwo!65, draw=white, line width=0.3pt}] coordinates {(1.65,1.06)};
\node[font=\tiny, text=cAxis!90, anchor=west, inner sep=0.4pt, xshift=2.5pt, yshift=-1.0pt] at (axis cs:1.65,1.06) {RAEv2};
\addplot[only marks, mark=*, mark size=1.6pt, mark options={fill=cTwo!65, draw=white, line width=0.3pt}] coordinates {(1.43,1.03)};
\node[font=\tiny, text=cAxis!90, anchor=south east, inner sep=0.4pt, xshift=-2.2pt, yshift=1.0pt] at (axis cs:1.43,1.03) {PAE};
\addplot[only marks, mark=square*, mark size=1.4pt, mark options={fill=cTwo!65, draw=white, line width=0.3pt}] coordinates {(1.42,0.99)};
\node[font=\tiny, text=cAxis!90, anchor=west, inner sep=0.4pt, xshift=2.5pt, yshift=-1.0pt] at (axis cs:1.42,0.99) {BAR-L};
\addplot[only marks, mark=*, mark size=1.6pt, mark options={fill=cTwo!65, draw=white, line width=0.3pt}] coordinates {(1.34,1.19)};
\node[font=\tiny, text=cAxis!90, anchor=south west, inner sep=0.4pt, xshift=-8.0pt, yshift=2.0pt] at (axis cs:1.34,1.19) {LightningDiT+IG};
\node[star, star points=5, star point ratio=2.3, fill=cAGT, draw=white, line width=0.4pt, minimum size=9pt, inner sep=0pt] at (axis cs:1.15,0.99) {};
\node[star, star points=5, star point ratio=2.3, fill=cAGT, draw=white, line width=0.3pt, minimum size=7pt, inner sep=0pt] at (rel axis cs:0.035,0.905) {};
\node[font=\scriptsize\bfseries, text=cAGT!80!black, anchor=west, inner sep=0.5pt] at (rel axis cs:0.060,0.905) {\SG{} (ours)};
\node[circle, fill=cTwo!65, draw=white, line width=0.3pt, minimum size=3.2pt, inner sep=0pt] at (rel axis cs:0.035,0.820) {};
\node[font=\tiny, text=cAxis!90, anchor=west, inner sep=0.5pt] at (rel axis cs:0.060,0.820) {diffusion};
\node[rectangle, fill=cTwo!65, draw=white, line width=0.3pt, minimum size=2.8pt, inner sep=0pt] at (rel axis cs:0.035,0.735) {};
\node[font=\tiny, text=cAxis!90, anchor=west, inner sep=0.5pt] at (rel axis cs:0.060,0.735) {autoregressive};
\end{axis}
\end{tikzpicture}
\hfill
    \begin{tikzpicture}
\begin{axis}[width=0.5\linewidth, height=0.36\linewidth, tick label style={font=\scriptsize, /pgf/number format/assume math mode=true},
  label style={font=\scriptsize}, title style={font=\scriptsize\bfseries, yshift=-3pt},
  ymajorgrids, grid style={cGrid, line width=0.3pt},
  axis line style={cAxis!65, line width=0.45pt},
  tick style={cAxis!65, line width=0.45pt}, tick align=outside,
  xtick style={line width=0.4pt}, ytick style={line width=0.4pt},
  every axis plot/.append style={line width=0.9pt, mark size=1.25pt,
      mark options={line width=0.5pt}},
  legend style={font=\tiny, draw=none, fill=white, fill opacity=0.9, text opacity=1,
      row sep=-2.5pt, inner xsep=3pt, inner ysep=2pt},
  legend cell align=left, legend image post style={line width=0.9pt}, title={ImageNet-256, DINOv2-B (615M), w/o guidance}, xlabel={training epoch}, ylabel={FID-50k (no guidance)}, ymin=1.30, ymax=2.22, xmin=15, xmax=330, xtick={40,80,120,160,200,240,280,320}, ytick={1.4,1.6,1.8,2.0,2.2}, legend style={at={(0.975,0.95)}, anchor=north east}]
\addplot[only marks, mark=*, mark size=0.7pt, cAGT!35, forget plot] coordinates {(20,3.5170) (40,2.3183) (60,1.9953) (80,1.7677) (85,1.6996) (90,1.6953) (95,1.6538) (100,1.5986) (105,1.5800) (110,1.5731) (115,1.5966) (120,1.4982) (125,1.5237) (130,1.5240) (135,1.5077) (140,1.5076) (145,1.4841) (150,1.5174) (155,1.4889) (160,1.4550) (165,1.4740) (170,1.4540) (175,1.4335) (180,1.4924) (185,1.4206) (190,1.4655) (195,1.3961) (200,1.4394) (205,1.3797) (210,1.3992) (215,1.4157) (220,1.4137) (225,1.4010) (230,1.3891) (235,1.3252) (240,1.3766)};
\addplot[only marks, mark=*, mark size=0.7pt, cTwo!35, forget plot] coordinates {(20,3.6473) (40,2.7156) (60,2.3012) (80,2.1297) (85,2.0699) (90,2.0603) (95,1.9673) (100,1.9537) (105,1.9463) (110,1.8963) (115,1.9009) (120,1.8513) (125,1.9213) (130,1.8331) (135,1.7658) (140,1.8685) (145,1.8172) (150,1.7634) (155,1.8021) (160,1.7606) (165,1.7469) (170,1.7427) (175,1.7039) (180,1.7359) (185,1.6663) (190,1.7156) (195,1.7120) (200,1.6596) (205,1.6709) (210,1.6528) (215,1.6860) (220,1.6632) (225,1.6243) (230,1.6284) (235,1.6086) (240,1.6023) (245,1.5890) (250,1.6062) (255,1.5688) (260,1.5575) (265,1.6191) (270,1.5940) (275,1.5668) (280,1.5910) (285,1.5969) (290,1.5732) (295,1.5878) (300,1.5475) (305,1.5619) (310,1.5301) (315,1.5626) (320,1.5302)};
\addplot[cTwo, line width=1.1pt] coordinates {(20,3.6473) (21,3.5882) (22,3.5300) (23,3.4727) (24,3.4164) (25,3.3612) (26,3.3072) (27,3.2544) (28,3.2030) (29,3.1529) (30,3.1044) (31,3.0573) (32,3.0119) (33,2.9682) (34,2.9262) (35,2.8861) (36,2.8478) (37,2.8116) (38,2.7774) (39,2.7454) (40,2.7156) (41,2.6873) (42,2.6597) (43,2.6329) (44,2.6068) (45,2.5816) (46,2.5570) (47,2.5333) (48,2.5104) (49,2.4883) (50,2.4670) (51,2.4466) (52,2.4269) (53,2.4082) (54,2.3902) (55,2.3732) (56,2.3570) (57,2.3417) (58,2.3273) (59,2.3138) (60,2.3012) (61,2.2894) (62,2.2784) (63,2.2679) (64,2.2581) (65,2.2487) (66,2.2398) (67,2.2313) (68,2.2232) (69,2.2153) (70,2.2076) (71,2.2001) (72,2.1927) (73,2.1853) (74,2.1779) (75,2.1704) (76,2.1628) (77,2.1550) (78,2.1469) (79,2.1385) (80,2.1297) (81,2.1205) (82,2.1111) (83,2.1014) (84,2.0916) (85,2.0819) (86,2.0722) (87,2.0627) (88,2.0534) (89,2.0446) (90,2.0362) (91,2.0283) (92,2.0209) (93,2.0137) (94,2.0066) (95,1.9995) (100,1.9648) (105,1.9329) (110,1.9097) (115,1.9032) (120,1.8806) (125,1.8545) (130,1.8480) (135,1.8412) (140,1.8096) (145,1.8034) (150,1.8024) (155,1.7780) (160,1.7631) (165,1.7512) (170,1.7380) (175,1.7191) (180,1.7129) (185,1.7067) (190,1.6979) (195,1.6849) (200,1.6822) (205,1.6763) (210,1.6665) (215,1.6594) (220,1.6509) (225,1.6421) (230,1.6254) (235,1.6105) (240,1.6069) (245,1.5950) (250,1.5848) (255,1.5881) (260,1.5891) (265,1.5812) (270,1.5857) (275,1.5936) (280,1.5844) (285,1.5831) (290,1.5793) (295,1.5735) (300,1.5601) (305,1.5580) (310,1.5465) (315,1.5462) (320,1.5410)};
\addlegendentry{Baseline, to ep.\ 320}
\addplot[cAGT, line width=1.1pt] coordinates {(20,3.5170) (21,3.4357) (22,3.3555) (23,3.2766) (24,3.1991) (25,3.1234) (26,3.0496) (27,2.9778) (28,2.9082) (29,2.8412) (30,2.7767) (31,2.7150) (32,2.6564) (33,2.6010) (34,2.5489) (35,2.5004) (36,2.4557) (37,2.4150) (38,2.3784) (39,2.3461) (40,2.3183) (41,2.2936) (42,2.2704) (43,2.2486) (44,2.2281) (45,2.2088) (46,2.1906) (47,2.1733) (48,2.1570) (49,2.1414) (50,2.1266) (51,2.1123) (52,2.0986) (53,2.0853) (54,2.0723) (55,2.0595) (56,2.0468) (57,2.0341) (58,2.0214) (59,2.0085) (60,1.9953) (61,1.9821) (62,1.9690) (63,1.9562) (64,1.9436) (65,1.9311) (66,1.9189) (67,1.9069) (68,1.8951) (69,1.8834) (70,1.8720) (71,1.8608) (72,1.8497) (73,1.8388) (74,1.8281) (75,1.8176) (76,1.8073) (77,1.7971) (78,1.7872) (79,1.7773) (80,1.7677) (81,1.7583) (82,1.7492) (83,1.7403) (84,1.7317) (85,1.7233) (86,1.7150) (87,1.7069) (88,1.6989) (89,1.6909) (90,1.6830) (91,1.6750) (92,1.6670) (93,1.6593) (94,1.6520) (95,1.6455) (100,1.6202) (105,1.6004) (110,1.5693) (115,1.5543) (120,1.5431) (125,1.5300) (130,1.5122) (135,1.5094) (140,1.5082) (145,1.5011) (150,1.4906) (155,1.4839) (160,1.4779) (165,1.4611) (170,1.4618) (175,1.4549) (180,1.4532) (185,1.4416) (190,1.4428) (195,1.4203) (200,1.4160) (205,1.4060) (210,1.4095) (215,1.4019) (220,1.4037) (225,1.3889) (230,1.3811) (235,1.3730) (240,1.3636)};
\addlegendentry{\SG, $\alpha{=}0.4$, to ep.\ 240}
\addplot[cAGT, densely dashed, line width=0.6pt, forget plot] coordinates {(15,1.4982) (330,1.4982)};
\node[font=\tiny, cAGT!80!black, anchor=north east, inner sep=2pt] at (axis cs:326,1.4882) {\SG\ at ep.\ 120: 1.50};
\end{axis}
\end{tikzpicture}
    \caption{
    \textbf{\swift{} reaches sub-$1.0$ gFID on ImageNet-256.}
    \emph{Left:} FID-50K without and with inference-time guidance for
    representative published ImageNet-256 models and \swift{}
    (\Cref{tab:published}).
    Our RAEv2/DINOv3-L model reaches $1.15$ gFID without inference-time
    guidance and $0.99$ with retuned internal guidance
    (\Cref{app:raev2}).
    \emph{Right:} on the matched RAE/DINOv2-B configuration
    \citep{wang2025ddt,zheng2025rae,oquab2024dinov2},
    the advantage of \swift{} persists throughout extended training
    (\Cref{app:imagenet-traj}).
    Dots denote raw evaluations and lines show five-point moving averages.
    }
    \label{fig:overview}
\end{minipage}\par

\restoregeometry

\section{Introduction}
\label{sec:intro}

Diffusion models have achieved strong generative performance across
images, video, and audio
\citep{ho2020denoising,ho2022video,kong2021diffwave}.
They are typically trained with a simple denoising regression objective,
yet many of the strongest generation results rely on guidance applied only
at inference
\citep{ho2022classifier,karras2024guiding}.
In particular, autoguidance~\citep{karras2024guiding} shows that a
deliberately weaker denoiser can improve a stronger one by defining an
extrapolation direction during sampling.
This suggests a broader role for weak models: rather than serving as
teachers to imitate, they can provide directions along which a stronger
predictor should move.

Autoguidance realizes this idea by extrapolating the strong model's
prediction away from that of a smaller or less-trained counterpart
\citep{karras2024guiding}.
Related training-time approaches instead use evolving model predictions,
intermediate supervision, or guidance-conditioned objectives
\citep{tang2025mg,zhou2026internal,chen2025gft,yuan2026segmented}
(\Cref{app:guidance-comparison}).
These developments motivate a simple question:
\begin{tcolorbox}
    \centering
    \emph{Can weak-to-strong extrapolation be moved from inference into
    training, so that the final denoiser itself benefits from the correction
    and the weak model can be discarded afterward?}
\end{tcolorbox}

We introduce \swift{}, a self-guided training method that uses a frozen
weak denoiser to construct extrapolated regression targets.
For each noisy input, instead of regressing directly toward the clean
sample, \swift{} moves the target farther in the direction from the weak
prediction toward and beyond the clean sample.
The weak denoiser therefore acts as a \emph{directional reference}, rather
than a teacher whose output should be imitated.
It can be obtained from an early checkpoint of ordinary denoising training,
remains fixed during the main training run, and requires no backpropagation.
After training, the guide is discarded, leaving a single denoiser and the
original sampling procedure.
The same construction applies to both conditional and unconditional
generation.

Our motivation is finite-time underfitting: because a weak denoiser has
learned less of the target function, extrapolating away from its prediction
can amplify directions that remain underfit.
We show theoretically that, for a nonempty range of positive extrapolation
strengths, deliberately shifting the training target in this way can produce
a predictor strictly closer to the original posterior mean than matched
standard training (\Cref{thm:rkhs-swift}; proof in
\Cref{app:proof-rkhs-theorem}).

Experiments on CIFAR-10 and ImageNet-256 show that the effect is not tied
to a particular representation or model family.
On ImageNet-256, \swift{} improves diffusion models built on conventional
VAE latents, representation-autoencoder latents
\citep{zheng2025rae,singh2026raev2}, end-to-end learned latent
representations, and direct pixel-space generation.
Across the evaluated DiffusionBench configurations
\citep{diffusionbench2026}, the same fixed \swift{} recipe lowers unguided
FID in every configuration, and the improvement persists under longer
training. The complete ImageNet comparisons and training trajectories are in
\Cref{app:imagenet-all,app:imagenet-traj}, and the additional CIFAR-10
experiments are in \Cref{app:cifar}.

\swift{} also remains complementary to inference-time guidance.
When guidance is independently retuned for the baseline and \swift{},
the \swift{} model achieves lower FID in every tested configuration
(\Cref{tab:main-cfg,app:eval}).
Moreover, on RAEv2 the preferred internal-guidance scale decreases from
$1.78$ for the baseline to $1.20$ after \swift{} training, while the best
guided gFID improves from $1.04$ to $0.99$
(\swiftpanelref{fig:ag}{b}; \Cref{app:raev2}).
This is consistent with \swift{} internalizing part of the weak-to-strong
correction during training while leaving additional room for
inference-time guidance.

\Needspace{4\baselineskip}
Our \textbf{contributions} are:
\begin{itemize}[leftmargin=1.2em, itemsep=1pt, topsep=1pt]

    \item \textbf{Weak-to-strong learning through target extrapolation.}
    We introduce \swift{}, which uses a frozen weak denoiser as a
    directional reference to extrapolate clean-sample regression targets.
    The guide is fixed during main-model training and discarded at
    inference (\Cref{sec:swift-method}).

    \item \textbf{Broad improvements across diffusion representations.}
    We demonstrate consistent gains across CIFAR-10 and ImageNet-256,
    spanning VAE, RAE, end-to-end latent, and pixel-space diffusion models.
    The gains persist under longer training and remain after
    inference-time guidance is independently retuned
    (\Cref{sec:imagenet,sec:ag}).

    \item \textbf{A finite-training explanation.}
    We show theoretically that target extrapolation can strictly reduce
    posterior-mean error relative to matched standard training over a
    nonempty range of positive extrapolation strengths
    (\Cref{sec:rkhs-theory}).

\end{itemize}

\section{\swift{}: Self-Guided Training of Diffusion Models}
\label{sec:swift}

\subsection{Preliminaries}
\label{sec:swift-preliminaries}

Denoising diffusion models generate samples by reversing Gaussian
corruption \citep{ho2020denoising,karras2022elucidating}.
Let $t\in[0,1]$ index the noise level, not the training
iteration, and let $\sigma_t$ increase with $t$. We use the
additive-noise parameterization
\begin{equation*}
    \bm x_t=\bm x_0+\bm\varepsilon_t,
    \qquad
    \bm x_0\sim p_{\mathrm{data}},
    \qquad
    \bm\varepsilon_t\sim
    \mathcal N(\bm 0,\sigma_t^2\mathbf I_d).
\end{equation*}
The noise $\bm\varepsilon_t$ is independent of $\bm x_0$. We write $p_t$ for the distribution of $\bm x_t$ at fixed $t$.
When a condition $\bm c$ is present, it is an additional denoiser
input; we suppress it below and discuss conditional and null-condition
predictions in \Cref{sec:swift-comparison}.

\paragraph{Training.}
Standard clean-sample prediction~\citep{vincent2011denoising,karras2022elucidating} trains a denoiser
$\mathcal D_{\bm\theta}(\bm x_t,\sigma_t)$ to predict $\bm x_0$ by
minimizing
\begin{equation}
    \mathcal L_{\mathrm{base}}(\bm\theta)
    =
    \mathbb E\left[
        \lambda(t)
        \left\|
            \mathcal D_{\bm\theta}(\bm x_t,\sigma_t)-\bm x_0
        \right\|_2^2
    \right],
    \label{eq:diffusion-loss}
\end{equation}
where $t\sim p_{\mathrm{noise}}$ is sampled independently of
$\bm x_0\sim p_{\mathrm{data}}$, and $p_{\mathrm{noise}}$ is the distribution of noise levels used during training. The expectation also
includes $\bm\varepsilon_t$ drawn as above, and $\lambda(t)>0$ on
the support of $p_{\mathrm{noise}}$.
With finite second moments, the unrestricted population minimizer
of \Cref{eq:diffusion-loss} is the posterior mean
$\mathcal D^\star(\bm x_t,\sigma_t)
=\mathbb E[\bm x_0\mid\bm x_t,t]$ (\Cref{app:target-identity}).
For $\sigma_t>0$, Tweedie's formula relates it to the score
\citep{efron2011tweedie}:
\begin{equation}
    \nabla_{\bm x_t}\log p_t(\bm x_t)
    =\frac{\mathcal D^\star(\bm x_t,\sigma_t)-\bm x_t}
           {\sigma_t^2}.
    \label{eq:tweedie}
\end{equation}

\paragraph{Inference.}
With a differentiable variance-exploding noise schedule, the
probability-flow ODE is \citep{song2021score}
\begin{equation}
    d\bm x_t
    =-\dot{\sigma}_t\sigma_t
    \nabla_{\bm x_t}\log p_t(\bm x_t)\,dt.
    \label{eq:pf-ode}
\end{equation}
Sampling starts at a large noise level, where the noisy marginal can
be approximated by a Gaussian, and integrates \Cref{eq:pf-ode} toward
smaller $t$. In practice, the score in \Cref{eq:tweedie} is formed
using the trained denoiser $\mathcal D_{\bm\theta}$ in place of
$\mathcal D^\star$. Thus the sampler depends on the denoiser reached
after finite training.

\subsection{Proposed Method: Learning with a Frozen Weak Guide}
\label{sec:swift-method}

Building on the preliminaries, we introduce \swift{}, which trains a
denoiser using extrapolated regression targets. Instead of regressing
directly onto the clean sample $\bm x_0$ as in
\Cref{eq:diffusion-loss}, \swift{} minimizes
\begin{equation}
    \mathcal L_{\mathrm{SWIFT}}(\bm\theta)
    =
    \mathbb E\left[
        \lambda(t)
        \left\|
            \mathcal D_{\bm\theta}(\bm x_t,\sigma_t)
            -{\color{red}\bm y_\alpha}
        \right\|_2^2
    \right],
    \label{eq:swift-loss}
\end{equation}
with the same expectation and weighting as in
\Cref{eq:diffusion-loss}, but using the extrapolated target:
\begin{equation}
    \tcboxmath[
        colback=white,
        opacityback=1,
        colframe=black!75!white,
        boxrule=1.6pt,
        arc=5pt,
        boxsep=0pt,
        left=12pt,
        right=12pt,
        top=8pt,
        bottom=8pt
    ]{
        {\color{red}\bm y_\alpha}
        =
        \bm x_0
        +\alpha\bigl(
            \bm x_0
            -\mathcal D_{\mathrm{weak}}(\bm x_t,\sigma_t)
        \bigr)
    }
    \label{eq:swift-target}
\end{equation}
Here, $\mathcal D_{\mathrm{weak}}$ is a pretrained, frozen weak denoiser,
and $\alpha\geq0$ sets the extrapolation strength. Setting
$\alpha=0$ recovers standard denoising training. For $\alpha>0$,
the target extends beyond the clean sample in the direction away
from the weak denoiser's prediction. Velocity-predicting models~\citep{lipman2023flow} apply the
same extrapolation to the velocity, which is equivalent (\Cref{app:config}).

\paragraph{Constructing the frozen weak denoiser.}
We train $\mathcal D_{\mathrm{weak}}$ in the ordinary way with
\Cref{eq:diffusion-loss} and freeze it before the main run. The default
guide uses a smaller architecture; capacity and training-stage variants
are studied separately. The default DiffusionBench guide has approximately one
quarter as many parameters as the main denoiser and is trained for half
the main model's standard training duration; the CIFAR-10 and RAEv2 guide
checkpoints are specified in
\Cref{app:config}.
The method is \emph{self-guided} because the correction comes from
a weaker denoiser trained on the same task, without an external
teacher, a separate guidance objective, or an auxiliary prediction
head.

The intuition is to compensate for what ordinary training has not
learned within a finite training budget. Extrapolating the clean target
away from $\mathcal D_{\mathrm{weak}}$ emphasizes directions in which
the weak denoiser's prediction remains inaccurate. Whether this
improves the main denoiser's approximation to the posterior mean
depends on how those directions relate to its own finite-training
error.
For a frozen weak denoiser obtained by early stopping ordinary
training, \Cref{sec:rkhs-theory,thm:rkhs-swift} establish conditions
under which \swift{} achieves strictly lower posterior-mean error
than baseline training (proofs in \Cref{app:rkhs-theory}; a
Gaussian-mixture illustration is in \Cref{app:mog-experiment}).

\paragraph{Training and inference.}
During training, only the main denoiser
$\mathcal D_{\bm\theta}$ is optimized with \Cref{eq:swift-loss}.
The weak denoiser $\mathcal D_{\mathrm{weak}}$ is evaluated without
gradients and remains fixed. Thus, for a fixed clean sample, noise
realization, and noise level, the regression target does not change
as the main denoiser is updated. Beyond preparing the weak denoiser,
each training step requires one additional weak-denoiser forward pass
(about $9\%$ more FLOPs on DiffusionBench; \Cref{app:compute}).
After training, it is discarded; sampling uses the trained denoiser
and the original sampler, with no additional inference-time model
evaluation (\Cref{app:eval}).

\section{From Autoguidance to Self-Guided Training}
\label{sec:swift-comparison}

\begin{figure}[!t]
    \centering
    \includegraphics[width=\linewidth]
        {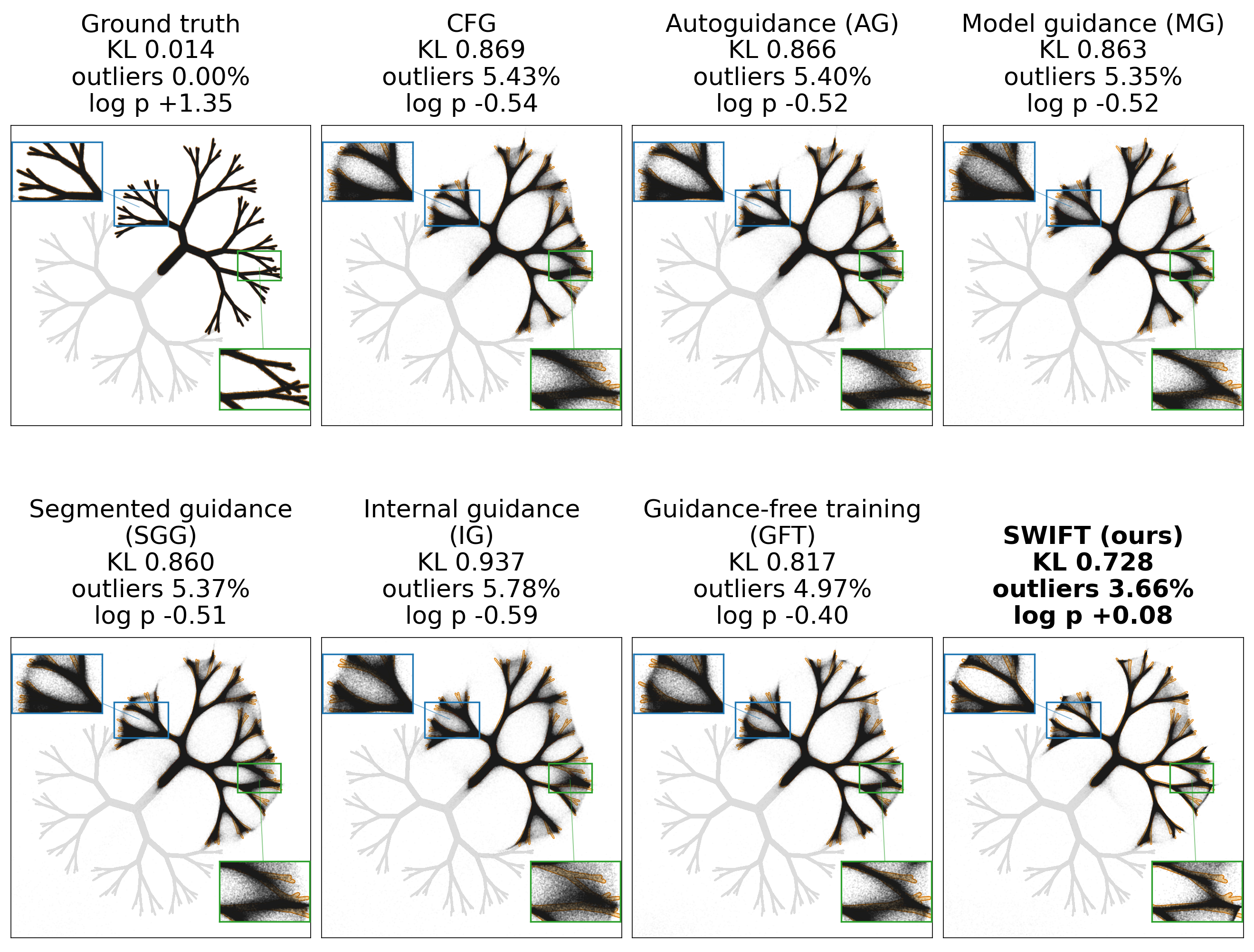}
    \caption{
        \textbf{\swift{} on a two-class, tree-shaped 2D toy distribution.}
        Ground truth and samples from CFG, AG, MG, SGG, IG, GFT, and
        \swift{}. Each compared method uses independently tuned
        settings. Blue boxes mark corresponding branching regions.
        Lower KL and outlier rates and higher $\log p$ are better.
    }
    \label{fig:toy-guidance}
\end{figure}

\paragraph{A visual comparison.}
Following the fractal-like Gaussian-mixture example of
\citet{karras2024guiding}, we compare how different methods reproduce a
distribution with thin, branching support. In \Cref{fig:toy-guidance},
\swift{} has the lowest reported KL divergence and outlier rate and the
highest reported $\log p$ among the seven methods. Relative to
GFT, the strongest baseline on these three reported metrics, its KL
decreases from $0.817$ to $0.728$, its outlier rate from $4.97\%$ to
$3.66\%$, and its $\log p$ increases from $-0.40$ to $+0.08$.
The highlighted regions show how the generated samples follow the
narrow branches.

Inference-time guidance includes classifier-free guidance
(CFG)~\citep{ho2022classifier} and autoguidance
(AG)~\citep{karras2024guiding}. Training-based approaches include
Model Guidance (MG)~\citep{tang2025mg}, Segmented Guidance
(SGG)~\citep{yuan2026segmented}, and Guidance-Free Training
(GFT)~\citep{chen2025gft}. Internal Guidance
(IG)~\citep{zhou2026internal} trains an intermediate output
for guidance at sampling time. Additional comparisons appear in
\Cref{app:guidance-comparison,tab:guidance-comparison}.

\subsection{From Sampling-Time Guidance to Training Targets}
\label{sec:swift-sampling-guidance}

CFG combines conditional and unconditional predictions during
sampling. AG is closer to \swift{} in its use of a weak denoiser.
Let $\mathcal D_{\mathrm{base}}$ be a denoiser trained with the
standard objective in \Cref{eq:diffusion-loss}. At the same noisy
input and noise level, AG evaluates
\begin{equation*}
    \mathcal D_{\mathrm{AG}}
    =
    (1+\alpha)\mathcal D_{\mathrm{base}}
    -\alpha\mathcal D_{\mathrm{weak}}
\end{equation*}
during sampling. This requires the weak prediction whenever AG is
applied. In \Cref{eq:swift-target}, \swift{} instead places the
observed clean sample $\bm x_0$ in the position occupied by
$\mathcal D_{\mathrm{base}}$ and uses the extrapolation as a
\emph{training target}. The weak denoiser is discarded after training.

Conditioned on $(\bm x_t,t)$, the expectation of the \swift{} target is
\begin{equation*}
    (1+\alpha)\mathcal D^\star(\bm x_t,\sigma_t)
    -\alpha\mathcal D_{\mathrm{weak}}(\bm x_t,\sigma_t),
\end{equation*}
the expression obtained by applying AG to the ideal posterior mean
(\swiftfullref{Lemma}{lem:target-decomposition}).
AG extrapolates a baseline denoiser already trained for a finite
time, while \swift{} changes the regression target before main-model
training. Under the assumptions of \Cref{sec:rkhs-theory}, this change
can improve approximation to the original posterior mean; it is not
a guarantee for every weak denoiser.

\subsection{Comparison with Training-Based Guidance}
\label{sec:swift-training-guidance}

MG constructs a training correction using conditional and
unconditional model predictions \citep{tang2025mg}.
SGG studies incorporating guidance into training using an updated
weak network or an intermediate prediction branch in its
training-time variants \citep{yuan2026segmented}.
GFT trains across guidance strengths using an unconditional
prediction \citep{chen2025gft}. In these constructions,
the auxiliary predictions can change as their models are trained.
IG jointly supervises intermediate and final outputs for
sampling-time guidance; it also studies training with a detached
final-minus-intermediate EMA correction
\citep[Sec.~6]{zhou2026internal}.

These designs use a changing prediction in the training correction
or supervise an additional output. \swift{} obtains
$\mathcal D_{\mathrm{weak}}$ by ordinary denoising training and
freezes it before training the main denoiser. For a fixed training
example and noise realization, its extrapolated target does not
change as the main model is updated. This avoids joint optimization
of the weak and main denoisers, which provided no clear additional
endpoint benefit in our CIFAR-10 comparisons
(\Cref{app:cifar}), and leaves a single denoiser at
inference. Preparing and evaluating the weak denoiser still incur
training compute, which our compute comparisons account for
(\Cref{app:compute}).
Further methodological details appear in
\Cref{app:guidance-comparison,eq:comparison-gft,eq:comparison-fixed-guide-gradient}.

\FloatBarrier

\providecommand{\swift}{\textnormal{\textsc{Swift}}}

\section{Why \swift{} Can Help}
\label{sec:rkhs-theory}

We compare ordinary denoising and \swift{} after the same finite
main-model training time. The analysis uses population gradient flow
at one fixed noise level.

\begin{assumption}[Population data and matched kernel training]
\label{ass:swift-setting}
Fix a noise level $t$ and a noisy--clean pair $(\bm x_0,\bm x_t)$ with
$\mathbb E\|\bm x_0\|_2^2<\infty$. All runs use population expectations
under the same joint law. The diffusion case has
$\bm x_t=\bm x_0+\bm\varepsilon$, with independent
$\bm\varepsilon\sim\mathcal N(\bm0,\sigma_t^2\mathbf I_d)$ and
$\sigma_t>0$; the theorem only requires the joint law and moment
condition above.

Fix a bounded, measurable, symmetric positive-semidefinite scalar
kernel $\mathsf K$. Every predictor starts from zero and follows
population gradient flow for normalized squared loss in the same
coordinate-wise RKHS $\mathcal H_{\mathsf K}^d$.
Choose $0\leq T_g<T<\infty$. An ordinary run on the clean label
$\bm x_0$ yields the guide at time $T_g$; freeze it and continue the
ordinary trajectory to $T$ for the baseline. Train a separate
\swift{} predictor from zero to the same time $T$, using the label
$(1+\alpha)\bm x_0-\alpha\mathcal D_{\mathrm{weak}}(\bm x_t)$ while the
guide stays frozen. The two main runs thus share the data law,
kernel, initialization, and optimization horizon.
\end{assumption}

Write $P_{\bm x_t}$ for the noisy-input distribution and
$\mathcal D^\star(\bm x)=\mathbb E[\bm x_0\mid\bm x_t=\bm x]$
for the posterior mean. Prediction error is measured in
$\mathcal H_2=L^2(P_{\bm x_t};\mathbb R^d)$.
The kernel couples predictions at different noisy inputs;
a finite-feature linear predictor is one example.

\paragraph{Objectives and training dynamics.}
At the fixed noise level, ordinary training and \swift{} minimize
\begin{equation}
  \begin{aligned}
    \mathcal J_0(\bm d)
    &=\tfrac12\mathbb E\|\bm d(\bm x_t)-\bm x_0\|_2^2,\\
    \mathcal J_\alpha(\bm d)
    &=\tfrac12\mathbb E\left\|
       \bm d(\bm x_t)-\bigl[(1+\alpha)\bm x_0
                    -\alpha\mathcal D_{\mathrm{weak}}(\bm x_t)\bigr]
       \right\|_2^2.
  \end{aligned}
  \label{eq:main-theory-objectives}
\end{equation}
For either loss, let $\bm y$ denote its regression label. Population
gradient flow in the RKHS is
\begin{equation}
  \frac{d\bm d_s}{ds}=-\nabla_{\mathcal H_{\mathsf K}^d}
       \tfrac12\mathbb E\|\bm d_s(\bm x_t)-\bm y\|_2^2,
  \qquad \bm d_0=\bm0.
  \label{eq:main-training-flow}
\end{equation}
Here $s$ is optimization time, distinct from noise level $t$;
the fixed positive loss weight $\lambda(t)$ is absorbed into $s$.
The gradient has the explicit function form
\begin{equation}
  \nabla\mathcal J(\bm d)(\bm x)
  =\mathbb E[\mathsf K(\bm x,\bm x_t)
                 (\bm d(\bm x_t)-\bm y)].
  \label{eq:main-gradient-formula}
\end{equation}
The corresponding derivation, including the RKHS-to-prediction-space
step, is in \swiftfullref{Lemma}{lem:kernel-flow}.

The resulting predictors are denoted by
$\mathcal D_{\mathrm{weak}}$, $\mathcal D_{\mathrm{base}}$, and
$\mathcal D_{\mathrm{SWIFT}}^{(\alpha)}$, respectively.
The guide is weaker here through shorter training in the same RKHS;
\swiftfullref{Corollary}{cor:reduced-rank-guide} also treats a spectrally nested
smaller guide.

\paragraph{Finite-time response.}
For any label $\bm y$, let
$\bm g(\bm x)=\mathbb E[\bm y\mid\bm x_t=\bm x]$.
The two conditional-mean targets are $\bm g_0=\mathcal D^\star$ and
\begin{equation}
  \bm g_\alpha=\mathcal D^\star+
              \alpha(\mathcal D^\star-\mathcal D_{\mathrm{weak}}).
  \label{eq:main-mean-target}
\end{equation}
The latter identity and its denoising-MSE interpretation are proved in
\swiftfullref{Lemma}{lem:target-decomposition}.
Define the population kernel operator on $\mathcal H_2$ by
\begin{equation}
  (\boldsymbol{\mathcal L}\bm f)(\bm x)
  =\int \mathsf K(\bm x,\bm x')\bm f(\bm x')
                 \,dP_{\bm x_t}(\bm x').
  \label{eq:main-population-operator}
\end{equation}
It is bounded, self-adjoint, and positive semidefinite. Kernel gradient
flow yields
\begin{equation}
  \frac{d\bm d_s}{ds}
  =-\boldsymbol{\mathcal L}(\bm d_s-\bm g),
  \qquad
  \bm d_s=\mathbf H_s\bm g,\quad
  \mathbf H_s:=\mathbf I-e^{-s\boldsymbol{\mathcal L}}.
  \label{eq:main-response-operator}
\end{equation}
The response $\mathbf H_s$ learns a fraction
$1-e^{-s\kappa}$ of a kernel mode with eigenvalue $\kappa>0$.
The operator construction and solution of the flow appear in
\Cref{app:rkhs-flow}.

\begin{theorem}[Finite-time improvement from a frozen weak denoiser]
\label{thm:rkhs-swift}
Under \Cref{ass:swift-setting},
\begin{equation*}
  \begin{aligned}
    \mathcal D_{\mathrm{weak}}&=\mathbf H_{T_g}\mathcal D^\star,
    &\mathcal D_{\mathrm{base}}&=\mathbf H_T\mathcal D^\star,\\
    \mathcal D_{\mathrm{SWIFT}}^{(\alpha)}
    &=\mathbf H_T[(1+\alpha)\mathcal D^\star
                         -\alpha\mathcal D_{\mathrm{weak}}].
  \end{aligned}
\end{equation*}
If $\mathbf H_T\mathcal D^\star\ne\bm0$, the guide has strictly larger
posterior-mean risk than the baseline. Moreover, for
$R(\alpha)=\|\mathcal D_{\mathrm{SWIFT}}^{(\alpha)}
                 -\mathcal D^\star\|_{\mathcal H_2}^2$,
there is an $\alpha^\star>0$ such that
\begin{equation}
  R(\alpha)<R(0)\qquad(0<\alpha<\alpha^\star).
  \label{eq:main-helpful-window}
\end{equation}
\end{theorem}

\Cref{app:proof-rkhs-theorem} proves both claims using
\Cref{thm:operator-improvement}. An earlier guide has learned less of
each learnable mode; a small positive extrapolation partially cancels
the baseline's finite-time residual. The modewise calculation is
\Cref{eq:app-scalar-mechanism}, and a larger coefficient can overshoot.
The result concerns population denoising MSE at equal main-model
optimization time, with guide preparation excluded; it does not
directly imply a FID improvement.
A Gaussian-mixture example gives an explicit posterior mean
(\swiftfullref{Lemma}{lem:mog-posterior}), and for an RBF kernel
\swiftfullref{Proposition}{prop:rbf-injective} verifies the nonzero-predictor condition.
\Cref{app:finite-step-affine,app:finite-sample} give complementary finite-step
and finite-sample analyses.
\FloatBarrier

\section{Experiments}
\label{sec:experiments}

We evaluate \swift{} on ImageNet-256 \citep{deng2009imagenet} and
CIFAR-10 \citep{krizhevsky2009cifar}. In every experiment \swift{} changes only the regression
target; architectures, training schedules, samplers, and evaluation
pipelines are kept the same as the baseline. On ImageNet, \swift{} uses one fixed
recipe unless stated otherwise: $\alpha{=}0.4$ with a smaller pretrained guide
frozen before main-model training. The default DiffusionBench guide is
trained for half the main model's standard training duration
(\Cref{app:config}).

\paragraph{Evaluation setup.}
Inspired by DiffusionBench~\citep{diffusionbench2026}, we evaluate
\swift{} across 19 ImageNet-256 model configurations. For each
configuration, we compare the baseline and \swift{} with and without
inference-time guidance, tuning both guidance strengths separately
where checkpoints are available and identifying published
baselines otherwise. We report FID \citep{heusel2017gans},
sFID \citep{nash2021generating}, precision and recall
\citep{kynkaanniemi2019improved}, representation-space FD$_r$
\citep{yang2026fdr}, and MIND \citep{berthet2026mind} where available.
\Cref{tab:representative} shows selected results; training and evaluation
settings are in \Cref{app:config,app:eval}, with complete ImageNet-256
and CIFAR-10 results in \Cref{app:imagenet-full,app:cifar}.

\subsection{Generation Quality and Transfer Across Representations}
\label{sec:imagenet}

\paragraph{Matched improvements.}
\Cref{tab:representative} isolates the target change within each model
configuration. On RAE/DINOv2-B
features~\citep{wang2025ddt,zheng2025rae,oquab2024dinov2},
\swift{} lowers unguided FID from $2.09$ to $1.75$ and the
five-feature mean FD$_r$ from $4.13$ to $3.69$ after 80 epochs.
The gain persists on the independent RAEv2/DINOv3-L
configuration~\citep{singh2026raev2}, where unguided FID improves
from $1.56$ to $1.15$ within a 160-epoch training budget; the \swift{}
value is selected at epoch 140 (\Cref{app:eval}).

\begin{table}[t]
\centering
\caption{\textbf{Representative ImageNet-256 results.}
Each pair uses the same model configuration and sampler. Guidance is
tuned separately for re-evaluated pairs; source-reported guided
baselines retain their published settings (\Cref{app:eval}).
$\overline{\mathrm{FD}}_r$ and $\overline{\mathrm{MIND}}$ are arithmetic
means across Inception, ConvNeXt, DINOv2, MAE, and SigLIP features.}
\label{tab:representative}
\setlength{\tabcolsep}{2.2pt}
\renewcommand{\arraystretch}{1.08}
\resizebox{\linewidth}{!}{%
\begin{tabular}{@{}l*{12}{r}@{}}
\toprule
& \multicolumn{6}{c}{Without guidance}
& \multicolumn{6}{c}{With guidance} \\
\cmidrule(lr){2-7}\cmidrule(l){8-13}
Family: representation
& FID$\downarrow$ & sFID$\downarrow$
& $\overline{\mathrm{FD}}_r\downarrow$
& $\overline{\mathrm{MIND}}\downarrow$
& Prec.$\uparrow$ & Rec.$\uparrow$
& FID$\downarrow$ & sFID$\downarrow$
& $\overline{\mathrm{FD}}_r\downarrow$
& $\overline{\mathrm{MIND}}\downarrow$
& Prec.$\uparrow$ & Rec.$\uparrow$ \\
\midrule
\multicolumn{13}{@{}l}{\textit{615M parameters, 80 epochs; guidance = CFG}} \\
RAE: DINOv2-B
& 2.09 & 4.48 & 4.13 & 21.34 & 0.83 & 0.57
& 1.98 & 4.50 & 4.17 & 20.71 & 0.81 & 0.59 \\
\quad $+$ \swift{}
& \textbf{1.75} & 4.26 & 3.69 & 21.54 & 0.81 & 0.61
& \textbf{1.59} & 4.30 & 3.57 & 19.40 & 0.80 & 0.62 \\
\addlinespace[0.6ex]
VAE: SD-VAE (EMA)$+$reg
& 5.49 & 5.52 & 11.08 & 87.84 & 0.71 & 0.66
& 2.32 & 4.70 & 7.39 & 32.52 & 0.80 & 0.60 \\
\quad $+$ \swift{}
& \textbf{3.77} & 5.06 & 8.44 & 57.83 & 0.73 & 0.66
& \textbf{1.81} & 4.42 & 5.93 & 26.15 & 0.80 & 0.61 \\
\addlinespace[0.6ex]
E2E: Qwen-VAE
& 4.94 & 5.23 & 9.17 & 95.20 & 0.72 & 0.66
& 1.50 & 4.83 & 4.73 & 24.40 & 0.80 & 0.62 \\
\quad $+$ \swift{}
& \textbf{3.23} & 4.89 & 6.47 & 55.63 & 0.73 & 0.66
& \textbf{1.40} & 4.59 & 3.33 & 23.40 & 0.80 & 0.63 \\
\addlinespace[0.6ex]
Pixel: PixNerd
& 20.37 & 6.14 & 24.52 & 384.17 & 0.61 & 0.63
& 4.17 & 5.50 & 9.50 & 42.66 & 0.65 & 0.63 \\
\quad $+$ \swift{}
& \textbf{10.31} & 5.38 & 15.28 & 197.54 & 0.70 & 0.61
& \textbf{3.38} & 4.95 & 7.75 & 35.12 & 0.68 & 0.62 \\
\midrule
\multicolumn{13}{@{}l}{\textit{875M parameters, up to 160 epochs; guidance = internal guidance}} \\
RAE: DINOv3-L
& 1.56 & 4.10 & 3.41 & 20.01 & 0.82 & 0.65
& 1.04 & 4.00 & 3.31 & 15.67 & 0.82 & 0.61 \\
\quad $+$ \swift{}
& \textbf{1.15} & 3.60 & 2.94 & 18.03 & 0.81 & 0.67
& \textbf{0.99} & 3.80 & 3.07 & 14.45 & 0.81 & 0.63 \\
\bottomrule
\end{tabular}%
}
\end{table}

\paragraph{Transfer with one fixed recipe.}
Across 18 DiffusionBench configurations spanning latent and pixel
representations, one fixed recipe lowers unguided FID in every case,
by $16.3\%$--$49.4\%$ (median $31.3\%$). Even the three
lowest-FID baselines improve by $16.3\%$--$17.4\%$.
FD$_r$ and MIND each improve in at least 17 of 18 configurations
in every feature space. Full results appear in \Cref{tab:main}.

\paragraph{Compatibility with inference guidance.}
At the best available guidance settings, \swift{} also lowers guided
FID in all 18 DiffusionBench configurations
(\Cref{tab:main-cfg}). On RAEv2, \swift{} improves internal-guidance
FID from $1.04$ to $0.99$ (\Cref{tab:representative}), while its
unguided FID improves from $1.56$ to $1.15$.
Representative samples from the $0.99$-gFID model are shown in
\Cref{fig:swift-imagenet-samples}; additional class-conditional samples
appear in \Cref{app:samples}.

\begin{figure}[H]
    \centering
    \includegraphics[width=\linewidth]
        {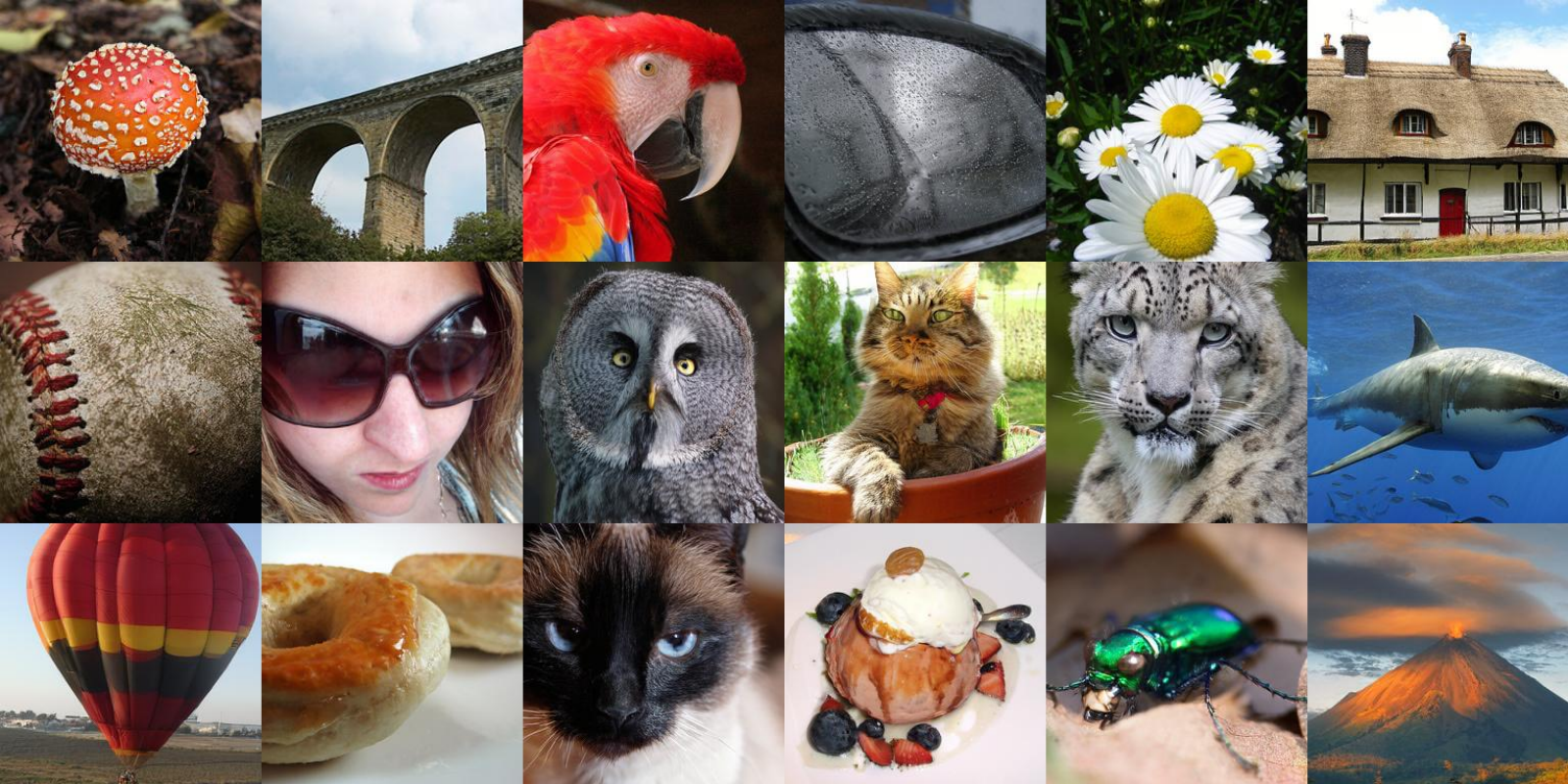}
    \caption{
    \textbf{ImageNet-256 samples from \swift{} on RAEv2/DINOv3-L
    (gFID-50K $0.99$).}
    Samples are generated with the \swift{}-trained model using retuned
    internal guidance and a 50-step ODE sampler (\Cref{app:eval}).
    }
    \label{fig:swift-imagenet-samples}
\end{figure}

\subsection{Combining \swift{} with Inference-Time Guidance}
\label{sec:ag}

We combine \swift{} with autoguidance~\citep{karras2024guiding} on the
DiffusionBench models and native internal REPA-head guidance on
RAEv2~\citep{singh2026raev2}. These mechanisms apply a weak-to-strong
extrapolation at sampling time, whereas \swift{} learns a related correction
during training and then discards its guide. We ask whether inference-time
guidance still improves a \swift{}-trained model and whether \swift{} changes
the preferred guidance strength.

\begin{figure}[H]
    \centering

    \begin{minipage}[t]{0.43\textwidth}
        \centering
        \vspace{0pt}

        {\small
        \textbf{(a) Guidance further improves \swift{}.}
        }

        \vspace{5pt}

        {\small
        \setlength{\tabcolsep}{3.0pt}
        \renewcommand{\arraystretch}{1.10}
        \begin{tabular}{@{}lrrrr@{}}
            \toprule
            & Base & \swift{} & Guided & \shortstack{\swift{}+\\guidance} \\
            \midrule
            DINOv2-B & 2.09  & 1.75  & 1.70 & \textbf{1.63} \\
            SD-VAE   & 5.49  & 3.77  & 2.16 & \textbf{1.95} \\
            Qwen-VAE & 4.94  & 3.23  & 1.79 & \textbf{1.62} \\
            PixNerd  & 20.37 & 10.31 & 4.02 & \textbf{3.77} \\
            RAEv2    & 1.56  & 1.15  & 1.04 & \textbf{0.99} \\
            \bottomrule
        \end{tabular}
        }
    \end{minipage}
    \hfill
    \begin{minipage}[t]{0.53\textwidth}
        \centering
        \vspace{0pt}

        {\small
        \textbf{(b) \swift{} shifts the guidance optimum left.}
        }

        \vspace{1pt}

        \begin{tikzpicture}
        \begin{axis}[
            width=\linewidth,
            height=0.52\linewidth,
            xlabel={internal guidance scale $s$},
            ylabel={FID-50K},
            xmin=0.98,
            xmax=2.02,
            ymin=0.97,
            ymax=1.62,
            xtick={1.0,1.2,1.4,1.6,1.8,2.0},
            ytick={1.0,1.2,1.4,1.6},
            tick label style={
                font=\scriptsize,
                /pgf/number format/assume math mode=true
            },
            label style={font=\scriptsize},
            ymajorgrids,
            grid style={cGrid, line width=0.3pt},
            axis line style={cAxis!65, line width=0.45pt},
            tick style={cAxis!65, line width=0.45pt},
            every axis plot/.append style={
                line width=0.9pt,
                mark size=1.5pt,
                mark options={line width=0.5pt}
            },
            legend style={
                font=\tiny,
                draw=none,
                fill=white,
                fill opacity=0.90,
                text opacity=1,
                at={(0.98,0.98)},
                anchor=north east,
                inner sep=1.5pt,
                row sep=-2pt
            },
            legend cell align=left
        ]

        \addplot[cBase, dashdotted, mark=*] coordinates {
            (1.00,1.5628)
            (1.40,1.1603)
            (1.60,1.0875)
            (1.78,1.0400)
            (2.00,1.0790)
        };
        \addlegendentry{Baseline}

        \addplot[cAGT, mark=*] coordinates {
            (1.00,1.1479)
            (1.10,1.0630)
            (1.20,0.9900)
            (1.30,1.0016)
            (1.40,1.0163)
            (1.50,1.0321)
            (1.60,1.0492)
            (1.70,1.0760)
            (1.78,1.1090)
        };
        \addlegendentry{\swift{}}

        \addplot[
            only marks,
            mark=*,
            mark size=2.8pt,
            cBase
        ] coordinates {
            (1.78,1.0400)
        };

        \addplot[
            only marks,
            mark=*,
            mark size=2.8pt,
            cAGT
        ] coordinates {
            (1.20,0.9900)
        };

        \end{axis}
        \end{tikzpicture}
    \end{minipage}

    \vspace{-3pt}

    \caption{
    \textbf{\swift{} remains beneficial with inference-time guidance.}
    \textbf{(a)} FID with autoguidance on the DiffusionBench models and
    native internal REPA-head guidance on RAEv2.
    \textbf{(b)} On RAEv2, \swift{} shifts the preferred internal-guidance
    scale from $1.78$ to $1.20$.
    }
    \label{fig:ag}

    \vspace{-5pt}
\end{figure}
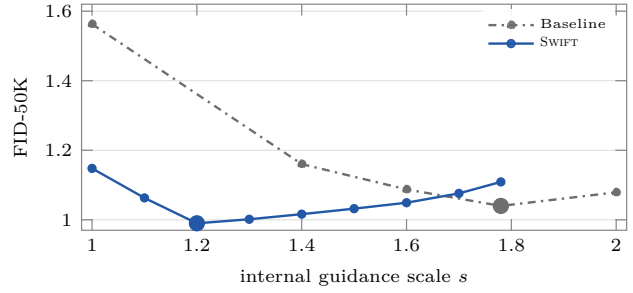

Across all five configurations in \swiftpanelref{fig:ag}{a}, the same ordering
appears. \swift{} improves the unguided baseline, inference-time guidance
improves it further, and independently retuned guidance on top of \swift{} gives the
best FID in every case. This shows that the improvement learned during
training is not replaced by inference-time guidance: even after the baseline
has been separately retuned for guidance, the \swift{} model retains an
additional gain.

\swift{} also changes the preferred inference-time guidance strength. On RAEv2,
the optimal internal-guidance scale shifts from $1.78$ for the baseline to
$1.20$ for \swift{}, while the best guided FID improves from $1.04$ to
$0.99$ (\swiftpanelref{fig:ag}{b}; \Cref{app:raev2}). The lower preferred scale is consistent with
overlap between the training-time and inference-time corrections.
The remaining gain shows that inference-time guidance remains complementary
to the correction learned during training.

\subsection{Longer Training and Training Cost}
\label{sec:efficiency}

\begingroup
\setlength{\columnsep}{0.42cm}
\paragraph{The gain persists with longer training.}
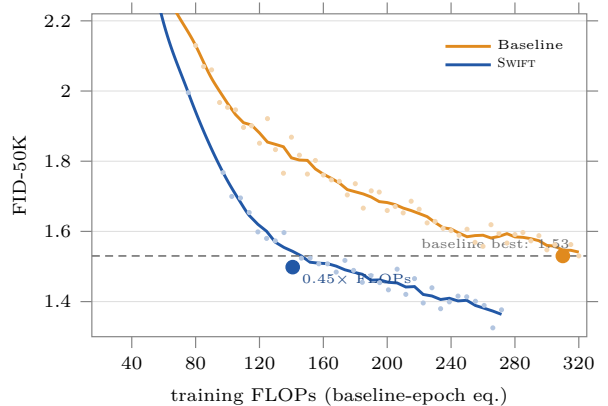
\begin{wrapfigure}{r}{0.49\textwidth}
    \centering
    \begin{minipage}{\linewidth}
    \centering

    \begin{tikzpicture}
    \begin{axis}[
        width=\linewidth,
        height=0.72\linewidth,
        tick label style={
            font=\scriptsize,
            /pgf/number format/assume math mode=true
        },
        label style={font=\scriptsize},
        ymajorgrids,
        grid style={cGrid, line width=0.3pt},
        axis line style={cAxis!65, line width=0.45pt},
        tick style={cAxis!65, line width=0.45pt},
        tick align=outside,
        every axis plot/.append style={
            line width=0.9pt,
            mark size=1.25pt,
            mark options={line width=0.5pt}
        },
        legend style={
            font=\tiny,
            draw=none,
            fill=white,
            fill opacity=0.9,
            text opacity=1,
            row sep=-2pt,
            inner xsep=2pt,
            inner ysep=1.5pt,
            at={(0.975,0.95)},
            anchor=north east
        },
        legend cell align=left,
        xlabel={training FLOPs (baseline-epoch eq.)},
        ylabel={FID-50K},
        xmin=15,
        xmax=325,
        ymin=1.30,
        ymax=2.22,
        xtick={40,80,120,160,200,240,280,320},
        ytick={1.4,1.6,1.8,2.0,2.2}
    ]

    \addplot[
        only marks,
        mark=*,
        mark size=0.7pt,
        cTwo!35,
        forget plot
    ] coordinates {
        (20,3.6473)
        (40,2.7156)
        (60,2.3012)
        (80,2.1297)
        (85,2.0699)
        (90,2.0603)
        (95,1.9673)
        (100,1.9537)
        (105,1.9463)
        (110,1.8963)
        (115,1.9009)
        (120,1.8513)
        (125,1.9213)
        (130,1.8331)
        (135,1.7658)
        (140,1.8685)
        (145,1.8172)
        (150,1.7634)
        (155,1.8021)
        (160,1.7606)
        (165,1.7469)
        (170,1.7427)
        (175,1.7039)
        (180,1.7359)
        (185,1.6663)
        (190,1.7156)
        (195,1.7120)
        (200,1.6596)
        (205,1.6709)
        (210,1.6528)
        (215,1.6860)
        (220,1.6632)
        (225,1.6243)
        (230,1.6284)
        (235,1.6086)
        (240,1.6023)
        (245,1.5890)
        (250,1.6062)
        (255,1.5688)
        (260,1.5575)
        (265,1.6191)
        (270,1.5940)
        (275,1.5668)
        (280,1.5910)
        (285,1.5969)
        (290,1.5732)
        (295,1.5878)
        (300,1.5475)
        (305,1.5619)
        (310,1.5301)
        (315,1.5626)
        (320,1.5302)
    };

    \addplot[cTwo, line width=1.1pt] coordinates {
        (20,3.6473)
        (21,3.5882)
        (22,3.5300)
        (23,3.4727)
        (24,3.4164)
        (25,3.3612)
        (26,3.3072)
        (27,3.2544)
        (28,3.2030)
        (29,3.1529)
        (30,3.1044)
        (31,3.0573)
        (32,3.0119)
        (33,2.9682)
        (34,2.9262)
        (35,2.8861)
        (36,2.8478)
        (37,2.8116)
        (38,2.7774)
        (39,2.7454)
        (40,2.7156)
        (41,2.6873)
        (42,2.6597)
        (43,2.6329)
        (44,2.6068)
        (45,2.5816)
        (46,2.5570)
        (47,2.5333)
        (48,2.5104)
        (49,2.4883)
        (50,2.4670)
        (51,2.4466)
        (52,2.4269)
        (53,2.4082)
        (54,2.3902)
        (55,2.3732)
        (56,2.3570)
        (57,2.3417)
        (58,2.3273)
        (59,2.3138)
        (60,2.3012)
        (61,2.2894)
        (62,2.2784)
        (63,2.2679)
        (64,2.2581)
        (65,2.2487)
        (66,2.2398)
        (67,2.2313)
        (68,2.2232)
        (69,2.2153)
        (70,2.2076)
        (71,2.2001)
        (72,2.1927)
        (73,2.1853)
        (74,2.1779)
        (75,2.1704)
        (76,2.1628)
        (77,2.1550)
        (78,2.1469)
        (79,2.1385)
        (80,2.1297)
        (81,2.1205)
        (82,2.1111)
        (83,2.1014)
        (84,2.0916)
        (85,2.0819)
        (86,2.0722)
        (87,2.0627)
        (88,2.0534)
        (89,2.0446)
        (90,2.0362)
        (91,2.0283)
        (92,2.0209)
        (93,2.0137)
        (94,2.0066)
        (95,1.9995)
        (100,1.9648)
        (105,1.9329)
        (110,1.9097)
        (115,1.9032)
        (120,1.8806)
        (125,1.8545)
        (130,1.8480)
        (135,1.8412)
        (140,1.8096)
        (145,1.8034)
        (150,1.8024)
        (155,1.7780)
        (160,1.7631)
        (165,1.7512)
        (170,1.7380)
        (175,1.7191)
        (180,1.7129)
        (185,1.7067)
        (190,1.6979)
        (195,1.6849)
        (200,1.6822)
        (205,1.6763)
        (210,1.6665)
        (215,1.6594)
        (220,1.6509)
        (225,1.6421)
        (230,1.6254)
        (235,1.6105)
        (240,1.6069)
        (245,1.5950)
        (250,1.5848)
        (255,1.5881)
        (260,1.5891)
        (265,1.5812)
        (270,1.5857)
        (275,1.5936)
        (280,1.5844)
        (285,1.5831)
        (290,1.5793)
        (295,1.5735)
        (300,1.5601)
        (305,1.5580)
        (310,1.5465)
        (315,1.5462)
        (320,1.5410)
    };
    \addlegendentry{Baseline}

    \addplot[
        only marks,
        mark=*,
        mark size=0.7pt,
        cAGT!35,
        forget plot
    ]
table[
        x expr={10 + 1.09*\thisrowno{0}},
        y index=1
    ] {figures/compute_swift_raw.dat};

    \addplot[cAGT, line width=1.1pt]
table[
        x expr={10 + 1.09*\thisrowno{0}},
        y index=1
    ] {figures/compute_swift_smooth.dat};
    \addlegendentry{\swift{}}

    \addplot[
        cAxis!70,
        densely dashed,
        line width=0.65pt,
        forget plot
    ] coordinates {
        (15,1.5301)
        (325,1.5301)
    };

    \addplot[
        only marks,
        mark=*,
        mark size=2.5pt,
        cTwo,
        forget plot
    ] coordinates {
        (310,1.5301)
    };

    \addplot[
        only marks,
        mark=*,
        mark size=2.5pt,
        cAGT,
        forget plot
    ] coordinates {
        (140.8,1.4982)
    };

    \node[
        font=\tiny,
        cAxis!80,
        anchor=south east,
        inner sep=2pt
    ] at (axis cs:318,1.535) {
        baseline best: 1.53
    };

    \node[
        font=\tiny,
        cAGT!80!black,
        anchor=north west,
        inner sep=2pt
    ] at (axis cs:143,1.495) {
        $0.45\times$ FLOPs
    };

    \end{axis}
    \end{tikzpicture}
    \end{minipage}

    \caption{
        \textbf{Training-compute comparison on DINOv2-B.}
        \swift{} surpasses the best baseline FID observed through 320 epochs
        using approximately $0.45\times$ the FLOPs needed to reach that
        baseline checkpoint.
    }
    \label{fig:flops}
\end{wrapfigure}
We continue the matched DINOv2-B pair to 240 epochs. \swift{} is ahead at
all 36 evaluations from epochs 20 to 240, with the gap growing from $0.12$
to $0.26$ FID, with the training-time evaluator giving
$1.63\!\to\!1.38$ at epoch 240. A separate baseline continued to 320 epochs reaches at
best $1.53$, still above the $1.50$ reached by \swift{} at epoch 120.
The matched training-time FIDs are listed in \Cref{tab:long} and plotted in
\Cref{fig:long-run}; four additional trajectories appear in
\Cref{fig:four,app:imagenet-traj}.

\paragraph{Training cost.}
\swift{} adds $9\%$ FLOPs per epoch for the frozen guide, plus roughly
10 baseline-equivalent epochs to pretrain it (\Cref{app:compute}).
For context, MG and SGG incur about $1.23\times$ and $1.22\times$
wall-clock time per training iteration in the unified SiT-B/2 comparison
of \citet[Table~6]{yuan2026segmented}, while GFT reports 10--20\%
additional computation per training update \citep{chen2025gft}.
IG reports negligible extra sampling FLOPs \citep[App.~B]{zhou2026internal}.
Charging both of our costs, \swift{} at
epoch 120 already surpasses the best DINOv2-B baseline result observed
through 320 epochs ($1.53$), using approximately $0.45\times$ the FLOPs
needed to reach that baseline checkpoint
(\Cref{fig:flops}). On RAEv2, epoch-50 \swift{} similarly improves on the
epoch-160 baseline ($1.34$ vs.\ $1.56$) at $0.44\times$ the FLOPs
(including guide preparation; \Cref{app:compute}).

These time-to-quality ratios are conservative: within the observed training
horizon, the baselines never reach the best \swift{} FID. No baseline
time-to-quality is observed for those targets.
\endgroup

\subsection{How to Choose \texorpdfstring{$\alpha$}{alpha} and the Weak Guide}
\label{sec:selection}

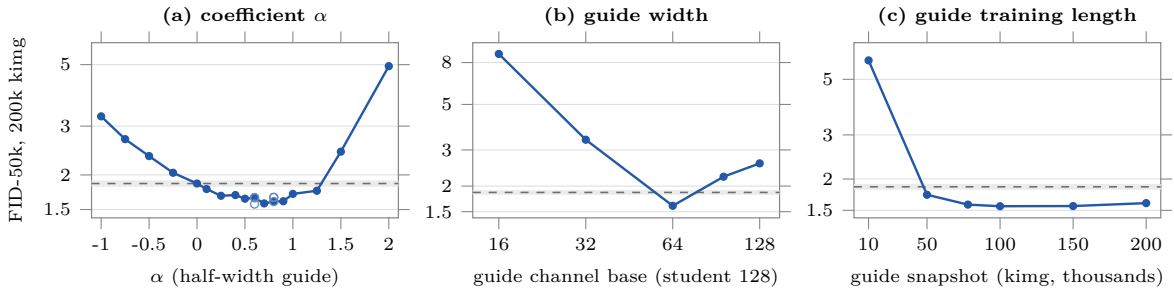
\begin{figure}[!htbp]
\centering
\begin{tikzpicture}
\begin{groupplot}[group style={group size=3 by 1, horizontal sep=28pt}, width=0.34\linewidth, height=0.235\linewidth,
tick label style={font=\scriptsize, /pgf/number format/assume math mode=true},
  label style={font=\scriptsize}, title style={font=\scriptsize\bfseries, yshift=-3pt},
  ymajorgrids, grid style={cGrid, line width=0.3pt},
  axis line style={cAxis!65, line width=0.45pt},
  tick style={cAxis!65, line width=0.45pt}, tick align=outside,
  xtick style={line width=0.4pt}, ytick style={line width=0.4pt},
  every axis plot/.append style={line width=0.9pt, mark size=1.25pt,
      mark options={line width=0.5pt}},
  legend style={font=\tiny, draw=none, fill=white, fill opacity=0.9, text opacity=1,
      row sep=-2.5pt, inner xsep=3pt, inner ysep=2pt},
  legend cell align=left, legend image post style={line width=0.9pt}]
\nextgroupplot[title={(a) coefficient $\alpha$}, xlabel={$\alpha$ (half-width guide)}, ylabel={FID-50k, 200k kimg}, ymode=log, ymin=1.4, ymax=6, ytick={1.5,2,3,5}, yticklabels={1.5,2,3,5}, xmin=-1.1, xmax=2.1, xtick={-1,-0.5,0,0.5,1,1.5,2}]
\addplot[draw=none, fill=cBase!14, forget plot] coordinates {(-1.1,1.8131) (2.1,1.8131) (2.1,1.9111) (-1.1,1.9111)} \closedcycle;
\addplot[cBase, dashed, line width=0.6pt, domain=-1.1:2.1, samples=2, forget plot] {1.8621};
\addplot[cAGT, mark=*, mark size=1.3pt, forget plot] coordinates {(-1.00,3.2536) (-0.75,2.6916) (-0.50,2.3387) (-0.25,2.0357) (0.00,1.8621) (0.10,1.7788) (0.25,1.6816) (0.40,1.6919) (0.50,1.6415) (0.60,1.6536) (0.70,1.5783) (0.80,1.6027) (0.90,1.6074) (1.00,1.7085) (1.25,1.7534) (1.50,2.4265) (2.00,4.9364)};
\addplot[only marks, mark=o, mark size=1.5pt, cAGT!70, line width=0.5pt, forget plot] coordinates {(0.60,1.6540) (0.60,1.5700) (0.60,1.6580)};
\addplot[only marks, mark=o, mark size=1.5pt, cAGT!70, line width=0.5pt, forget plot] coordinates {(0.80,1.6030) (0.80,1.6580) (0.80,1.6610)};
\nextgroupplot[title={(b) guide width}, xlabel={guide channel base (student 128)}, ymode=log, ymin=1.4, ymax=10, ytick={1.5,2,3,5,8}, yticklabels={1.5,2,3,5,8}, xmode=log, xmin=13, xmax=150, xtick={16,32,64,128}, xticklabels={16,32,64,128}]
\addplot[draw=none, fill=cBase!14, forget plot] coordinates {(13,1.8131) (150,1.8131) (150,1.9111) (13,1.9111)} \closedcycle;
\addplot[cBase, dashed, line width=0.6pt, domain=13:150, samples=2, forget plot] {1.8621};
\addplot[cAGT, mark=*, mark size=1.3pt, forget plot] coordinates {(16,8.8151) (32,3.3649) (64,1.6027) (96,2.2217) (128,2.5806)};
\nextgroupplot[title={(c) guide training length}, xlabel={guide snapshot (kimg, thousands)}, ymode=log, ymin=1.4, ymax=7, ytick={1.5,2,3,5}, yticklabels={1.5,2,3,5}, xmin=0, xmax=210, xtick={10,50,100,150,200}]
\addplot[draw=none, fill=cBase!14, forget plot] coordinates {(0,1.8131) (210,1.8131) (210,1.9111) (0,1.9111)} \closedcycle;
\addplot[cBase, dashed, line width=0.6pt, domain=0:210, samples=2, forget plot] {1.8621};
\addplot[cAGT, mark=*, mark size=1.3pt, forget plot] coordinates {(10,5.9480) (50,1.7304) (78,1.5810) (100,1.5575) (150,1.5602) (200,1.6027)};
\end{groupplot}
\end{tikzpicture}
\caption{
\textbf{Choosing $\alpha$ and the weak guide on CIFAR-10.}
Moderate positive $\alpha$, a half-width guide, and a sufficiently trained weak guide give the strongest results.
The CIFAR-10 training and evaluation settings are in \Cref{app:config,app:eval}.
}
\label{fig:selection}
\end{figure}

\paragraph{Choosing $\alpha$.}
Performance is broadly U-shaped in $\alpha$ (\swiftpanelref{fig:selection}{a}).
On CIFAR-10, moderate positive values form a broad favorable region around
$\alpha=0.7$--$0.9$, where differences are comparable to replicate noise.
ImageNet shows the same qualitative behavior: on DINOv2-B,
$\alpha\in[0.5,1.0]$ substantially improves over the baseline.
The precise optimum depends on the model,
however, so our breadth experiments use the fixed default $\alpha=0.4$
without tuning; in practice, a small sweep over moderate positive values is
sufficient. Negative $\alpha$ can accelerate the initial decrease in FID,
but the tested negative warm starts provide no consistent final-quality
advantage (\Cref{app:negative}).

\paragraph{Choosing the weak guide.}
The guide should be weaker than the student, but still competent.
At $\alpha=0.8$ on CIFAR-10, reducing the channel base from the student's
128 to 64 improves FID from $2.58$ to $1.60$, while substantially smaller
guides degrade rapidly (\swiftpanelref{fig:selection}{b}). The half-width guide is
also the best width across the tested $\alpha$ values on CIFAR-10 and
consistently outperforms wider alternatives on ImageNet. Guide training
length shows a similar pattern: once the guide crosses a competence
threshold, performance is relatively insensitive over a broad range
(\swiftpanelref{fig:selection}{c}). On ImageNet, freezing the guide earlier can even
help, with an earlier guide outperforming the default guide in our 50k re-evaluation;
the corresponding training trajectories are discussed in \Cref{app:imagenet-traj}.

\paragraph{A simple fixed recipe gives substantial gains.}
The full CIFAR-10 campaign also tests jointly trained guides, capacity
switches, and coefficients that change with training time or noise level
(\Cref{app:cifar,tab:cifar-summary}). Co-training provides no clear additional endpoint
benefit in the tested comparisons. Some early capacity switches and
late coefficient profiles give useful CIFAR-10 refinements, whereas other
schedules lose the gain; their effects depend on the guide and coefficient,
and many small differences remain unresolved by repetitions. We therefore
retain a frozen half-width guide and a fixed, moderate positive $\alpha$
as the simple default for the broader experiments.

\section{Conclusion}
\label{sec:conclusion}

The regression target has received comparatively little attention as a design choice in diffusion training, yet it turns out to be a simple and effective place to intervene. By carrying the clean target past itself and away from a frozen weak model, \swift{} moves a weak-to-strong correction from inference into training, leaving a single denoiser and the original sampler at deployment with only about $9\%$ additional main-training FLOPs. Crucially, the modified target is not itself a better estimator of the posterior mean: it is deliberately biased, yet finite training on it can produce a better final predictor and better samples. This suggests a broader view of the regression target as a design interface rather than a fixed quantity---and raises the question of what other useful corrections can be written into it.

\section*{Reproducibility statement}
\swift{} changes only the regression target: the objective is
\Cref{eq:swift-loss,eq:swift-target}, the frozen weak denoiser is constructed as
described in \Cref{sec:swift-method}, and everything else is the training loop
of the respective framework. \Cref{app:config} specifies guide architectures,
freezing points, coefficients, training budgets, and hardware for the
EDM, DiffusionBench, and RAEv2 models, and explains the velocity-target
implementation. \Cref{app:eval} specifies samplers, guidance and checkpoint
selection, and evaluation metrics, distinguishing re-evaluated pairs from
published baselines;
\Cref{app:compute} reports the GPU hours and the FLOPs accounting behind
\Cref{fig:overview,fig:flops}. \Cref{app:imagenet-full} gives all 19
ImageNet-256 configurations with and without guidance on all 15 metrics, and
\Cref{app:cifar} the CIFAR-10 experiments, including an account of all 156
launched (152 scored) arms (\Cref{app:arms}). The matched Gaussian-mixture
comparison is specified in \Cref{app:mog-experiment}, and all assumptions and proofs are in
\Cref{app:rkhs-theory}. Our implementation will be made publicly available
upon publication.

\section*{AI use statement}
We used large language models (LLMs) in preparing this paper. They edited the text for clarity; drafted parts of the text, including appendix sections, from our experiment logs and notes; generated LaTeX and TikZ code for tables and figures; computed summary statistics from our result files. The authors reviewed and revised all LLM-assisted content and take full responsibility for the paper.

\bibliographystyle{plainnat}
\bibliography{refs}

\clearpage
\renewcommand{\swiftboxbreakable}{true}
\beginappendix
\setcounter{theorem}{0}
\renewcommand{\thetheorem}{A.\arabic{theorem}}
\renewcommand{\theHtheorem}{appendix.\arabic{theorem}}
\setcounter{equation}{0}
\renewcommand{\theequation}{A.\arabic{equation}}
\renewcommand{\theHequation}{appendix.\arabic{equation}}
\tcolorboxenvironment{lemma}{
  enhanced, breakable,
  colback=metabg, colframe=rtpblue!60!black,
  boxrule=0.5pt, arc=2pt,
  left=6pt, right=6pt, top=4pt, bottom=4pt,
  before skip=8pt, after skip=8pt
}
\providecommand{\swift}{\textnormal{\textsc{Swift}}}

\paragraph{Overview.}
\Cref{app:rkhs-theory} proves the results of \Cref{sec:rkhs-theory}.
\Cref{app:guidance-comparison} extends the comparison with guidance methods
in \Cref{sec:swift-comparison}. \Cref{app:mog-experiment} gives a matched
Gaussian-mixture comparison in which the class-conditional posterior mean
is known exactly; the separate tree-shaped example is shown in
\Cref{fig:toy-guidance}. \Cref{app:setup} gives the three models, the
evaluation procedure, and the compute we used.
\Cref{app:imagenet-full} reports every ImageNet-256 result,
\Cref{app:cifar} the CIFAR-10 experiments, and \Cref{app:samples} additional
samples.

\section{Proofs and Extensions of the Finite-Time Analysis}
\label{app:rkhs-theory}

We first establish the target and training-response identities, then prove
\Cref{thm:rkhs-swift} through a general operator comparison.
The remaining subsections give the Gaussian-mixture posterior, verify
learnability for Gaussian kernels, and extend the analysis to narrower
guides and finite-sample fixed-feature training.

Throughout, the noise level $t$ and the data law are those of
\Cref{ass:swift-setting}. We suppress the noise argument of
$\mathcal D^\star$ and $\mathcal D_{\mathrm{weak}}$.
The space $\mathcal H_2$, its inner product, the kernel operator
$\boldsymbol{\mathcal L}$, and the response $\mathbf H_s$ are defined in
\Cref{sec:rkhs-theory}. Unsubscripted function norms and inner products
refer to $\mathcal H_2$. Write $\bm d_c=\mathcal D_{\mathrm{base}}$ and
$\bm d_\alpha=\mathcal D_{\mathrm{SWIFT}}^{(\alpha)}$.
All regression targets and predictors are square integrable.

\subsection{Target identity and the quantity being compared}
\label{app:target-identity}

Write
\begin{equation*}
  \begin{aligned}
    \bm{\xi}&=\bm{x}_0-\mathcal D^\star(\bm{x}_t),
    &\mathbb E[\bm{\xi}\mid \bm{x}_t]&=\bm0,\\
    \bm y_{\alpha}&=(1+\alpha)\bm x_0-\alpha\mathcal D_{\mathrm{weak}}(\bm x_t).
  \end{aligned}
\end{equation*}

\begin{lemma}[\swift{} target decomposition]
\label{lem:target-decomposition}
For every fixed square-integrable guide $\mathcal D_{\mathrm{weak}}$,
\begin{align}
  \bm y_{\alpha}
  &=\mathcal D^\star+\alpha(\mathcal D^\star-\mathcal D_{\mathrm{weak}})+(1+\alpha)\bm{\xi},\label{eq:app-target-decomposition}\\
  \bm g_{\alpha}(\bm{x}_t):=\mathbb E[\bm y_{\alpha}\mid \bm{x}_t]
  &=\mathcal D^\star(\bm{x}_t)+\alpha\bigl(\mathcal D^\star(\bm{x}_t)-\mathcal D_{\mathrm{weak}}(\bm{x}_t)\bigr).
  \label{eq:app-target-conditional-mean}
\end{align}
Moreover,
\begin{equation}
  \operatorname{Cov}(\bm y_{\alpha}\mid \bm{x}_t)
  =(1+\alpha)^2\operatorname{Cov}(\bm{x}_0\mid \bm{x}_t).
  \label{eq:app-target-conditional-covariance}
\end{equation}
For every square-integrable predictor $\bm d$,
\begin{equation}
  \mathbb E\|\bm d(\bm x_t)-\bm y_{\alpha}\|_2^2
  =\|\bm d-\bm g_{\alpha}\|_{\mathcal H_2}^2
   +(1+\alpha)^2\mathbb E\|\bm{x}_0-\mathcal D^\star(\bm{x}_t)\|_2^2.
  \label{eq:app-target-pythagoras}
\end{equation}
Consequently, the unrestricted population minimizer of the \swift{} regression
loss is $\bm g_{\alpha}$, and its posterior-mean risk is
$\alpha^2\|\mathcal D_{\mathrm{weak}}-\mathcal D^\star\|_{\mathcal H_2}^2$.
\end{lemma}

\begin{proof}
Substituting $\bm{x}_0=\mathcal D^\star+\bm{\xi}$ gives
\Cref{eq:app-target-decomposition}; conditioning on $\bm{x}_t$ gives
\Cref{eq:app-target-conditional-mean}.  Since $\mathcal D_{\mathrm{weak}}(\bm{x}_t)$ is fixed after
conditioning on $\bm{x}_t$, the same decomposition gives
\Cref{eq:app-target-conditional-covariance}.  Moreover,
$\bm y_{\alpha}-\bm g_{\alpha}=(1+\alpha)\bm{\xi}$ and
$\bm d-\bm g_{\alpha}$ is $\bm x_t$-measurable, so
\begin{equation*}
  \mathbb E\bigl[(\bm d-\bm g_{\alpha})^{\!\top}\bm{\xi}\bigr]
  =\mathbb E\bigl[(\bm d-\bm g_{\alpha})^{\!\top}
    \mathbb E[\bm{\xi}\mid\bm x_t]\bigr]=0.
\end{equation*}
Expanding the square proves \Cref{eq:app-target-pythagoras}; the remaining
claims follow immediately.
\end{proof}

The posterior-mean risk used in \Cref{thm:rkhs-swift} is also exactly
the reducible part of clean-sample denoising error: for every $\bm d$,
\begin{equation*}
  \mathbb E\|\bm d(\bm x_t)-\bm x_0\|_2^2
  =\|\bm d-\mathcal D^\star\|_{\mathcal H_2}^2
   +\mathbb E\|\bm{x}_0-\mathcal D^\star(\bm{x}_t)\|_2^2.
\end{equation*}
Thus posterior-mean risk and population clean-target MSE induce the same
ordering of predictors. The target shift is beneficial only through the
training response; it does not improve the unrestricted population optimum.

\subsection{Kernel training response}
\label{app:rkhs-flow}
\label{app:kernel-primer}

The following lemma establishes the prediction-space equation and the
operator properties needed for the spectral proof, without
requiring finite-dimensional features or $\bm g$ to belong to the RKHS.

\begin{lemma}[Prediction-space kernel flow]
\label{lem:kernel-flow}
Let $\bm y$ be a square-integrable vector regression target with
conditional mean $\bm g(\bm x_t)=\mathbb E[\bm y\mid\bm x_t]$.
Population gradient flow for the objective
$\frac12\mathbb E\|\bm d(\bm x_t)-\bm y\|_2^2$ in
$\mathcal H_{\mathsf K}^{d}$ from the zero predictor satisfies, in prediction
space,
\begin{equation}
  \frac{d}{ds}\bm d_s=-\boldsymbol{\mathcal L}(\bm d_s-\bm g),
  \qquad
  \bm d_0=\bm0,
  \label{eq:app-flow-ode}
\end{equation}
and therefore
\begin{equation*}
  \bm d_s=\mathbf H_s\bm g,
  \qquad
  \mathbf H_s=\mathbf I-e^{-s\boldsymbol{\mathcal L}}.
\end{equation*}
The conditional variance of $\bm y$ affects the value of the loss but not this
population trajectory.
\end{lemma}

\begin{proof}
Conditioning on $\bm{x}_t$ gives the orthogonal decomposition
\begin{equation}
  \mathbb E\|\bm d(\bm x_t)-\bm y\|_2^2
  =\|\bm d-\bm g\|_{\mathcal H_2}^2
   +\mathbb E\|\bm y-\bm g(\bm x_t)\|_2^2.
  \label{eq:app-regression-decomposition}
\end{equation}
After multiplying \Cref{eq:app-regression-decomposition} by $1/2$, the
second term is independent of $\bm d$.  If
$\mathbf J:\mathcal H_{\mathsf K}^{d}\to\mathcal H_2$ sends an RKHS function to its $L^2$ prediction class,
then $\mathbf J$ is bounded because
$\|\mathbf J\bm f\|_{\mathcal H_2}^2
\leq\sup_{\bm x}\mathsf K(\bm x,\bm x)
\|\bm f\|_{\mathcal H_{\mathsf K}^{d}}^2$.
The adjoint $\mathbf J^*:\mathcal H_2\to\mathcal H_{\mathsf K}^d$
satisfies $\langle\mathbf J^*\bm v,\bm u\rangle_{\mathcal H_{\mathsf K}^d}
=\langle\bm v,\mathbf J\bm u\rangle_{\mathcal H_2}$.
Testing this identity against a kernel section in coordinate $i$ gives
\begin{equation*}
  (\mathbf J^*\bm v)_i(\bm x)
  =\int\mathsf K(\bm x,\bm x')v_i(\bm x')\,dP_{\bm x_t}(\bm x').
\end{equation*}
Thus $\boldsymbol{\mathcal L}=\mathbf J\mathbf J^*$, which is bounded,
self-adjoint, and positive, since
$\langle\bm v,\boldsymbol{\mathcal L}\bm v\rangle
=\|\mathbf J^*\bm v\|_{\mathcal H_{\mathsf K}^d}^2\geq0$.
Boundedness of the kernel and the fact that $P_{\bm x_t}$ is a probability
measure also give
$\iint|\mathsf K(\bm x,\bm x')|^2\,dP_{\bm x_t}(\bm x)\,dP_{\bm x_t}(\bm x')<\infty$.
The scalar integral operator is therefore Hilbert--Schmidt and compact;
its finite direct sum $\boldsymbol{\mathcal L}$ is compact as well.
For the RKHS-valued trajectory $\bm u_s$, the gradient is
$\mathbf J^*(\mathbf J\bm u_s-\bm g)$. Hence
\begin{equation*}
  \frac{d}{ds}\bm u_s
  =-\mathbf J^*(\mathbf J\bm u_s-\bm g),
  \qquad \bm u_0=\bm0.
\end{equation*}
Because $\mathbf J^*\mathbf J$ is bounded, this linear Hilbert-space ODE
has a unique global solution. Applying $\mathbf J$ to
the ODE and writing $\bm d_s=\mathbf J\bm u_s$ proves
\Cref{eq:app-flow-ode}. This argument requires only
$\bm g\in\mathcal H_2$, not membership in the RKHS.
Since $\boldsymbol{\mathcal L}$ is bounded and
self-adjoint, functional calculus gives
\begin{equation*}
  \bm d_s=e^{-s\boldsymbol{\mathcal L}}\bm d_0
      +\int_0^se^{-(s-u)\boldsymbol{\mathcal L}}\boldsymbol{\mathcal L}\bm g\,du
      =(\mathbf I-e^{-s\boldsymbol{\mathcal L}})\bm g.
\end{equation*}
\end{proof}

For finite features $\mathsf K(\bm x,\bm x')=\bm\phi(\bm x)^\top\bm\phi(\bm x')$,
define $\mathbf V:\mathbb R^p\to L^2(P_{\bm x_t})$ by
$(\mathbf V\bm a)(\bm x)=\bm\phi(\bm x)^\top\bm a$.
Then $\mathbf V^*f=\mathbb E[\bm\phi(\bm x_t)f(\bm x_t)]$,
$\boldsymbol{\mathcal L}_0=\mathbf V\mathbf V^*$, and
$\mathbf V^*\mathbf V=\mathbb E[\bm\phi(\bm x_t)\bm\phi(\bm x_t)^\top]$.
For any eigenvalue $\kappa>0$, $\mathbf V$ and
$\kappa^{-1}\mathbf V^*$ are inverse maps between the corresponding
eigenspaces. Hence the two operators have the same nonzero eigenvalues,
including multiplicities.

Each $\mathbf H_s$ is self-adjoint with spectral multipliers in $[0,1]$;
responses at different times commute because they are functions of the
same operator. For the weighted diffusion loss, the normalized time is
$s=2\lambda(t)\tau$, where $\tau$ is time under the unnormalized objective.
This rescaling is the same for baseline and \swift{} training.

\subsection{A general commuting-response theorem}
\label{app:operator-theorem}

The following comparison applies to linear training responses and
also yields the spectrally nested guide result in
\swiftfullref{Corollary}{cor:reduced-rank-guide}.

Let $\mathbf H_c,\mathbf H_g:\mathcal H_2\to\mathcal H_2$ denote the population
responses of the main denoiser and the guide: training on a target with
conditional mean $\bm g$ returns $\mathbf H_c\bm g$ under the main-denoiser
training dynamics and $\mathbf H_g\bm g$ under the guide training dynamics.
We require
\begin{equation}
  \begin{aligned}
    \mathbf H_c&=\mathbf H_c^*,
    &\mathbf H_g&=\mathbf H_g^*,\\
    \mathbf0&\preceq \mathbf H_c\preceq \mathbf I,
    &\mathbf0&\preceq \mathbf H_g\preceq \mathbf I,\\
    \mathbf H_c\mathbf H_g&=\mathbf H_g\mathbf H_c.
  \end{aligned}
  \label{eq:app-operator-assumptions}
\end{equation}
The bounds are in the quadratic-form sense. Self-adjointness and
commutation give a common spectral representation, with each response
multiplying target components by a factor in $[0,1]$.

\Needspace{20\baselineskip}
\begin{theorem}[Commuting-response improvement]
\label{thm:operator-improvement}
Suppose \Cref{eq:app-operator-assumptions} holds.  Let
\begin{equation*}
  \begin{aligned}
    \bm d_c&=\mathbf H_c\mathcal D^\star,
    &\mathcal D_{\mathrm{weak}}&=\mathbf H_g\mathcal D^\star,\\
    \bm d_{\alpha}&=\mathbf H_c\bigl((1+\alpha)\mathcal D^\star-\alpha\mathcal D_{\mathrm{weak}}\bigr).
  \end{aligned}
\end{equation*}
Define $R(\alpha)=\|\bm d_\alpha-\mathcal D^\star\|^2$ and
\begin{equation*}
  \begin{aligned}
    \bm e&=(\mathbf I-\mathbf H_c)\mathcal D^\star,
    &\bm q&=\mathbf H_c(\mathbf I-\mathbf H_g)\mathcal D^\star,\\
    A&=\langle\bm e,\bm q\rangle,
    &C&=\|\bm q\|^2.
  \end{aligned}
\end{equation*}
Then
\begin{equation}
  \bm d_{\alpha}-\mathcal D^\star=-\bm e+\alpha\bm q,
  \qquad
  R(\alpha)=R(0)-2A\alpha+C\alpha^2,
  \label{eq:app-operator-risk}
\end{equation}
and
\begin{equation}
  A
  =\left\langle
  \mathcal D^\star,(\mathbf I-\mathbf H_c)\mathbf H_c(\mathbf I-\mathbf H_g)\mathcal D^\star
  \right\rangle
  \geq0.
  \label{eq:app-operator-alignment}
\end{equation}
If $A>0$, then $C>0$, the unique minimizer is
$\alpha_{\mathrm{PM}}^\star=A/C>0$, and
\begin{equation*}
  R(\alpha)<R(0)
  \quad\Longleftrightarrow\quad
  0<\alpha<2\alpha_{\mathrm{PM}}^\star.
\end{equation*}
\end{theorem}

\begin{proof}
Using $\mathcal D_{\mathrm{weak}}=\mathbf H_g\mathcal D^\star$ and linearity of $\mathbf H_c$,
\begin{align*}
  \bm d_{\alpha}-\mathcal D^\star
  &=\mathbf H_c\mathcal D^\star+\alpha\mathbf H_c(\mathbf I-\mathbf H_g)\mathcal D^\star-\mathcal D^\star\\
  &=-(\mathbf I-\mathbf H_c)\mathcal D^\star+\alpha\mathbf H_c(\mathbf I-\mathbf H_g)\mathcal D^\star\\
  &=-\bm e+\alpha\bm q.
\end{align*}
Taking the squared norm gives \Cref{eq:app-operator-risk}.
By the spectral calculus, the commuting positive operators have commuting
positive square roots. Thus
\begin{equation*}
  A=\left\|
    (\mathbf I-\mathbf H_g)^{1/2}
    [\mathbf H_c(\mathbf I-\mathbf H_c)]^{1/2}\mathcal D^\star
    \right\|^2\geq0,
\end{equation*}
proving \Cref{eq:app-operator-alignment}.
When $A>0$, $\bm q\neq\bm0$ and hence $C>0$. The remaining
claims follow by minimizing the strictly convex quadratic
$R(0)-2A\alpha+C\alpha^2$.
\end{proof}

\subsection{Spectral comparison and proof of \texorpdfstring{\Cref{thm:rkhs-swift}}{the main theorem}}
\label{app:proof-rkhs-theorem}

Let $\boldsymbol{\mathcal L}_0$ be the scalar version of
$\boldsymbol{\mathcal L}$ on $L^2(P_{\bm x_t})$.
Compactness and positivity give orthonormal eigenfunctions $u_j$ with
positive eigenvalues $\kappa_j$. Decompose
\begin{equation*}
  \mathcal D^\star=\bm d_0^\star+\sum_j\bm d_j^\star u_j,
  \qquad \bm d_0^\star\in\ker\boldsymbol{\mathcal L},
  \qquad \bm d_j^\star=\mathbb E[\mathcal D^\star(\bm x_t)u_j(\bm x_t)].
\end{equation*}
The series converges in $\mathcal H_2$; the list may be finite.
Set $h_j=1-e^{-T\kappa_j}$, $h_{g,j}=1-e^{-T_g\kappa_j}$, and
$b_j=h_j(1-h_{g,j})$.
The response lemma and target identity give the baseline coefficient
$h_j\bm d_j^\star$ and the \swift{} coefficient
$\bm d_{\alpha,j}=h_j[1+\alpha(1-h_{g,j})]\bm d_j^\star$.
Consequently,
\begin{equation}
  \bm d_{\alpha,j}-\bm d_j^\star
  =\left[
     \underbrace{-e^{-T\kappa_j}}_{\text{finite-time residual}}
     +\underbrace{\alpha b_j}_{\text{target correction}}
    \right]\bm d_j^\star.
  \label{eq:app-scalar-mechanism}
\end{equation}
All three trained predictors have zero projection onto
$\ker\boldsymbol{\mathcal L}$, so their error in that subspace is
$-\bm d_0^\star$ and receives no correction.
These are function-space coefficients, not pointwise scaling
of image pixels.

\begin{proof}[Proof of \Cref{thm:rkhs-swift}]
\swiftfullref{Lemma}{lem:target-decomposition} and \swiftfullref{Lemma}{lem:kernel-flow}
yield the three predictor
expressions in the main theorem. Apply \Cref{thm:operator-improvement}
with $\mathbf H_c=\mathbf H_T$ and $\mathbf H_g=\mathbf H_{T_g}$.
The responses satisfy its assumptions because they are positive
contractions and commute. In the notation of that theorem, Parseval's
identity gives
\begin{equation}
  \begin{aligned}
    A&=\sum_j\|\bm d_j^\star\|_2^2 e^{-T\kappa_j}b_j,\\
    C&=\sum_j\|\bm d_j^\star\|_2^2 b_j^2.
  \end{aligned}
  \label{eq:app-rkhs-alignment}
\end{equation}
Both sums converge because the spectral multipliers are bounded and
$\sum_j\|\bm d_j^\star\|_2^2<\infty$.
The condition $\mathbf H_T\mathcal D^\star\ne\bm0$ means that at least
one positive-eigenvalue coefficient is nonzero. Since $T>0$ and $T_g<T$,
$e^{-T\kappa_j}>0$ and $b_j>0$ for every such mode, so $A>0$ and $C>0$.
The operator theorem therefore gives
\begin{equation*}
  R(\alpha)<R(0)\quad\Longleftrightarrow\quad 0<\alpha<2A/C.
\end{equation*}
Taking $\alpha^\star=2A/C$ proves the improvement claim. The guide is
strictly weaker because
\begin{equation*}
  \|\mathcal D_{\mathrm{weak}}-\mathcal D^\star\|^2-R(0)
  =\sum_j\|\bm d_j^\star\|_2^2
       (e^{-2T_g\kappa_j}-e^{-2T\kappa_j})>0.
\end{equation*}
The summands are positive on every active mode, completing the proof.
\end{proof}

For completeness, the same expansion gives the exact modal risk
\begin{equation}
  R(\alpha)=\|\bm d_0^\star\|^2
   +\sum_j\|\bm d_j^\star\|_2^2(-e^{-T\kappa_j}+\alpha b_j)^2.
  \label{eq:app-modal-risk}
\end{equation}
Its minimizing coefficient is $\alpha_{\mathrm{PM}}^\star=A/C$;
the beneficial interval extends to twice this value.
If the nonzero-predictor condition fails, both population flows remain
zero. More generally, when the main-model response is an exact
orthogonal projection, its residual is orthogonal to every learned
correction, giving $A=0$ and $R(\alpha)=R(0)+\alpha^2 C$.
Thus restricted expressivity alone does not establish the improvement;
finite-time partial fitting supplies the positive alignment.

\subsection{Gaussian-mixture posterior and learnability}
\label{app:mog-posterior}

The main theorem applies to any square-integrable noisy--clean pair.
As a concrete diffusion example, take
\begin{equation}
  \bm x_0\sim\sum_{k=1}^{K}\pi_k
                 \mathcal N(\bm\mu_k,\boldsymbol\Sigma),
  \qquad
  \bm x_t=\bm x_0+\bm\varepsilon,\quad
  \bm\varepsilon\sim\mathcal N(\bm0,\sigma_t^2\mathbf I_d).
  \label{eq:main-mog}
\end{equation}
Here $K<\infty$, $\pi_k>0$, $\sum_k\pi_k=1$,
$\boldsymbol\Sigma\succ\mathbf0$ is shared across components, and
$\bm\varepsilon$ is independent of $\bm x_0$. The component means are
arbitrary. This distribution satisfies the data and moment conditions in
\Cref{ass:swift-setting}.

Let $I\in\{1,\ldots,K\}$ denote the mixture component.  For any component
index $k\in\{1,\ldots,K\}$, conditional on $I=k$,
\begin{equation*}
  \begin{aligned}
    \bm{x}_0\mid I=k&\sim\mathcal N(\bm\mu_k,\boldsymbol{\Sigma}),\\
    \bm x_t\mid I=k&\sim\mathcal N(\bm\mu_k,\boldsymbol{\Sigma}_{\rm obs}),\\
    \boldsymbol{\Sigma}_{\rm obs}&=\boldsymbol{\Sigma}+\sigma_t^2\mathbf I_d.
  \end{aligned}
\end{equation*}

\begin{lemma}[Gaussian-mixture posterior]
\label{lem:mog-posterior}
Under \Cref{eq:main-mog}, define
\begin{equation*}
  \begin{aligned}
    \ell_k(\bm{x}_t)
    &=\bm\mu_k^\top\boldsymbol{\Sigma}_{\rm obs}^{-1}\bm x_t
      -\tfrac12\bm\mu_k^\top\boldsymbol{\Sigma}_{\rm obs}^{-1}\bm\mu_k,\\
    w_k(\bm{x}_t)
    &=\frac{\pi_k e^{\ell_k(\bm{x}_t)}}
      {\sum_{j=1}^{K}\pi_j e^{\ell_j(\bm{x}_t)}},\\
    \overline{\bm\mu}(\bm{x}_t)&=\sum_{k=1}^{K}w_k(\bm{x}_t)\bm\mu_k.
  \end{aligned}
\end{equation*}
Then $w_k(\bm x_t)=\mathbb P(I=k\mid\bm x_t)$ and
\begin{equation}
  \mathcal D^\star(\bm x_t)=\mathbb E[\bm x_0\mid\bm x_t]
  =\mathbf B\bm x_t+(\mathbf I_d-\mathbf B)\overline{\bm\mu}(\bm x_t),
  \qquad
  \mathbf B=\boldsymbol{\Sigma}\boldsymbol{\Sigma}_{\rm obs}^{-1}.
  \label{eq:app-posterior-mean}
\end{equation}
\end{lemma}

\begin{proof}
For fixed $k$, $(\bm{x}_0,\bm{x}_t)$ is jointly Gaussian with
$\operatorname{Cov}(\bm x_0,\bm x_t\mid I=k)=\boldsymbol{\Sigma}$ and
$\operatorname{Cov}(\bm x_t\mid I=k)=\boldsymbol{\Sigma}_{\rm obs}$.  Gaussian
conditioning gives
\begin{align}
  \mathbb E[\bm x_0\mid\bm x_t,I=k]
  &=\bm\mu_k+\boldsymbol{\Sigma}\boldsymbol{\Sigma}_{\rm obs}^{-1}(\bm x_t-\bm\mu_k)\notag
\\
  &=\mathbf B\bm x_t+(\mathbf I_d-\mathbf B)\bm\mu_k.
  \label{eq:app-component-posterior}
\end{align}
Bayes' rule yields
\begin{align*}
  \mathbb P(I=k\mid\bm x_t)
  &\propto
  \pi_k\exp\!\left[-\frac12(\bm{x}_t-\bm\mu_k)^\top
  \boldsymbol{\Sigma}_{\rm obs}^{-1}(\bm x_t-\bm\mu_k)\right]\\
  &\propto
  \pi_k\exp\!\left[\bm\mu_k^\top\boldsymbol{\Sigma}_{\rm obs}^{-1}\bm x_t
  -\frac12\bm\mu_k^\top\boldsymbol{\Sigma}_{\rm obs}^{-1}\bm\mu_k\right],
\end{align*}
where the common term
$-\bm x_t^\top\boldsymbol{\Sigma}_{\rm obs}^{-1}\bm x_t/2$ cancels.  Averaging
\Cref{eq:app-component-posterior} over these weights proves
\Cref{eq:app-posterior-mean}.

\end{proof}

\paragraph{Learnability with a Gaussian kernel.}
\label{app:gaussian-kernel}

Consider the Gaussian/RBF kernel
\begin{equation}
  \mathsf K_\ell(\bm{x}_t,\bm{x}_t')
  =\exp\!\left(-\frac{\|\bm{x}_t-\bm{x}_t'\|_2^2}{2\ell^2}\right),
  \qquad \ell>0.
  \label{eq:app-rbf-kernel}
\end{equation}

\begin{proposition}[Injectivity of the Gaussian-kernel operator]
\label{prop:rbf-injective}
For the mixture in \Cref{eq:main-mog}, the marginal $P_{\bm x_t}$ has a smooth,
strictly positive density on $\mathbb R^d$.  The integral operator associated
with \Cref{eq:app-rbf-kernel} is injective on
$L^2(P_{\bm x_t};\mathbb R^d)$. Moreover, $\mathcal D^\star\neq\bm0$.
Therefore, for every $T>0$, ordinary training produces
a nonzero predictor in $\mathcal H_2$, as required by the nonzero-predictor condition of
\Cref{thm:rkhs-swift}.
\end{proposition}

\begin{proof}
Each conditional law $\bm x_t\mid I=k$ is Gaussian with covariance
$\boldsymbol{\Sigma}_{\rm obs}\succ\mathbf0$, so its density is smooth and strictly positive; the same is
true of their positive mixture.

It suffices to prove injectivity for one scalar coordinate.  Let
$f\in L^2(P_{\bm x_t})$ and suppose $\boldsymbol{\mathcal L}_0f=0$.  Then
$\langle f,\boldsymbol{\mathcal L}_0f\rangle_{L^2(P_{\bm x_t})}=0$.  The signed measure
$d\eta(\bm{x}_t)=f(\bm{x}_t)dP_{\bm x_t}(\bm{x}_t)$ has finite total variation by Cauchy--Schwarz.
Use the convention
$\widehat\eta(\bm\omega)=\int e^{-i\bm\omega^\top\bm x}\,d\eta(\bm x)$.
The Gaussian Fourier identity is
\begin{equation*}
  e^{-\|\bm z\|_2^2/(2\ell^2)}
  =c_{d,\ell}\int_{\mathbb R^d}
    e^{i\bm\omega^\top\bm z}e^{-\ell^2\|\bm\omega\|_2^2/2}\,d\bm\omega,
  \qquad c_{d,\ell}=\left(\frac{\ell}{\sqrt{2\pi}}\right)^d.
\end{equation*}
The Gaussian multiplier is integrable and $\eta$ has finite total
variation, so Fubini's theorem applies after integrating this identity
against $d\eta(\bm x)\,d\eta(\bm x')$ with $\bm z=\bm x-\bm x'$.
It gives
\begin{equation*}
  \langle f,\boldsymbol{\mathcal L}_0f\rangle
  =c_{d,\ell}\int_{\mathbb R^d}
  |\widehat\eta(\bm{\omega})|^2
  e^{-\ell^2\|\bm{\omega}\|_2^2/2}\,d\bm{\omega},
\end{equation*}
The Gaussian Fourier multiplier is strictly
positive everywhere, so zero energy implies that $\widehat\eta=0$
Lebesgue-almost everywhere.  The Fourier transform of a finite signed measure
is continuous, hence it vanishes everywhere; uniqueness of Fourier transforms
of finite measures gives $\eta=0$. Since $d\eta=f\,dP_{\bm x_t}$,
$f=0$ $P_{\bm x_t}$-almost surely. Applying this coordinatewise proves
injectivity.

Finally, the tower property and independence of the corruption noise give
\begin{equation*}
  \mathbb E[\mathcal D^\star(\bm{x}_t)\bm{x}_t^\top]
  =\mathbb E[\bm x_0\bm x_t^\top]
  =\mathbb E[\bm x_0\bm x_0^\top]\succeq\boldsymbol{\Sigma}\succ\mathbf0.
\end{equation*}
Thus $\mathcal D^\star$ cannot vanish in $\mathcal H_2$.
Injectivity gives $\ker\boldsymbol{\mathcal L}=\{\bm0\}$.
For every $T>0$, spectral calculus implies
\begin{equation*}
  \ker\mathbf H_T
  =\ker(\mathbf I-e^{-T\boldsymbol{\mathcal L}})
  =\ker\boldsymbol{\mathcal L}
  =\{\bm0\}.
\end{equation*}
Hence $\mathcal D_{\mathrm{base}}
=\mathbf H_T\mathcal D^\star\neq\bm0$ in $\mathcal H_2$,
verifying the nonzero-predictor condition of
\Cref{thm:rkhs-swift}.
\end{proof}

The Gaussian kernel is injective, so every nonzero posterior mean has
learnable signal. The theorem therefore does not require a deliberately
underpowered function class.

\subsection{Reduced-rank and spectrally nested guides}
\label{app:reduced-rank-guide}

A narrower guide can be modeled by retaining a subset of the main model's
learning directions. The projection below specifies the spectral nesting
needed for this interpretation.

Let $\mathbf P_g$ be an orthogonal projection on $\mathcal H_2$ satisfying
\begin{equation*}
  \mathbf P_g\boldsymbol{\mathcal L}=\boldsymbol{\mathcal L}\mathbf P_g.
\end{equation*}
For example, $\mathbf P_g$ may retain any subset of the eigenfunctions of
$\boldsymbol{\mathcal L}$; if it retains finitely many, the guide response is finite rank.
For $0<T<\infty$ and $0\leq T_g<\infty$, define
\begin{equation}
  \mathbf H_c=\mathbf I-e^{-T\boldsymbol{\mathcal L}},
  \qquad
  \mathbf H_g=\mathbf P_g(\mathbf I-e^{-T_g\boldsymbol{\mathcal L}}).
  \label{eq:app-reduced-guide-response}
\end{equation}

\begin{corollary}[Reduced-rank guide]
\label{cor:reduced-rank-guide}
The operators in \Cref{eq:app-reduced-guide-response} satisfy
\Cref{eq:app-operator-assumptions}.
If $\mathbf H_c\mathcal D^\star\ne\bm0$, then $A>0$ and $C>0$ in
\Cref{thm:operator-improvement}. Consequently, the \swift{} predictor
has smaller posterior-mean risk than the baseline exactly for
$0<\alpha<2A/C$.
\end{corollary}

\begin{proof}
Because $\mathbf P_g$ is an orthogonal projection commuting with
$\boldsymbol{\mathcal L}$, $\mathbf H_g$ is a positive self-adjoint
contraction and commutes with $\mathbf H_c$.
Moreover,
\begin{equation*}
  \mathbf I-\mathbf H_g-e^{-T_g\boldsymbol{\mathcal L}}
  =(\mathbf I-\mathbf P_g)(\mathbf I-e^{-T_g\boldsymbol{\mathcal L}})
  \succeq\mathbf0.
\end{equation*}
All factors commute, so the alignment obeys
\begin{equation*}
  A\geq\left\langle\mathcal D^\star,
    e^{-(T+T_g)\boldsymbol{\mathcal L}}
    (\mathbf I-e^{-T\boldsymbol{\mathcal L}})\mathcal D^\star
    \right\rangle>0.
\end{equation*}
The strict inequality follows because the multiplier
$e^{-(T+T_g)\kappa}(1-e^{-T\kappa})$ is positive for every $\kappa>0$,
and $\mathbf H_c\mathcal D^\star\ne\bm0$ ensures a nonzero component
in such a direction. Apply \Cref{thm:operator-improvement}.
\end{proof}

This corollary is the rigorous counterpart of a reduced-width guide only when
the narrower model corresponds to a nested spectral subspace.  Arbitrary
architectural width reduction need not preserve the eigenfunctions and is not
covered without an additional approximation argument.

\Needspace{8\baselineskip}
\subsection{Exact label-affine response for finite-step training}
\label{app:finite-step-affine}

We next separate a second exact statement from the population theorem.  It
applies to finite data and a finite number of optimization steps, provided the
features are fixed.  Let $\bm x_{t,1},\ldots,\bm x_{t,n}$ be training inputs,
let $\boldsymbol{\Phi}\in\mathbb R^{n\times p}$ be a fixed design matrix, and
train the vector-valued linear predictor $\boldsymbol{\Phi}\boldsymbol{\Theta}$
with full-batch gradient descent:
\begin{equation*}
  \boldsymbol{\Theta}_{s+1}
  =\boldsymbol{\Theta}_s-\frac{\gamma_s}{n}\boldsymbol{\Phi}^\top
   (\boldsymbol{\Phi}\boldsymbol{\Theta}_s-\mathbf Y),
  \qquad s=0,\ldots,S-1,
\end{equation*}
where $\mathbf Y\in\mathbb R^{n\times d}$ is the label matrix.
The initialization $\boldsymbol{\Theta}_0$ and step sizes $\gamma_s$ are fixed
independently of the labels and held the same across target choices. No convergence
assumption is made.

\begin{proposition}[Finite-step label affinity]
\label{prop:finite-step-affinity}
For every $S\geq0$, there exist $\mathbf U_S\in\mathbb R^{p\times d}$ and a
linear operator
$\mathbf A_S:\mathbb R^{n\times d}\to\mathbb R^{p\times d}$, both
independent of $\mathbf Y$, such that
\begin{equation}
  \boldsymbol{\Theta}_S=\mathbf U_S+\mathbf A_S(\mathbf Y).
  \label{eq:app-parameter-affinity}
\end{equation}
Consequently, predictions under any fixed bounded linear evaluation map into a
Hilbert prediction space are affine functions of the labels.
\end{proposition}

\begin{proof}
Set $\mathbf Q_s=\mathbf I_p-
\gamma_s\boldsymbol{\Phi}^\top\boldsymbol{\Phi}/n$.  Then
\begin{equation*}
  \boldsymbol{\Theta}_{s+1}=\mathbf Q_s\boldsymbol{\Theta}_s
  +\frac{\gamma_s}{n}\boldsymbol{\Phi}^\top\mathbf Y.
\end{equation*}
The assertion holds at $s=0$ with
$\mathbf U_0=\boldsymbol{\Theta}_0$ and $\mathbf A_0=\mathbf0$.  If it holds at
step $s$, it holds at step $s+1$ with
\begin{equation*}
  \mathbf U_{s+1}=\mathbf Q_s\mathbf U_s,
  \qquad
  \mathbf A_{s+1}(\mathbf Y)=\mathbf Q_s\mathbf A_s(\mathbf Y)
  +\frac{\gamma_s}{n}\boldsymbol{\Phi}^\top\mathbf Y.
\end{equation*}
Induction proves \Cref{eq:app-parameter-affinity}.  Evaluation is linear
in $\boldsymbol{\Theta}_S$, so it preserves affinity.
\end{proof}

The same proof covers mini-batch schedules fixed in advance, fixed linear
preconditioners, weight decay, and momentum after augmenting the state.  It
does not cover label-dependent feature learning without further assumptions.
In particular, the fixed-feature response has the label-affine form
assumed in \swiftfullref{Proposition}{prop:finite-sample-variance}.

\subsection{Finite-sample bias--variance extension}
\label{app:finite-sample}

For finite data, target extrapolation also amplifies the conditional
label noise. The resulting risk depends on both residual alignment and
variance inflation. The following result is exact under conditional
centering, which holds in the independent-data example below but is
not automatic for a guide trained on the same clean labels.

Let $\mathcal F$ be a sigma-algebra with respect to which the noisy
training inputs, the response map, and the guide are measurable.
It represents the information held fixed when averaging the remaining
label randomness.
Let $\mathbf M$ be the matrix of stacked posterior means and write the
clean labels as
\begin{equation*}
  \mathbf M
  =\bigl(\mathcal D^\star(\bm x_{t,1}),\ldots,
    \mathcal D^\star(\bm x_{t,n})\bigr)^\top\in\mathbb R^{n\times d},
  \qquad
  \mathbf X_0=\mathbf M+\mathbf E,
  \qquad \mathbb E[\mathbf E\mid\mathcal F]=\mathbf0.
\end{equation*}
For example, the displayed conditional-centering assumption holds for
i.i.d. training pairs when the guide is trained on an independent dataset
and the response map and initialization depend only on the noisy inputs
and independent algorithmic randomness, with $\mathcal F$ generated by
these noisy inputs, auxiliary data, and randomness.
With $\mathbf Y_g$ the stacked guide predictions and
$\boldsymbol{\Delta}_g=\mathbf M-\mathbf Y_g$, the \swift{} labels are
\begin{equation*}
  \mathbf Y_{\alpha}=(1+\alpha)\mathbf X_0-\alpha\mathbf Y_g
  =\mathbf M+\alpha\boldsymbol{\Delta}_g+(1+\alpha)\mathbf E.
\end{equation*}

Let the trained predictor in a Hilbert space $\mathcal Y$ have the exact
label-affine form $\bm a+\boldsymbol{\mathcal G}\mathbf Y$, where
$\bm a$ and the bounded linear operator $\boldsymbol{\mathcal G}$ are
$\mathcal F$-measurable, and $\mathcal Y$ is separable.
Let $\bm d^\star_{\mathrm{ev}}\in\mathcal Y$ be a fixed
(or $\mathcal F$-measurable) posterior-mean target on the evaluation
distribution. Assume $\bm a$, $\boldsymbol{\mathcal G}\mathbf M$,
$\boldsymbol{\mathcal G}\boldsymbol{\Delta}_g$,
$\bm d^\star_{\mathrm{ev}}$, and $\boldsymbol{\mathcal G}\mathbf E$
are square-integrable $\mathcal Y$-valued random variables, and set
\begin{equation*}
  \begin{aligned}
    \bm e&=\bm d^\star_{\mathrm{ev}}
      -(\bm a+\boldsymbol{\mathcal G}\mathbf M),
    &\bm q&=\boldsymbol{\mathcal G}\boldsymbol{\Delta}_g,\\
    \bm{\eta}&=\boldsymbol{\mathcal G}\mathbf E,
    &A&=\langle\bm e,\bm q\rangle,\\
    C&=\|\bm q\|^2,
    &V&=\mathbb E[\|\bm{\eta}\|^2\mid\mathcal F].
  \end{aligned}
\end{equation*}
These integrability assumptions make the conditional expectations
well defined and ensure $V<\infty$ almost surely.
Then $\mathbb E[\bm{\eta}\mid\mathcal F]=\bm0$ and, exactly,
\begin{equation*}
  \widehat{\bm d}_{\alpha}-\bm d^\star_{\mathrm{ev}}
  =-\bm e+\alpha\bm q+(1+\alpha)\bm{\eta}.
\end{equation*}

\begin{proposition}[Finite-sample variance tax]
\label{prop:finite-sample-variance}
Under the preceding conditional label-affine model,
define
\begin{equation*}
  R_{\mathcal F}(\alpha)
  :=\mathbb E[\|\widehat{\bm d}_{\alpha}-\bm d^\star_{\mathrm{ev}}\|^2
  \mid\mathcal F].
\end{equation*}
Then
\begin{align}
  \mathbb E[\|\widehat{\bm d}_{\alpha}-\bm d^\star_{\mathrm{ev}}\|^2\mid\mathcal F]
  &=\|\bm e-\alpha\bm q\|^2+(1+\alpha)^2V,
  \label{eq:app-finite-sample-risk}\\
  R_{\mathcal F}(\alpha)-R_{\mathcal F}(0)
  &=-2(A-V)\alpha+(C+V)\alpha^2.
  \label{eq:app-finite-sample-difference}
\end{align}
If $C+V>0$, the risk-minimizing coefficient over $\alpha\geq0$ is
\begin{equation}
  \alpha_{\rm FS}^\star
  =\left[\frac{A-V}{C+V}\right]_+,
  \qquad [u]_+=\max\{u,0\}.
  \label{eq:app-finite-sample-optimum}
\end{equation}
There exists a strictly improving positive coefficient if and only if $A>V$;
in that case the complete improvement interval is
\begin{equation*}
  0<\alpha<\frac{2(A-V)}{C+V},
\end{equation*}
and the maximal conditional risk reduction is $(A-V)^2/(C+V)$.  If $C+V=0$,
then $\bm q=\bm0$, $\bm{\eta}=\bm0$ almost surely, and the risk is constant in $\alpha$.
\end{proposition}

\begin{proof}
Since $\boldsymbol{\mathcal G}$ is $\mathcal F$-measurable,
conditional centering gives
$\mathbb E[\boldsymbol{\mathcal G}\mathbf E\mid\mathcal F]
=\boldsymbol{\mathcal G}\mathbb E[\mathbf E\mid\mathcal F]=\bm0$.
For a random operator norm this follows first on the
$\mathcal F$-measurable events $\{\|\boldsymbol{\mathcal G}\|\leq m\}$,
then on their union; square integrability ensures the required expectations
exist. Conditionally on $\mathcal F$, $\bm e$ and $\bm q$ are fixed.
Thus the cross term between
$-\bm e+\alpha\bm q$ and
$(1+\alpha)\bm{\eta}$ has conditional expectation zero, proving
\Cref{eq:app-finite-sample-risk}.  Expanding and subtracting the value at
$\alpha=0$ gives \Cref{eq:app-finite-sample-difference}.  If $C+V>0$, completing
the square gives the unconstrained minimizer $(A-V)/(C+V)$ and hence
\Cref{eq:app-finite-sample-optimum}.  The quadratic difference has roots
$0$ and $2(A-V)/(C+V)$ when $A>V$, yielding the interval and maximum reduction.
Finally, $C+V=0$ implies both nonnegative terms vanish.
\end{proof}

Setting $V=0$ gives the same quadratic form as
\Cref{eq:app-operator-risk}, with $A$ and $C$ determined by the
finite-sample response. With finite label noise, a beneficial coefficient
exists only when residual alignment exceeds variance inflation: $A>V$.
The optimizer in \Cref{eq:app-finite-sample-optimum} is conditional on
$\mathcal F$ and may vary with the realized inputs and guide.

\Needspace{10\baselineskip}
\subsection{Scope of the analysis}
\label{app:scope-rkhs}

The population improvement argument requires square-integrable targets,
not Gaussianity. It compares fixed-noise denoising MSE at equal main-model
optimization time; guide preparation and forward passes add compute.
The finite-step result assumes fixed features, and the finite-sample
result additionally requires conditional centering and integrability.
Neural feature learning need not satisfy these conditions, and denoising
MSE alone does not guarantee improved FID or iterative sampling.

\FloatBarrier

\section{Additional Comparisons with Diffusion Guidance Methods}
\label{app:guidance-comparison}
\setcounter{equation}{0}
\renewcommand{\theequation}{B.\arabic{equation}}
\renewcommand{\theHequation}{appendixB.\arabic{equation}}

\Cref{tab:guidance-comparison} compares the six methods discussed in
\Cref{sec:swift-comparison}, including both IG formulations. We compare
their guidance references, auxiliary training components, and sampling
requirements.

\begin{table}[H]
    \centering
    \providecommand{\TableTwoCheck}{\textcolor[HTML]{2F6B4F}{\ensuremath{\checkmark}}}
    \providecommand{\TableTwoCross}{\textcolor[HTML]{B8352E}{\ensuremath{\times}}}
    \providecommand{\swift}{\textnormal{\textsc{Swift}}}
    \caption{\textbf{Training components and sampling.}
    Checks indicate the stated property.
    Unconditional predictions omit the class label; extra heads are intermediate
    outputs with their own denoising loss. One pass means one backbone evaluation
    per sampling prediction, including both IG outputs.
    SGG is its conditional training variant. Optional backbone components and
    additional guidance are excluded.}
    \label{tab:guidance-comparison}
    \small
    \setlength{\tabcolsep}{3.5pt}
    \renewcommand{\arraystretch}{1.16}
    \renewcommand{\tabularxcolumn}[1]{m{#1}}
    \begin{tabularx}{\linewidth}{@{}>{\raggedright\arraybackslash}m{0.15\linewidth} >{\raggedright\arraybackslash}X *{3}{>{\centering\arraybackslash}m{0.14\linewidth}}@{}}
        \toprule
        Method & Guidance reference
        & No unconditional prediction
        & No extra denoising head
        & One forward pass per prediction \\
        \midrule
        CFG~\citep{ho2022classifier} & unconditional prediction (same model) & \TableTwoCross & \TableTwoCheck & \TableTwoCross \\
        AG~\citep{karras2024guiding} & frozen weak model (same conditioning) & \TableTwoCheck & \TableTwoCheck & \TableTwoCross \\
        MG~\citep{tang2025mg} & conditional and unconditional EMA predictions & \TableTwoCross & \TableTwoCheck & \TableTwoCheck \\
        SGG (train)~\citep{yuan2026segmented} & unconditional or intermediate prediction, depending on noise level & \TableTwoCross & \TableTwoCross & \TableTwoCheck \\
        IG (main)~\citep{zhou2026internal} & supervised intermediate prediction & \TableTwoCheck & \TableTwoCross & \TableTwoCheck \\
        IG (training)~\citep{zhou2026internal} & final and intermediate EMA predictions & \TableTwoCheck & \TableTwoCross & \TableTwoCheck \\
        GFT~\citep{chen2025gft} & unconditional prediction (current model) & \TableTwoCross & \TableTwoCheck & \TableTwoCheck \\
        \midrule
        \textbf{\swift{}} & \textbf{pretrained, frozen weak model (same conditioning)} & \TableTwoCheck & \TableTwoCheck & \TableTwoCheck \\
        \bottomrule
    \end{tabularx}
\end{table}

\paragraph{Common notation.}
We write the diffusion objectives in one regression parameterization,
suppressing inputs, expectations, and noise-dependent loss weights.
Let $\bm q$ be the sample-wise label (clean sample, noise, or velocity),
$\bm f$ the final conditional prediction, $\bm u$ the unconditional
prediction, and $\bm h$ a supervised intermediate prediction.
Bars denote EMA predictions and $\operatorname{sg}$ stops gradients.
All outputs and labels in each expression use the same parameterization.

\paragraph{CFG and AG: extrapolation at sampling.}
CFG retains the ordinary denoising target and trains conditional and
unconditional outputs through label dropout; sampling uses
$\bm f+w(\bm f-\bm u)$~\citep[Algorithms~1--2]{ho2022classifier}.
AG also retains ordinary denoising losses, but samples with
$\bm f+\alpha(\bm f-\bm g)$, where $\bm g$ is a separately trained,
less accurate model with the same conditioning~\citep[Sec.~4]{karras2024guiding}.
AG's frozen weak model is a sampling reference, not a target in main training.

\paragraph{MG, SGG, and IG's training variant: prediction-gap targets.}
Their final-output losses have the form
\begin{equation}
 \mathcal L_{\Delta}
 =\|\bm f-[\bm q+\lambda\operatorname{sg}(\bm\Delta)]\|_2^2,
 \qquad
 \begin{array}{@{}l@{\;}l@{}}
 \bm\Delta_{\mathrm{MG}}&{}=\bar{\bm f}-\bar{\bm u},\\
 \bm\Delta_{\mathrm{SGG}}&{}=\begin{cases}
   \bm f-\bm u,&\text{high noise},\\
   \bm f-\bm h,&\text{low noise},
 \end{cases}\\
 \bm\Delta_{\mathrm{IG\text{-}train}}&{}=\bar{\bm f}-\bar{\bm h}.
 \end{array}
 \label{eq:comparison-gap-targets}
\end{equation}
MG uses conditional/unconditional EMA predictions~\citep[Eqs.~17--18]{tang2025mg}.
SGG applies its correction within a chosen time interval, switches
the weak source across noise levels, and supervises
its intermediate branch against $\bm q$~\citep[Sec.~4.3 and App.~D.1]{yuan2026segmented}.
The same paper's training-time AG variant instead uses a separate small
model, updated once per four main-model updates. Its strong-minus-weak
correction therefore evolves during training. IG's training variant
also supervises the intermediate output and uses the EMA gap in
\Cref{eq:comparison-gap-targets}~\citep[Sec.~6, Eq.~7]{zhou2026internal}.
The entire gap, including its strong prediction, is detached in these
target corrections.

\paragraph{IG's main method: auxiliary supervision and shared-pass sampling.}
The main IG objective is
$\|\bm f-\bm q\|_2^2+\gamma\|\bm h-\bm q\|_2^2$.
Both losses update the shared backbone, but the final label remains
$\bm q$. Sampling extrapolates the final and intermediate outputs from
the same backbone pass~\citep[Sec.~3.2, Eqs.~3--5]{zhou2026internal}.
Thus needing two outputs does not imply two full network passes.

\paragraph{GFT: regression through an unconditional mixture.}
GFT samples $\beta$ and conditions the shared model on it; its practical
loss mixes the sampling prediction $\bm f_\beta$ with a detached
unconditional output evaluated at $\beta=1$~\citep[Eq.~9 and Algorithm~1]{chen2025gft}:
\begin{equation}
 \|\beta\bm f_\beta+(1-\beta)\operatorname{sg}(\bm u)-\bm q\|_2^2
 =\beta^2\|\bm f_\beta-[\bm q+\alpha(\bm q-\operatorname{sg}(\bm u))]\|_2^2,
 \quad \alpha=\frac{1-\beta}{\beta},\quad 0<\beta\leq1.
 \label{eq:comparison-gft}
\end{equation}
The unconditional evaluation has no gradient, but its shared parameters
continue to change. This is a target-extrapolation identity for each
update, with a $\beta^2$ loss weight; it does not make the reference frozen.

\paragraph{Where \swift{} differs.}
\swift{} uses $\|\bm f-[\bm q+\alpha(\bm q-\bm g)]\|_2^2$ with
a separately pretrained, frozen, class-matched $\bm g$.
For fixed data and noise, the target is independent of the main
parameters. There is also an update-level overlap with prediction-gap
training: if the strong prediction inside the gap is the current
$\bm f$ and the weak $\bm g$ is fixed, then, for $0\leq w<1$,
\begin{equation}
 \nabla_\theta\|\bm f-\bm q-w\operatorname{sg}(\bm f-\bm g)\|_2^2
 =(1-w)\nabla_\theta\|\bm f-(1+\alpha)\bm q+\alpha\bm g\|_2^2,
 \qquad \alpha=\frac{w}{1-w}.
 \label{eq:comparison-fixed-guide-gradient}
\end{equation}
Together, these identities locate \swift{}'s distinction in its fixed
class-matched training reference and the finite-budget analysis in
\Cref{sec:rkhs-theory,app:rkhs-theory}. The guide is discarded at sampling;
its pretraining cost is included in \Cref{app:compute}.
\Cref{app:mog-experiment} specifies the corresponding Gaussian-mixture
adaptations rather than treating them as the original training recipes.

\Needspace{0.35\textheight}
\section{Distance to the optimal denoiser on a Gaussian mixture}
\label{app:mog-experiment}
\setcounter{equation}{0}
\renewcommand{\theequation}{C.\arabic{equation}}
\renewcommand{\theHequation}{appendixC.\arabic{equation}}

This experiment compares what several finite-budget objectives learn about
the class-A posterior mean
$\mathcal D^\star=\mathbb E_A[\bm x_0\mid\bm x_t,\sigma]$.
Every arm uses the same main backbone and class-A training data.
We score each arm's \emph{implied estimate} of $\mathcal D^\star$, since a
guided network may have a different regression target.

\paragraph{Comparison metric.}
For \swift{}, the conditional-mean training target is
$\bm g_\alpha=\mathcal D^\star+
\alpha(\mathcal D^\star-\mathcal D_{\mathrm{weak}})$
(\swiftfullref{Lemma}{lem:target-decomposition}). The raw network output
therefore targets $\bm g_\alpha$ rather than $\mathcal D^\star$.
We invert that relation to obtain
$\hat{\mathcal D}=(\mathcal D_\theta+\alpha\,\mathcal D_{\mathrm{weak}})/(1+\alpha)$,
which equals $\mathcal D^\star$ when $\mathcal D_\theta=\bm g_\alpha$.
\Cref{tab:mog-arms} gives the corresponding conversion for each arm.
Conversions use only trained predictions; the known posterior mean is
used for scoring. CFG and autoguidance change only sampling, so they are
outside this training-objective comparison.

\begin{table}[htbp]
\centering
\caption{\textbf{Pointwise training losses and evaluated estimates.}
The expectation over training examples and noise is implicit. Here
$\mathcal D_{\mathrm w}$ is the frozen class-A weak guide,
$\mathcal D_{\mathrm c}$ the frozen class-A baseline, and
$\mathcal D_{\mathrm u}$ the frozen A$+$B denoiser;
$\operatorname{sg}$ denotes stop-gradient.}
\label{tab:mog-arms}
\small
\setlength{\tabcolsep}{3.5pt}
\renewcommand{\arraystretch}{1.18}
\begin{tabular}{@{}l p{0.46\linewidth} p{0.34\linewidth}@{}}
\toprule
Arm & Pointwise loss & Estimate of $\mathcal D^\star$ \\
\midrule
baseline & $\|\mathcal D_\theta-\bm x_0\|_2^2$ & $\mathcal D_\theta$ \\
\swift{} & $\|\mathcal D_\theta-(1+\alpha)\bm x_0+\alpha\mathcal D_{\mathrm w}\|_2^2$ & $(\mathcal D_\theta+\alpha\mathcal D_{\mathrm w})/(1+\alpha)$ \\
MG \citep{tang2025mg} & $\|\mathcal D_\theta-\bm x_0-w(\mathcal D_{\mathrm c}-\mathcal D_{\mathrm u})\|_2^2$ & $\mathcal D_\theta-w(\mathcal D_{\mathrm c}-\mathcal D_{\mathrm u})$ \\
GFT \citep{chen2025gft} & $\|\beta\mathcal D_\theta+(1-\beta)\operatorname{sg}[\mathcal D_{\mathrm u}]-\bm x_0\|_2^2$ & $\beta\mathcal D_\theta+(1-\beta)\mathcal D_{\mathrm u}$ \\
IG (aux. loss) \citep{zhou2026internal} & $\|\mathcal D_{\mathrm f}-\bm x_0\|_2^2+\gamma\|\mathcal D_{\mathrm i}-\bm x_0\|_2^2$ & $\mathcal D_{\mathrm f}$ \\
\bottomrule
\end{tabular}
\end{table}

For each method, we report its best setting under the selection rule in
\Cref{tab:mog-results}. The MG and GFT arms use frozen references in place
of the original evolving predictions; GFT trains a separate model for
each setting rather than conditioning a single model on it. IG uses an
output after the first hidden layer and the auxiliary-supervision
objective, rather than IG's target-correction variant in
\Cref{app:guidance-comparison}. These are task-specific adaptations
of the cited objectives, not full reproductions.

\paragraph{Setup.}
We draw eight component means and define class A by the first four
components. Its clean distribution is
$p_A=\frac14\sum_{j=1}^4\mathcal N(\bm\mu_j,s^2\bm I_d)$;
the unconditional distribution A$+$B uses all eight with weight $1/8$.
With $\bm x_t=\bm x_0+\sigma\bm\epsilon$ and
$\bm\epsilon\sim\mathcal N(\bm0,\bm I_d)$, the \emph{class-A} posterior
mean corresponding to this setup follows from
\swiftfullref{Lemma}{lem:mog-posterior}:
\begin{equation}
  \begin{aligned}
  \mathcal D^\star_A(\bm x_t,\sigma)
  &=\sum_{j=1}^{4}r_j^A(\bm x_t)\,
     \frac{s^2\bm x_t+\sigma^2\bm\mu_j}{s^2+\sigma^2},\\
  r_j^A(\bm x_t)
  &=\frac{\mathcal N(\bm x_t;\bm\mu_j,(s^2+\sigma^2)\bm I_d)}
          {\sum_{\ell=1}^{4}
           \mathcal N(\bm x_t;\bm\mu_\ell,(s^2+\sigma^2)\bm I_d)}.
  \end{aligned}
  \label{eq:app-mog-exp-posterior}
\end{equation}
Thus $\mathcal D^\star_A=\mathcal D^\star$ in this experiment. The
eight-component distribution supplies the training data for
$\mathcal D_{\mathrm u}$, not the evaluation target.
At both $d=16$ and $d=8$, the eight means are drawn independently from
$\mathcal N(\bm0,4^2\bm I_d)$ for each seed, with $s=0.5$.
Training samples $\log\sigma\sim\mathcal U(-3,1)$. Every main arm uses a
three-hidden-layer width-64 MLP with input $[\bm x_t,\log\sigma]$, paired
initialization and class-A training stream, and Adam for 20{,}000 steps
(learning rate $2\times10^{-3}$, batch size 512). IG additionally
trains an intermediate head after the first hidden layer.

The class-A weak guide and the A$+$B unconditional reference both regress
$\bm x_0$ for 4{,}000 steps with width-16 MLPs, then remain frozen. MG
also uses the fixed width-64 class-A baseline as $\mathcal D_{\mathrm c}$.
Thus the main backbone and training budget are matched; the auxiliary
predictions used by each arm are specified in \Cref{tab:mog-arms}.

\paragraph{Evaluation.}
For each noise level $\sigma\in\{0.05,0.1,0.2,0.5,1.0\}$ we draw 20{,}000 fresh
samples of class A, noise them, and report the mean Euclidean distance
$\frac1N\sum_i\|\hat{\mathcal D}(\bm x_{t,i})-\mathcal D^\star(\bm x_{t,i})\|_2$.
Each seed redraws the mixture means and retrains every arm: 40 seeds at
$d{=}16$ and 5 at $d{=}8$. The table reports the mean across these
independent mixture-and-training replicates, not the error of an
ensemble predictor. These distances evaluate the implied estimates,
which combine the main prediction with auxiliary predictions where
specified in \Cref{tab:mog-arms}; they are not distances of the standalone
main network or measurements of sampling cost.

\begin{table}[H]
\centering
\caption{\textbf{Distance to the class-A posterior mean.}
For each method and dimension, its best setting is selected by the lowest
arithmetic mean distance across the five noise levels; the whole row uses
that one setting. Means are over independent seeds; all arms share the main backbone.
Lower is better; bold marks the best displayed entry at each noise level.}
\label{tab:mog-results}
\small
\renewcommand{\arraystretch}{1.09}
\begin{tabular*}{0.90\linewidth}{@{\extracolsep{\fill}}l rrrrr@{}}
\toprule
\multicolumn{6}{@{}l}{$d=16$ (40 seeds)} \\
\midrule
Method & $\sigma=0.05$ & $0.1$ & $0.2$ & $0.5$ & $1.0$ \\
\midrule
baseline & 0.3100 & 0.2600 & 0.2311 & 0.2125 & 0.2272 \\
\swift{} & \textbf{0.2566} & \textbf{0.2175} & \textbf{0.1952} & \textbf{0.1865} & \textbf{0.1988} \\
MG & 0.3179 & 0.2685 & 0.2388 & 0.2216 & 0.2358 \\
GFT & 0.2617 & 0.2203 & 0.1972 & 0.1904 & 0.2075 \\
IG & 0.3418 & 0.2885 & 0.2522 & 0.2258 & 0.2312 \\
\midrule
\multicolumn{6}{@{}l}{$d=8$ (5 seeds)} \\
\midrule
Method & $\sigma=0.05$ & $0.1$ & $0.2$ & $0.5$ & $1.0$ \\
\midrule
baseline & 0.1853 & 0.1574 & 0.1378 & 0.1154 & 0.1232 \\
\swift{} & \textbf{0.1450} & \textbf{0.1202} & \textbf{0.1120} & \textbf{0.1018} & \textbf{0.1189} \\
MG & 0.1913 & 0.1659 & 0.1491 & 0.1246 & 0.1356 \\
GFT & 0.1544 & 0.1264 & 0.1195 & 0.1095 & 0.1257 \\
IG & 0.1976 & 0.1657 & 0.1468 & 0.1301 & 0.1330 \\
\bottomrule
\end{tabular*}
\end{table}

\paragraph{Results.}
\Cref{tab:mog-results} reports the mean distances. At $d{=}16$,
\swift{} is lowest at all five noise levels, reducing the baseline
distance by $12.2$--$17.2\%$. GFT is the closest comparator; its distance
exceeds \swift{}'s by $0.0020$--$0.0087$. At $d{=}8$, the selected
\swift{} row is lowest among the displayed methods at all five noise
levels. MG and IG have larger errors than the baseline in both displayed comparisons.
The small GFT margins and five-seed $d{=}8$ comparison should be read
as outcomes of this matched setup, not a ranking of the original
MG, GFT, and IG implementations.

\section{Experimental Setup}
\label{app:setup}

We specify training (\Cref{app:config}), evaluation (\Cref{app:eval}),
and compute accounting (\Cref{app:compute}). In matched comparisons,
\swift{} changes only the training target, leaving the model's
architecture, optimizer, schedule, sampler, and evaluator fixed.
In the default frozen-guide recipe, the guide receives the same noisy image (or
latent), noise level, and class condition as the main denoiser. Its
prediction is evaluated without gradients and used only to construct
the target in \Cref{eq:swift-target}; the original noise distribution,
loss weighting, and optimizer update are retained. The frozen guide
is not an inference-time component.

\subsection{Training Configurations}
\label{app:config}

\paragraph{CIFAR-10.}
We use the EDM configuration of \citet{karras2022elucidating}: a
class-conditional DDPM++ model at
$32\times32$ resolution with EDM preconditioning and loss weighting,
batch size 512, fp32 training, and a $200{,}000$-kimg horizon
($\mathrm{kimg}=10^3$ training images).
The main denoiser uses channel base 128 (\texttt{cb128}).
Guide capacity is specified by channel base
(\texttt{cb16}--\texttt{cb128}), with \texttt{cb64} denoting the
default half-width guide.
NCSN++ and ADM-style guide variants use the corresponding channel
bases; their full results are reported in \Cref{app:cifar-width}.

The default guide is a separately trained, class-conditional DDPM++
\texttt{cb64} model frozen at $200{,}000$ kimg. Unless a sweep changes it,
$\alpha$ is a fixed scalar. The $\alpha=0$ arm (\texttt{D0}) is the
ordinary-training reference at the same main-model budget and evaluation
endpoint; it is not a new seed-matched baseline trained for every exploratory
setting. \Cref{app:cifar} and the accompanying \texttt{DATA/} ledger
distinguish repeated runs, guide variants, and incomplete or extended runs.

\paragraph{ImageNet-256, DiffusionBench.}
Our main DiffusionBench model~\citep{diffusionbench2026} uses a
representation\hyp{}autoencoder tokenizer~\citep{zheng2025rae} over
frozen DINOv2-B features~\citep{oquab2024dinov2} and a
615M-parameter decoupled diffusion transformer~\citep{wang2025ddt}.
The breadth study keeps the 615M-parameter training recipe while
varying the tokenizer or using a pixel-space model
(\Cref{app:imagenet-all}).
The model is trained class-conditionally with global batch size 1024 for
80 epochs, corresponding to 1251 optimizer steps per epoch, and is
evaluated every five epochs.

Unless stated otherwise, the 18 DiffusionBench configurations use the
same \swift{} recipe: $\alpha=0.4$ and a separate guide for the corresponding
model, pretrained for half the main model's standard training duration
and then frozen. Each half-width guide
has about 155M parameters and approximately one quarter of the main
denoiser's forward-pass FLOPs. The short-guide ablation freezes the
same architecture at an earlier checkpoint.
CFG is used only at sampling time and is tuned independently for the
baseline and \swift{}.

The ImageNet backends use linear flow matching,
$\bm x_t=(1-t)\bm x_0+t\bm\epsilon$ and
$\bm v=\bm\epsilon-\bm x_0$, with
$\bm\epsilon\sim\mathcal N(\bm0,\mathbf I)$, rather than the additive-noise
parameterization of \Cref{sec:swift-preliminaries}.
Most DiffusionBench configurations predict velocity.
At fixed noisy input and time, velocity is affine in
the clean prediction. The implementation constructs the main loss target in
velocity space:
\begin{equation*}
    \bm v_{\swift}
    =
    (1+\alpha)\bm v
    -
    \alpha \bm v_{\mathrm g},
\end{equation*}
This is equivalent to extrapolating the clean-sample target and then
converting it to velocity (apart from numerical clamping near a time
endpoint). For clean-sample-prediction configurations, including RAEv2,
the guide output $\widehat{\bm x}_{\mathrm g}$ is first converted to
$\bm v_{\mathrm g}=(\bm x_t-\widehat{\bm x}_{\mathrm g})/
\max(t,\varepsilon)$; the model likewise converts the main output
before computing its velocity-space loss. The guide is detached, and
RAEv2's internal-guidance auxiliary head retains its ordinary target.

The main DiffusionBench sweep is trained on $4\times$B200 GPUs.
Longer-training experiments use the same optimization configuration and
are run on $8\times$H100 GPUs where noted.

\paragraph{ImageNet-256, RAEv2.}
We use the released 875M-parameter, class-conditional DINOv3-L
configuration of
\citet{singh2026raev2} (\texttt{dinov3l-k7}) together with its native
internal-guidance implementation.
The \texttt{k7} encoder sums layers $11,13,15,17,19,21,23$. The main
model predicts clean latents and trains for up to 160 epochs (global
batch 1024, bf16, $8\times$B200 GPUs). The matched baseline keeps the
model's gMuon schedule (learning rate $2\times10^{-4}$, 25-epoch
warmup, decay to $2\times10^{-5}$ by epoch 50), latent preprocessing,
and EMA decay 0.9995.
The \swift{} guide is the model's \texttt{GUIDE-w05} weak-model
configuration, loaded from a pretrained checkpoint and frozen before
main-model training. Unless stated otherwise, we use $\alpha=0.4$.

The guide components use (width, depth, heads) $=(720,28,10)$ and
$(1024,2,8)$; \texttt{w05} denotes approximately half the model's
parameters, not half of every hidden width.
All other architectural and optimization settings follow the released
RAEv2 configuration.

\subsection{Sampling and Evaluation Settings}
\label{app:eval}

\Cref{tab:settings} summarizes our evaluation settings. Reported
metrics use 50,000 generated samples per checkpoint. Published baselines,
where used, retain the metrics and evaluation procedure reported by
their authors rather than being described as our own re-evaluations.

\begin{table}[H]
\centering
\caption{
\textbf{Sampling and evaluation settings used for our runs.}
Within a re-evaluated pair, the baseline and \swift{} use the same
sampler, generated-sample count, and metric implementation. CIFAR-10's
18 Heun steps use 35 denoiser evaluations; inference guidance may add
network passes to the other models' stated solver steps.
}
\label{tab:settings}
\footnotesize
\setlength{\tabcolsep}{4pt}
\renewcommand{\arraystretch}{1.08}
\begin{tabular}{
@{}
>{\raggedright\arraybackslash}p{2.5cm}
>{\raggedright\arraybackslash}p{3.1cm}
>{\raggedright\arraybackslash}p{3.5cm}
>{\raggedright\arraybackslash}p{3.9cm}
@{}}
\toprule
& CIFAR-10 (EDM)
& ImageNet-256 (DiffusionBench)
& ImageNet-256 (RAEv2) \\
\midrule

resolution / space
& $32\times32$, pixels
& $256\times256$, latents or pixels
& $256\times256$, DINOv3-L latents \\

sampler
& EDM deterministic (Heun)~\citep{karras2022elucidating}
& Euler~\citep{diffusionbench2026}
& Euler~\citep[App.~A, Table~11]{singh2026raev2} \\

solver steps
& 18
& 50
& 50 \\

guidance
& none
& CFG, tuned independently
& internal guidance, tuned independently \\

generated samples
& 50k
& 50k
& 50k \\

reference statistics
& CIFAR-10 training set
& ImageNet-256 training set
& ImageNet-256 training set \\

metrics
& FID
& FID, sFID, IS, precision, recall, FD$_r$, MIND
& FID, sFID, IS, precision, recall, FD$_r$, MIND \\

\bottomrule
\end{tabular}
\end{table}

\paragraph{Guidance at sampling time.}
RAEv2's native internal guidance obtains its weak prediction from the
first 8 of the backbone's 28 blocks
(\texttt{base\_model\_depth=8}), rather than from the separately
pretrained \swift{} guide.

For unguided \swift{} evaluation, the training guide is removed entirely.
The evaluation script sets \texttt{self\_guide.k=0} and clears the guide
list before checkpoint loading; the training guide is therefore absent
from both unguided and natively guided sampling.
For guided re-evaluations, CFG or internal-guidance strength is swept
separately for each model; we report the best measured FID and compute
other metrics on the samples from that setting. A published
guided baseline is identified as such, without claiming that we retuned
its scale (\Cref{fig:guidance}).

\paragraph{Checkpoint selection.}
DiffusionBench breadth results compare an 80-epoch \swift{} model with
an 80-epoch baseline or source-reported baseline; the longer DINOv2-B
pair shares an evaluation schedule. RAEv2 has a matched epoch-160
comparison, while its headline $1.15$ unguided and $0.99$ guided
\swift{} values select the epoch-140 checkpoint within the 160-epoch
budget (\Cref{app:raev2}), not a same-epoch causal comparison.

\paragraph{Metric computation.}
FID uses the training-set reference statistics of the corresponding
dataset~\citep{heusel2017gans}. ImageNet-256 sFID, Inception Score~\citep{salimans2016improved},
precision, recall, FD$_r$, and MIND use the same generated set as FID.
FD$_r$ and MIND cover Inception, ConvNeXt, DINOv2, MAE, and SigLIP;
\Cref{tab:representative} averages the five spaces. Re-evaluated pairs
share metric implementations and reference features. RAEv2's FID uses
the standard ImageNet statistics; its additional representation-space
comparisons use our ImageNet training-set feature bank.

\subsection{Compute Accounting}
\label{app:compute}

We separate main-model training from the one-time guide pretraining,
which can serve multiple \swift{} runs.

\paragraph{Wall-clock measurements.}
On $4\times$B200 GPUs, 80 DiffusionBench baseline epochs take about
44 hours versus 49 hours for the \swift{} main run; guide pretraining
takes another 12 hours once. We use FLOPs for hardware-independent
comparison.

\paragraph{Per-step FLOPs.}
Analytical forward costs are 337.5 GFLOPs per sample for the main
denoiser and 87.3 GFLOPs for its half-width guide, a ratio of 0.259.
A baseline update uses a main-model forward pass and a backward pass
estimated at twice that cost. \swift{} adds one frozen-guide forward
pass without guide backpropagation. This raises main-training FLOPs by
about $9\%$ per update. A PyTorch FLOP counter agrees with the
full-update estimate to within $0.5\%$.

\paragraph{Guide preparation.}
For the same number of samples per epoch and the same backward-to-forward
ratio, pretraining the guide for half the standard main-model duration costs about 10.4
baseline-equivalent epochs. Charged to a single 80-epoch \swift{} run,
the guide and main training together cost about 98 baseline-equivalent
epochs, or 1.22 times an 80-epoch baseline. Once the guide is available,
each additional main run incurs only the per-update increase of about
$9\%$. These model-only counts exclude the frozen representation
encoder, which is shared by both methods.

\paragraph{RAEv2 guide cost.}
The RAEv2 guide costs about $37\%$ of a main-model forward pass, giving
about $12\%$ extra FLOPs per main-training update. Its
pretraining costs about 15 baseline-equivalent epochs. Charging that
once to a 50-epoch \swift{} model gives a total of about 71
baseline-equivalent epochs, or 0.44 times the cost of a 160-epoch
baseline under the same assumptions.

\clearpage
\section{ImageNet-256: complete results}
\label{app:imagenet-full}
\label{app:breadth}

This appendix provides the complete ImageNet-256 results. We first report
the full 19-configuration comparison with and without inference-time guidance
(\Cref{app:imagenet-all}), followed by the published results used in
\Cref{fig:overview} (\Cref{app:published}), a detailed analysis of the
875M-parameter RAEv2 model (\Cref{app:raev2}), training trajectories
(\Cref{app:imagenet-traj}). Training configurations, evaluation settings,
and compute accounting are given in \Cref{app:config,app:eval,app:compute}.

\subsection{All configurations}
\label{app:imagenet-all}

The 19 configurations consist of 18 DiffusionBench configurations
(615M parameters, 80 epochs; six representation autoencoders, five VAEs,
four end-to-end-tuned autoencoders, and three pixel-space models) together
with the larger RAEv2 model. Across all configurations, \swift{} uses the
same default coefficient, $\alpha{=}0.4$, with a smaller pretrained guide
frozen before main-model training; no coefficient is tuned per
configuration. Each baseline is either a reference run trained with the same
schedule or a released checkpoint. Our DiffusionBench
runs and available matched baselines are evaluated with 50k samples
under one pipeline; published values remain as reported.
RAEv2 uses its native evaluator.

\paragraph{Without guidance.}
\swift{} lowers FID in all 18 DiffusionBench configurations
(\Cref{tab:main}), by $16.3$--$49.4\%$ (median $31.3\%$), and improves
RAEv2 from $1.56$ to $1.15$. Across all 19 configurations, 269 of the 285
reported metric comparisons favor \swift{}, 7 are ties at the displayed
precision, and 9 favor the baseline. The exceptions are limited to four
precision entries, two Inception Scores, one recall entry, and one
ConvNeXt-based FD$_r$/MIND pair.

\paragraph{With guidance.}
Guidance is tuned separately where both checkpoints are available;
source-reported guided baselines retain their reported settings
(\Cref{tab:main-cfg}). \swift{} has the lower guided FID in all 19
configurations: by $0.10$--$1.31$ ($6.7$--$33.0\%$, median $21.9\%$)
across DiffusionBench with CFG, and from $1.04$ to $0.99$ on RAEv2 with
its native internal guidance. Of the 270 DiffusionBench metric
comparisons, 252 favor \swift{}, 14 favor the baseline, and 4 are ties at
the reported precision. \Cref{fig:guidance} illustrates six representative
DiffusionBench cases and also shows why the guidance scale must be tuned
separately for the two models.

\begin{table*}[htbp]
\centering
\caption{\textbf{All 19 ImageNet-256 configurations without guidance.}
\swift{} uses the same default coefficient ($\alpha{=}0.4$) with a
smaller pretrained guide frozen before main-model training, without
per-configuration coefficient tuning. FD$_r$ and MIND are reported in five
feature spaces; the lower FID of each pair is bold.}
\label{tab:main}
\footnotesize
\setlength{\tabcolsep}{1.5pt}
\renewcommand{\arraystretch}{1.10}
\resizebox{\textwidth}{!}{%
\begin{tabular}{@{}l rrrrr rrrrr rrrrr@{}}
\toprule
& & & & & & \multicolumn{5}{c}{FD$_r$ $\downarrow$} & \multicolumn{5}{c}{MIND $\downarrow$} \\
\cmidrule(lr){7-11}\cmidrule(l){12-16}
Configuration (baseline) & FID$\downarrow$ & sFID$\downarrow$ & IS$\uparrow$ & Prec.$\uparrow$ & Rec.$\uparrow$ & \rotatebox{90}{Inception} & \rotatebox{90}{ConvNeXt} & \rotatebox{90}{DINOv2} & \rotatebox{90}{MAE} & \rotatebox{90}{SigLIP} & \rotatebox{90}{Inception} & \rotatebox{90}{ConvNeXt} & \rotatebox{90}{DINOv2} & \rotatebox{90}{MAE} & \rotatebox{90}{SigLIP} \\
\midrule
\multicolumn{16}{@{}l}{\textbf{Representation autoencoders (RAE); 615M, 80 epochs}} \\
DINOv2-B & 2.09 & 4.48 & 210.6 & 0.83 & 0.57 & 1.32 & 2.34 & 3.19 & 6.25 & 7.57 & 2.45 & 70.70 & 26.72 & 0.44 & 6.41 \\
\quad $+$ \swift{} & \textbf{1.75} & 4.26 & 202.8 & 0.81 & 0.61 & 1.09 & 2.43 & 2.87 & 5.18 & 6.89 & 1.55 & 76.44 & 23.60 & 0.35 & 5.76 \\
DINOv2-B$+$reg & 2.07 & 4.55 & 208.2 & 0.82 & 0.58 & 1.29 & 2.42 & 3.23 & 6.33 & 7.67 & 2.27 & 77.41 & 27.21 & 0.44 & 6.21 \\
\quad $+$ \swift{} & \textbf{1.71} & 4.18 & 205.2 & 0.81 & 0.61 & 1.09 & 2.40 & 2.90 & 5.29 & 6.75 & 1.56 & 74.47 & 23.36 & 0.35 & 5.79 \\
LangPE & 2.73 & 4.43 & 183.2 & 0.80 & 0.60 & 1.62 & 3.41 & 6.57 & 9.00 & 10.10 & 2.55 & 135.04 & 63.77 & 0.69 & 8.79 \\
\quad $+$ \swift{} & \textbf{2.27} & 4.01 & 184.6 & 0.79 & 0.61 & 1.36 & 3.32 & 5.71 & 7.60 & 9.29 & 1.77 & 126.68 & 55.29 & 0.56 & 8.40 \\
PE & 3.11 & 6.15 & 205.1 & 0.81 & 0.56 & 1.88 & 3.12 & 6.12 & 11.01 & 9.83 & 3.20 & 97.46 & 58.28 & 0.88 & 8.07 \\
\quad $+$ \swift{} & \textbf{2.34} & 5.02 & 209.6 & 0.81 & 0.60 & 1.43 & 2.71 & 4.76 & 9.18 & 7.74 & 2.47 & 86.73 & 44.90 & 0.70 & 6.39 \\
SigLIP2 & 3.47 & 7.02 & 180.5 & 0.79 & 0.58 & 2.09 & 4.13 & 8.61 & 11.01 & 11.66 & 2.45 & 158.79 & 86.24 & 0.90 & 10.18 \\
\quad $+$ \swift{} & \textbf{2.60} & 5.39 & 188.3 & 0.79 & 0.61 & 1.59 & 3.55 & 7.00 & 9.09 & 9.60 & 1.89 & 132.62 & 70.05 & 0.72 & 8.35 \\
SpatialPE & 3.60 & 5.20 & 160.1 & 0.78 & 0.60 & 2.20 & 4.21 & 7.19 & 6.95 & 10.73 & 2.23 & 210.99 & 68.81 & 0.54 & 9.14 \\
\quad $+$ \swift{} & \textbf{2.92} & 4.47 & 165.8 & 0.78 & 0.62 & 1.78 & 3.70 & 6.07 & 5.60 & 9.27 & 1.60 & 174.36 & 57.57 & 0.43 & 8.02 \\
\addlinespace[2pt]
\multicolumn{16}{@{}l}{\textbf{VAEs;  615M, 80 epochs}} \\
SD-VAE (EMA)$+$reg & 5.49 & 5.52 & 161.3 & 0.71 & 0.66 & 3.21 & 5.26 & 11.51 & 9.67 & 25.77 & 3.52 & 257.98 & 148.15 & 0.87 & 28.66 \\
\quad $+$ \swift{} & \textbf{3.77} & 5.06 & 179.1 & 0.73 & 0.66 & 2.19 & 3.85 & 8.58 & 7.56 & 20.01 & 2.00 & 157.58 & 107.03 & 0.68 & 21.85 \\
SD-VAE (MSE) & 10.06 & 5.79 & 110.9 & 0.67 & 0.65 & 5.92 & 10.48 & 17.97 & 12.24 & 38.65 & 9.19 & 686.07 & 244.94 & 1.09 & 44.55 \\
\quad $+$ \swift{} & \textbf{6.51} & 5.26 & 137.1 & 0.70 & 0.66 & 3.82 & 7.02 & 12.97 & 9.77 & 30.52 & 4.94 & 388.49 & 170.63 & 0.87 & 35.39 \\
SD-VAE (EMA) & 10.21 & 5.56 & 112.4 & 0.67 & 0.66 & 6.01 & 10.01 & 18.12 & 11.12 & 35.26 & 9.69 & 685.61 & 244.65 & 0.99 & 38.47 \\
\quad $+$ \swift{} & \textbf{6.73} & 5.05 & 138.9 & 0.70 & 0.67 & 3.87 & 6.65 & 13.03 & 8.91 & 27.75 & 5.21 & 388.35 & 171.18 & 0.79 & 30.49 \\
Qwen-VAE & 10.69 & 8.46 & 109.1 & 0.66 & 0.64 & 6.43 & 13.29 & 19.59 & 13.41 & 41.22 & 9.09 & 806.84 & 266.43 & 1.31 & 43.80 \\
\quad $+$ \swift{} & \textbf{7.18} & 7.68 & 134.8 & 0.69 & 0.65 & 4.13 & 8.88 & 14.24 & 10.72 & 32.64 & 4.91 & 454.12 & 186.14 & 1.05 & 34.72 \\
SDXL-VAE & 12.84 & 7.52 & 104.1 & 0.65 & 0.64 & 7.48 & 12.63 & 21.45 & 13.58 & 40.39 & 12.90 & 862.20 & 299.17 & 1.22 & 43.99 \\
\quad $+$ \swift{} & \textbf{8.38} & 6.83 & 128.7 & 0.68 & 0.65 & 4.84 & 8.39 & 15.50 & 10.91 & 31.80 & 6.98 & 486.07 & 209.20 & 0.97 & 34.91 \\
\addlinespace[2pt]
\multicolumn{16}{@{}l}{\textbf{End-to-end tuned autoencoders (E2E); 615M, 80 epochs}} \\
E2E VA-VAE & 4.14 & 4.57 & 148.9 & 0.74 & 0.65 & 2.55 & 5.80 & 8.58 & 6.34 & 16.22 & 2.13 & 270.48 & 103.26 & 0.50 & 15.36 \\
\quad $+$ \swift{} & \textbf{3.13} & 4.39 & 164.2 & 0.74 & 0.67 & 1.96 & 4.62 & 6.82 & 5.21 & 12.99 & 1.25 & 188.11 & 80.39 & 0.41 & 12.24 \\
E2E Qwen-VAE & 4.94 & 5.23 & 138.9 & 0.72 & 0.66 & 3.07 & 6.75 & 9.60 & 6.48 & 19.94 & 2.82 & 339.60 & 114.11 & 0.53 & 18.94 \\
\quad $+$ \swift{} & \textbf{3.23} & 4.89 & 164.1 & 0.73 & 0.66 & 2.08 & 4.66 & 6.41 & 4.74 & 14.48 & 1.34 & 190.38 & 72.89 & 0.37 & 13.18 \\
E2E SD3.5-VAE & 5.26 & 5.18 & 141.1 & 0.71 & 0.65 & 3.34 & 5.03 & 9.14 & 7.13 & 18.44 & 3.17 & 260.37 & 107.42 & 0.55 & 17.58 \\
\quad $+$ \swift{} & \textbf{3.61} & 4.88 & 160.9 & 0.72 & 0.67 & 2.37 & 3.66 & 6.73 & 5.50 & 13.99 & 1.71 & 160.89 & 74.95 & 0.41 & 13.00 \\
E2E FLUX-VAE & 6.20 & 5.13 & 134.9 & 0.70 & 0.66 & 3.76 & 6.84 & 10.89 & 6.63 & 21.30 & 4.07 & 372.55 & 127.66 & 0.53 & 20.24 \\
\quad $+$ \swift{} & \textbf{4.27} & 4.78 & 153.4 & 0.71 & 0.67 & 2.61 & 4.87 & 7.80 & 5.00 & 15.54 & 2.21 & 220.74 & 87.92 & 0.39 & 14.49 \\
\addlinespace[2pt]
\multicolumn{16}{@{}l}{\textbf{Pixel space; 615M, 80 epochs}} \\
PixelGen & 12.10 & 6.15 & 104.2 & 0.66 & 0.61 & 7.00 & 13.92 & 17.92 & 18.96 & 34.61 & 9.04 & 771.89 & 222.22 & 1.68 & 32.40 \\
\quad $+$ \swift{} & \textbf{8.11} & 5.46 & 126.3 & 0.68 & 0.62 & 4.72 & 9.80 & 13.72 & 15.89 & 27.86 & 5.03 & 481.69 & 166.83 & 1.39 & 26.36 \\
PixNerd & 20.37 & 6.14 & 65.9 & 0.61 & 0.63 & 12.00 & 20.71 & 24.44 & 18.69 & 46.76 & 22.30 & 1513.80 & 335.52 & 1.62 & 47.63 \\
\quad $+$ \swift{} & \textbf{10.31} & 5.38 & 93.3 & 0.70 & 0.61 & 6.14 & 12.20 & 14.33 & 13.80 & 29.92 & 8.50 & 776.86 & 172.29 & 1.23 & 28.80 \\
JiT & 22.36 & 6.98 & 63.3 & 0.58 & 0.63 & 13.20 & 23.38 & 28.66 & 22.69 & 54.69 & 25.50 & 1690.32 & 422.58 & 2.04 & 56.31 \\
\quad $+$ \swift{} & \textbf{12.31} & 6.15 & 98.4 & 0.64 & 0.63 & 7.27 & 13.52 & 18.50 & 18.20 & 38.32 & 11.10 & 805.46 & 254.28 & 1.65 & 39.33 \\
\addlinespace[2pt]
\multicolumn{16}{@{}l}{\textbf{RAEv2, 875M, up to 160 epochs}} \\
DINOv3-L & 1.56 & 4.10 & 233.3 & 0.82 & 0.65 & 1.01 & 2.24 & 2.63 & 4.84 & 6.32 & 1.44 & 70.17 & 22.46 & 0.35 & 5.61 \\
\quad $+$ \swift{} & \textbf{1.15} & 3.60 & 242.2 & 0.81 & 0.67 & 0.81 & 2.05 & 2.24 & 4.17 & 5.41 & 1.11 & 65.39 & 18.56 & 0.29 & 4.82 \\
\bottomrule
\end{tabular}%
}
\end{table*}

\begin{table*}[htbp]
\centering
\caption{\textbf{All 19 ImageNet-256 configurations with guidance.}
Inference-time guidance is tuned separately for re-evaluated pairs;
source-reported guided baselines retain their published settings.
The training recipe is otherwise the same as in
\Cref{tab:main}. FD$_r$ and MIND are reported in five feature spaces; the
lower FID of each pair is bold.}
\label{tab:main-cfg}
\footnotesize
\setlength{\tabcolsep}{1.5pt}
\renewcommand{\arraystretch}{1.10}
\resizebox{\textwidth}{!}{%
\begin{tabular}{@{}l rrrrr rrrrr rrrrr@{}}
\toprule
& & & & & & \multicolumn{5}{c}{FD$_r$ $\downarrow$} & \multicolumn{5}{c}{MIND $\downarrow$} \\
\cmidrule(lr){7-11}\cmidrule(l){12-16}
Configuration (baseline) & FID$\downarrow$ & sFID$\downarrow$ & IS$\uparrow$ & Prec.$\uparrow$ & Rec.$\uparrow$ & \rotatebox{90}{Inception} & \rotatebox{90}{ConvNeXt} & \rotatebox{90}{DINOv2} & \rotatebox{90}{MAE} & \rotatebox{90}{SigLIP} & \rotatebox{90}{Inception} & \rotatebox{90}{ConvNeXt} & \rotatebox{90}{DINOv2} & \rotatebox{90}{MAE} & \rotatebox{90}{SigLIP} \\
\midrule
\multicolumn{16}{@{}l}{\textbf{Representation autoencoders (RAE);  615M, 80 epochs}} \\
DINOv2-B & 1.98 & 4.50 & 223.1 & 0.81 & 0.59 & 1.26 & 2.21 & 3.40 & 6.12 & 7.84 & 1.91 & 66.75 & 27.95 & 0.43 & 6.52 \\
\quad $+$ \swift{} & \textbf{1.59} & 4.30 & 221.6 & 0.80 & 0.62 & 1.00 & 2.27 & 2.78 & 5.14 & 6.65 & 1.15 & 66.61 & 23.23 & 0.33 & 5.70 \\
DINOv2-B$+$reg & 1.84 & 4.55 & 236.2 & 0.82 & 0.58 & 1.16 & 2.16 & 3.24 & 6.37 & 7.96 & 1.91 & 64.08 & 27.03 & 0.43 & 6.01 \\
\quad $+$ \swift{} & \textbf{1.52} & 4.18 & 232.9 & 0.81 & 0.61 & 0.98 & 2.10 & 2.73 & 5.31 & 6.57 & 1.31 & 62.96 & 22.34 & 0.35 & 5.92 \\
LangPE & 2.39 & 4.42 & 201.6 & 0.80 & 0.61 & 1.44 & 2.95 & 6.08 & 9.00 & 9.95 & 2.16 & 104.11 & 62.58 & 0.71 & 8.85 \\
\quad $+$ \swift{} & \textbf{2.06} & 4.02 & 199.0 & 0.79 & 0.62 & 1.26 & 2.84 & 5.44 & 7.54 & 9.44 & 1.40 & 103.39 & 53.36 & 0.55 & 8.82 \\
PE & 2.84 & 5.80 & 221.5 & 0.81 & 0.57 & 1.72 & 2.78 & 5.87 & 10.68 & 9.89 & 2.82 & 89.34 & 56.48 & 0.87 & 7.88 \\
\quad $+$ \swift{} & \textbf{2.14} & 4.73 & 226.4 & 0.80 & 0.62 & 1.32 & 2.44 & 4.61 & 9.05 & 7.78 & 2.20 & 78.90 & 44.01 & 0.71 & 6.44 \\
SigLIP2 & 2.61 & 6.90 & 222.9 & 0.78 & 0.59 & 1.62 & 2.97 & 7.31 & 10.76 & 11.33 & 1.85 & 97.56 & 75.42 & 0.87 & 9.73 \\
\quad $+$ \swift{} & \textbf{2.15} & 5.32 & 219.4 & 0.79 & 0.62 & 1.32 & 2.82 & 6.29 & 8.76 & 9.68 & 1.33 & 92.86 & 66.35 & 0.69 & 8.72 \\
SpatialPE & 1.86 & 4.50 & 247.1 & 0.79 & 0.61 & 1.18 & 1.82 & 4.76 & 6.58 & 8.29 & 1.44 & 54.32 & 46.04 & 0.51 & 6.84 \\
\quad $+$ \swift{} & \textbf{1.51} & 3.87 & 255.9 & 0.79 & 0.63 & 0.97 & 1.60 & 3.87 & 5.44 & 7.71 & 1.04 & 44.44 & 36.64 & 0.40 & 6.52 \\
\addlinespace[2pt]
\multicolumn{16}{@{}l}{\textbf{VAEs;  615M, 80 epochs}} \\
SD-VAE (EMA)$+$reg & 2.32 & 4.70 & 271.5 & 0.80 & 0.60 & 1.32 & 1.73 & 7.42 & 7.89 & 18.58 & 1.90 & 53.68 & 85.97 & 0.71 & 20.32 \\
\quad $+$ \swift{} & \textbf{1.81} & 4.42 & 272.9 & 0.80 & 0.61 & 0.99 & 1.46 & 5.82 & 6.19 & 15.20 & 1.64 & 46.10 & 65.92 & 0.55 & 16.52 \\
SD-VAE (MSE) & 2.56 & 5.50 & 259.7 & 0.70 & 0.65 & 1.43 & 2.14 & 7.60 & 7.93 & 21.85 & 3.09 & 55.50 & 87.48 & 0.69 & 24.63 \\
\quad $+$ \swift{} & \textbf{1.86} & 4.99 & 268.4 & 0.73 & 0.66 & 0.91 & 1.42 & 5.33 & 6.23 & 17.13 & 1.62 & 30.65 & 59.76 & 0.55 & 19.54 \\
SD-VAE (EMA) & 2.43 & 5.40 & 259.6 & 0.69 & 0.65 & 1.39 & 1.94 & 7.70 & 6.97 & 19.19 & 2.87 & 49.03 & 88.22 & 0.62 & 20.47 \\
\quad $+$ \swift{} & \textbf{1.90} & 5.08 & 260.9 & 0.69 & 0.66 & 1.07 & 1.60 & 5.99 & 5.40 & 15.37 & 2.52 & 42.61 & 67.34 & 0.49 & 16.57 \\
Qwen-VAE & 3.01 & 7.70 & 238.9 & 0.68 & 0.65 & 1.83 & 4.59 & 9.84 & 9.27 & 25.20 & 2.75 & 160.49 & 118.34 & 0.91 & 25.93 \\
\quad $+$ \swift{} & \textbf{2.32} & 6.99 & 250.8 & 0.71 & 0.66 & 1.19 & 3.07 & 7.06 & 7.41 & 20.16 & 1.48 & 90.34 & 82.45 & 0.70 & 20.68 \\
SDXL-VAE & 3.07 & 6.90 & 256.0 & 0.67 & 0.65 & 1.71 & 2.76 & 9.28 & 8.41 & 21.84 & 3.32 & 70.24 & 108.68 & 0.74 & 22.53 \\
\quad $+$ \swift{} & \textbf{2.21} & 6.27 & 268.8 & 0.70 & 0.66 & 1.09 & 1.84 & 6.78 & 6.61 & 17.01 & 1.79 & 39.10 & 76.95 & 0.61 & 17.94 \\
\addlinespace[2pt]
\multicolumn{16}{@{}l}{\textbf{End-to-end tuned autoencoders (E2E); 615M, 80 epochs}} \\
E2E VA-VAE & 1.65 & 4.50 & 275.4 & 0.74 & 0.65 & 1.09 & 1.96 & 4.61 & 4.82 & 9.82 & 2.16 & 62.56 & 51.43 & 0.36 & 9.01 \\
\quad $+$ \swift{} & \textbf{1.25} & 4.32 & 289.2 & 0.73 & 0.67 & 0.81 & 1.61 & 3.53 & 3.80 & 7.91 & 1.25 & 43.97 & 40.60 & 0.31 & 7.24 \\
E2E Qwen-VAE & 1.50 & 4.83 & 263.3 & 0.80 & 0.62 & 1.05 & 2.29 & 4.58 & 4.61 & 11.10 & 1.45 & 61.43 & 48.55 & 0.37 & 10.20 \\
\quad $+$ \swift{} & \textbf{1.40} & 4.59 & 296.5 & 0.80 & 0.63 & 0.98 & 1.75 & 3.14 & 3.31 & 7.49 & 1.81 & 78.14 & 30.31 & 0.26 & 6.50 \\
E2E SD3.5-VAE & 1.62 & 5.00 & 265.4 & 0.72 & 0.66 & 1.14 & 1.27 & 4.48 & 5.44 & 11.20 & 1.09 & 42.67 & 47.19 & 0.42 & 10.30 \\
\quad $+$ \swift{} & \textbf{1.31} & 4.81 & 278.7 & 0.71 & 0.68 & 0.88 & 1.03 & 3.47 & 4.47 & 8.92 & 0.65 & 31.27 & 36.05 & 0.33 & 8.13 \\
E2E FLUX-VAE & 1.67 & 5.00 & 266.3 & 0.70 & 0.66 & 1.04 & 1.81 & 5.17 & 4.57 & 11.60 & 1.16 & 51.16 & 54.27 & 0.36 & 10.39 \\
\quad $+$ \swift{} & \textbf{1.25} & 4.66 & 279.6 & 0.71 & 0.67 & 0.74 & 1.33 & 3.56 & 3.48 & 8.84 & 0.63 & 32.49 & 36.09 & 0.27 & 7.90 \\
\addlinespace[2pt]
\multicolumn{16}{@{}l}{\textbf{Pixel space;  615M, 80 epochs}} \\
PixelGen & 3.97 & 5.50 & 247.4 & 0.67 & 0.61 & 2.27 & 4.29 & 7.85 & 13.09 & 17.80 & 3.99 & 137.50 & 86.22 & 1.12 & 15.81 \\
\quad $+$ \swift{} & \textbf{2.66} & 4.88 & 259.8 & 0.69 & 0.63 & 1.53 & 3.07 & 6.07 & 11.06 & 14.28 & 2.20 & 86.24 & 65.50 & 0.93 & 12.93 \\
PixNerd & 4.17 & 5.50 & 213.8 & 0.65 & 0.63 & 2.45 & 4.00 & 8.46 & 11.70 & 20.87 & 4.21 & 103.10 & 86.18 & 0.98 & 18.85 \\
\quad $+$ \swift{} & \textbf{3.38} & 4.95 & 235.2 & 0.68 & 0.62 & 1.96 & 3.28 & 6.82 & 9.54 & 17.16 & 3.38 & 85.25 & 70.80 & 0.79 & 15.36 \\
JiT & 4.08 & 6.00 & 231.2 & 0.60 & 0.63 & 2.38 & 4.60 & 9.65 & 13.59 & 22.44 & 3.93 & 147.18 & 114.34 & 1.23 & 21.53 \\
\quad $+$ \swift{} & \textbf{2.86} & 5.29 & 277.4 & 0.66 & 0.64 & 1.31 & 2.63 & 6.21 & 10.91 & 15.81 & 1.73 & 71.94 & 68.08 & 0.99 & 14.87 \\
\addlinespace[2pt]
\multicolumn{16}{@{}l}{\textbf{RAEv2, 875M, up to 160 epochs}} \\
DINOv3-L & 1.04 & 4.00 & 266.0 & 0.82 & 0.61 & 0.87 & 1.66 & 2.67 & 5.21 & 6.16 & 1.34 & 49.36 & 22.47 & 0.34 & 4.84 \\
\quad $+$ \swift{} & \textbf{0.99} & 3.80 & 264.9 & 0.81 & 0.63 & 0.82 & 1.59 & 2.41 & 4.83 & 5.69 & 1.25 & 46.11 & 20.14 & 0.30 & 4.46 \\
\bottomrule
\end{tabular}%
}
\end{table*}

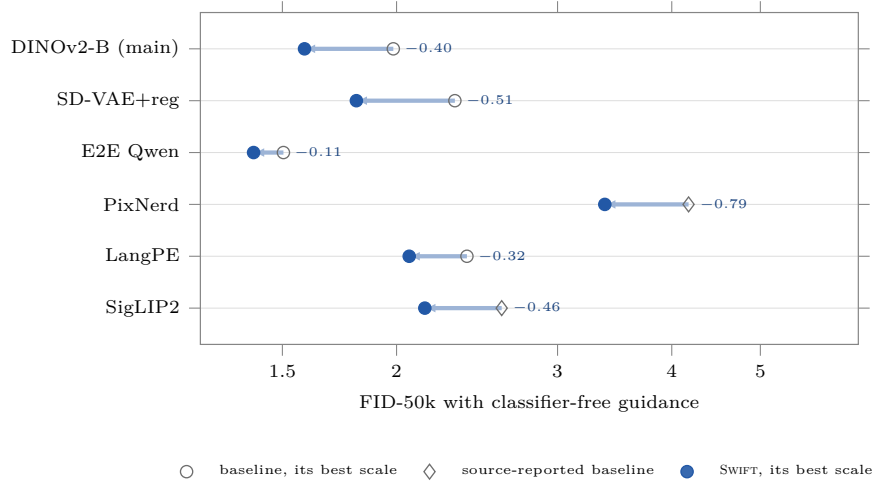
\begin{figure}[htbp]
\centering
\begin{tikzpicture}
\begin{axis}[width=0.62\linewidth, height=0.36\linewidth,
  tick label style={font=\scriptsize, /pgf/number format/assume math mode=true},
  label style={font=\scriptsize},
  ymajorgrids, grid style={cGrid, line width=0.3pt},
  axis line style={cAxis!65, line width=0.45pt},
  tick style={cAxis!65, line width=0.45pt}, tick align=outside,
  xtick style={line width=0.4pt}, ytick style={line width=0.4pt},
  every axis plot/.append style={line width=0.9pt, mark size=1.25pt,
      mark options={line width=0.5pt}},
  legend style={font=\tiny, draw=none, fill=white, fill opacity=0.9, text opacity=1,
      row sep=-2.5pt, inner xsep=3pt, inner ysep=2pt},
  legend cell align=left, legend image post style={line width=0.9pt},
  xlabel={FID-50k with classifier-free guidance}, ymin=0.3, ymax=6.7, xmode=log, xmin=1.22, xmax=6.4, xtick={1.5,2,3,4,5}, xticklabels={1.5,2,3,4,5}, ytick={1,...,6}, yticklabels={{SigLIP2},{LangPE},{PixNerd},{E2E Qwen},{SD-VAE$+$reg},{DINOv2-B (main)}}, yticklabel style={font=\scriptsize}, legend style={at={(0.5,-0.33)}, anchor=north, legend columns=3, column sep=8pt}]
\addplot[cSeg, line width=1.6pt, -{Latex[length=4pt,width=3.4pt]}, forget plot] coordinates {(1.9837,6) (1.5861,6)};
\node[font=\tiny, cAGT!75!black, anchor=west, inner sep=4pt] at (axis cs:1.9837,6) {$-0.40$};
\addplot[cSeg, line width=1.6pt, -{Latex[length=4pt,width=3.4pt]}, forget plot] coordinates {(2.3163,5) (1.8075,5)};
\node[font=\tiny, cAGT!75!black, anchor=west, inner sep=4pt] at (axis cs:2.3163,5) {$-0.51$};
\addplot[cSeg, line width=1.6pt, -{Latex[length=4pt,width=3.4pt]}, forget plot] coordinates {(1.5040,4) (1.3950,4)};
\node[font=\tiny, cAGT!75!black, anchor=west, inner sep=4pt] at (axis cs:1.5040,4) {$-0.11$};
\addplot[cSeg, line width=1.6pt, -{Latex[length=4pt,width=3.4pt]}, forget plot] coordinates {(4.1732,3) (3.3793,3)};
\node[font=\tiny, cAGT!75!black, anchor=west, inner sep=4pt] at (axis cs:4.1732,3) {$-0.79$};
\addplot[cSeg, line width=1.6pt, -{Latex[length=4pt,width=3.4pt]}, forget plot] coordinates {(2.3876,2) (2.0646,2)};
\node[font=\tiny, cAGT!75!black, anchor=west, inner sep=4pt] at (axis cs:2.3876,2) {$-0.32$};
\addplot[cSeg, line width=1.6pt, -{Latex[length=4pt,width=3.4pt]}, forget plot] coordinates {(2.6062,1) (2.1475,1)};
\node[font=\tiny, cAGT!75!black, anchor=west, inner sep=4pt] at (axis cs:2.6062,1) {$-0.46$};
\addplot[only marks, mark=o, mark size=2.3pt, cBase, line width=0.9pt] coordinates {(1.9837,6) (2.3163,5) (1.5040,4) (2.3876,2)};
\addlegendentry{baseline, its best scale}
\addplot[only marks, mark=diamond, mark size=2.8pt, cBase, line width=0.9pt] coordinates {(4.1732,3) (2.6062,1)};
\addlegendentry{source-reported baseline}
\addplot[only marks, mark=*, mark size=2.3pt, cAGT] coordinates {(1.5861,6) (1.8075,5) (1.3950,4) (3.3793,3) (2.0646,2) (2.1475,1)};
\addlegendentry{\swift{}, its best scale}
\end{axis}
\end{tikzpicture}
\caption{\textbf{Guided FID at each model's best available CFG scale}
for six DiffusionBench configurations from \Cref{tab:main-cfg}. Open
circles denote matched baselines whose guidance scale is swept
independently; open diamonds denote source-reported guided baselines.
Filled dots show \swift{} at its own best scale, and annotations give the
FID change. The preferred scale can shift substantially (e.g.,
$2.2\!\to\!1.8$ for SD-VAE$+$reg and $2.4\!\to\!1.8$ for PixNerd),
motivating separate tuning on the two sides.}
\label{fig:guidance}
\end{figure}

\subsection{Published results in \texorpdfstring{\Cref{fig:overview}}{Figure 1}}
\label{app:published}

\Cref{tab:published} lists the published results plotted in the left panel
of \Cref{fig:overview}. Every number is the one reported by the authors for
their best configuration at $256\times256$; we did not re-evaluate these
models.

\begin{table}[htbp]
\centering
\caption{\textbf{Published ImageNet-256 results shown in \Cref{fig:overview} (left).}
FID-50K values are taken directly from the corresponding papers.
Evaluation procedures differ between papers: some sample 50 images per class,
others sample class labels at random, which can change FID by about $0.1$
\citep[Table~14]{zheng2025rae}. Guidance: CFG, classifier-free guidance;
int., with a guidance interval; AG, autoguidance; IG, internal guidance
\citep{zhou2026internal}; REPA guid., the REPA guidance of
\citet{singh2026raev2}. $^{*}$Without CFG, but still with internal guidance.
Our two RAEv2 rows use the settings of \Cref{app:raev2}.}
\label{tab:published}
\small
\begin{tabular}{@{}llrrl@{}}
\toprule
Model & Type & Base FID & Guided FID & Guidance \\
\midrule
DiT-XL/2 \citep{peebles2023scalable} & latent diffusion & 9.62 & 2.27 & CFG \\
SiT-XL/2 \citep{ma2024sit} & latent diffusion & 8.30 & 2.06 & CFG \\
MDT-XL/2 \citep{gao2023mdt} & latent diffusion & 6.23 & 1.79 & CFG \\
DDT-XL/2 \citep{wang2025ddt} & latent diffusion & 6.27 & 1.26 & CFG, int. \\
REPA \citep{yu2025repa} & latent diffusion & 5.90 & 1.42 & CFG, int. \\
MAR-H \citep{li2024mar} & autoregressive & 2.35 & 1.55 & CFG \\
SFD-XXL \citep{pan2026sfd} & latent diffusion & 2.38 & 1.04 & AG \\
LightningDiT \citep{yao2025lightningdit} & latent diffusion & 2.17 & 1.35 & CFG, int. \\
REG \citep{wu2025reg} & latent diffusion & 1.80 & 1.36 & CFG, int. \\
REPA-E \citep{leng2025repae} & latent diffusion & 1.69 & 1.12 & CFG, int. \\
RAEv2 \citep{singh2026raev2} & latent diffusion & 1.65 & 1.06 & REPA guid. \\
RAE \citep{zheng2025rae} & latent diffusion & 1.51 & 1.13 & AG \\
PAE \citep{yue2026pae} & latent diffusion & 1.43 & 1.03 & CFG \\
BAR-L \citep{yu2026bar} & autoregressive & 1.42 & 0.99 & CFG \\
LightningDiT+IG \citep{zhou2026internal} & latent diffusion & 1.34$^{*}$ & 1.19 & IG, CFG, int. \\
GAE \citep{liu2026gae} & latent diffusion & 1.31 & 1.13 & CFG, int. \\
\midrule
RAEv2, our baseline & latent diffusion & 1.56 & 1.04 & \Cref{app:raev2} \\
\swift{} on RAEv2 (ours) & latent diffusion & \textbf{1.15} & \textbf{0.99} & \Cref{app:raev2} \\
\bottomrule
\end{tabular}
\end{table}

\FloatBarrier

\subsection{RAEv2}
\label{app:raev2}

RAEv2 provides a second, stronger backbone: an 875M-parameter model on
DINOv3-L latents with native internal guidance. \swift{} uses the default
$\alpha{=}0.4$ recipe and a pretrained \texttt{GUIDE-w05} guide frozen before main training
(\Cref{app:config}). The canonical RAEv2 results are those reported in
\Cref{tab:main,tab:main-cfg}: $1.56\!\to\!1.15$ without guidance and
$1.04\!\to\!0.99$ with independently tuned internal guidance.

\paragraph{Training trajectory.}
To examine when the gain appears, we additionally track \swift{} at 17
checkpoints and the $\alpha{=}0$ baseline at epochs 20, 40, 80, 120, and
160 (\Cref{fig:raev2}). The $\alpha{=}0.4$ run is already below the
released epoch-80 checkpoint by epoch 40 and remains on a substantially
lower trajectory thereafter. Its best observed unguided checkpoint occurs
at epoch 140 and yields the $1.15$ value reported in \Cref{tab:main}.
The denser checkpoint trace is used only to study training dynamics; the
headline values remain those in the two complete-result tables.

\paragraph{Interaction with internal guidance.}
RAEv2 uses
\begin{equation*}
    D_s = D + (s-1)(D-D_{\mathrm{REPA}}),
\end{equation*}
so $s=1$ disables guidance. The preferred guidance strength shifts
downward after \swift{} training, while the best guided result improves
from $1.04$ to $0.99$ (\Cref{tab:main-cfg}). This is consistent with the
main denoiser absorbing part of the weak-to-strong correction during
training, while inference-time guidance remains complementary.

\begin{figure}[htbp]
\centering
\begin{tikzpicture}
\begin{axis}[width=0.62\linewidth, height=0.34\linewidth, tick label style={font=\scriptsize, /pgf/number format/assume math mode=true},
  label style={font=\scriptsize}, title style={font=\scriptsize\bfseries, yshift=-3pt},
  ymajorgrids, grid style={cGrid, line width=0.3pt},
  axis line style={cAxis!65, line width=0.45pt},
  tick style={cAxis!65, line width=0.45pt}, tick align=outside,
  xtick style={line width=0.4pt}, ytick style={line width=0.4pt},
  every axis plot/.append style={line width=0.9pt, mark size=1.25pt,
      mark options={line width=0.5pt}},
  legend style={font=\tiny, draw=none, fill=white, fill opacity=0.9, text opacity=1,
      row sep=-2.5pt, inner xsep=3pt, inner ysep=2pt},
  legend cell align=left, legend image post style={line width=0.9pt}, xlabel={training epoch}, ylabel={FID-50k (no guidance)}, ymode=log, ymin=1.1, ymax=4.2, ytick={1.15,1.3,1.5,2,3,4}, yticklabels={1.15,1.3,1.5,2,3,4}, xmin=5, xmax=165, xtick={20,40,...,160}, legend pos=north east]
\addplot[cBase, mark=*] coordinates {(20,2.6706) (40,1.8726) (80,1.6459) (120,1.6201) (160,1.5628)};
\addlegendentry{matched baseline ($\alpha{=}0$)}
\addplot[cAGT, mark=*] coordinates {(10,3.6232) (20,2.2545) (25,2.0771) (30,1.8295) (40,1.5483) (50,1.3408) (60,1.2824) (70,1.2482) (80,1.2760) (90,1.1986) (100,1.2042) (110,1.2052) (120,1.2364) (130,1.2247) (140,1.1479) (150,1.1707) (160,1.2142)};
\addlegendentry{\swift{} $\alpha{=}0.4$}
\addplot[cTwo, mark=*] coordinates {(10,3.6358) (20,2.1088) (30,1.6342) (40,1.4162) (50,1.3723) (60,1.4049) (70,1.3229) (80,1.3506) (90,1.3362) (100,1.3119) (110,1.2541) (120,1.2832) (130,1.2723) (140,1.3016) (150,1.2325) (160,1.3237)};
\addlegendentry{\swift{} $\alpha{=}0.8$}
\addplot[only marks, mark=diamond*, mark size=2.6pt, cAxis] coordinates {(80,1.633)};
\addlegendentry{released checkpoint (ep.\ 80)}
\end{axis}
\end{tikzpicture}
\caption{\textbf{RAEv2 unguided training trajectories.}
The dense \swift{} traces are used to study training dynamics; the
canonical RAEv2 results are reported in \Cref{tab:main,tab:main-cfg}.
The diamond marks the released epoch-80 checkpoint.}
\label{fig:raev2}
\end{figure}
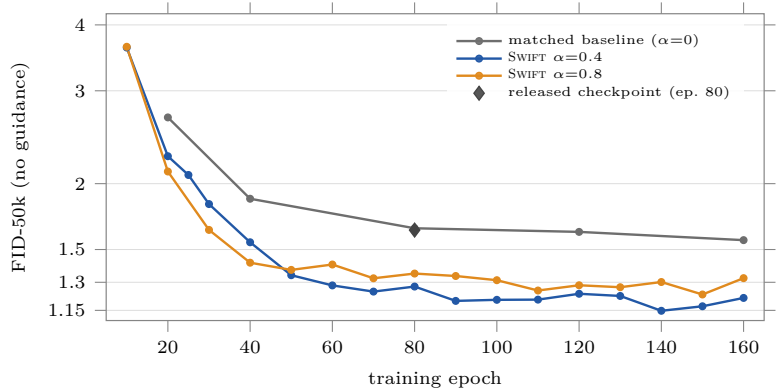

\FloatBarrier

\subsection{Training trajectories and the long schedule}
\label{app:imagenet-traj}

\paragraph{Across configurations.}
For the 12 DiffusionBench configurations whose baseline we also evaluated
during training, \swift{} is ahead at every evaluation from epoch 40 on.
\Cref{fig:four} shows four of them, with baselines spanning roughly an
order of magnitude in FID.

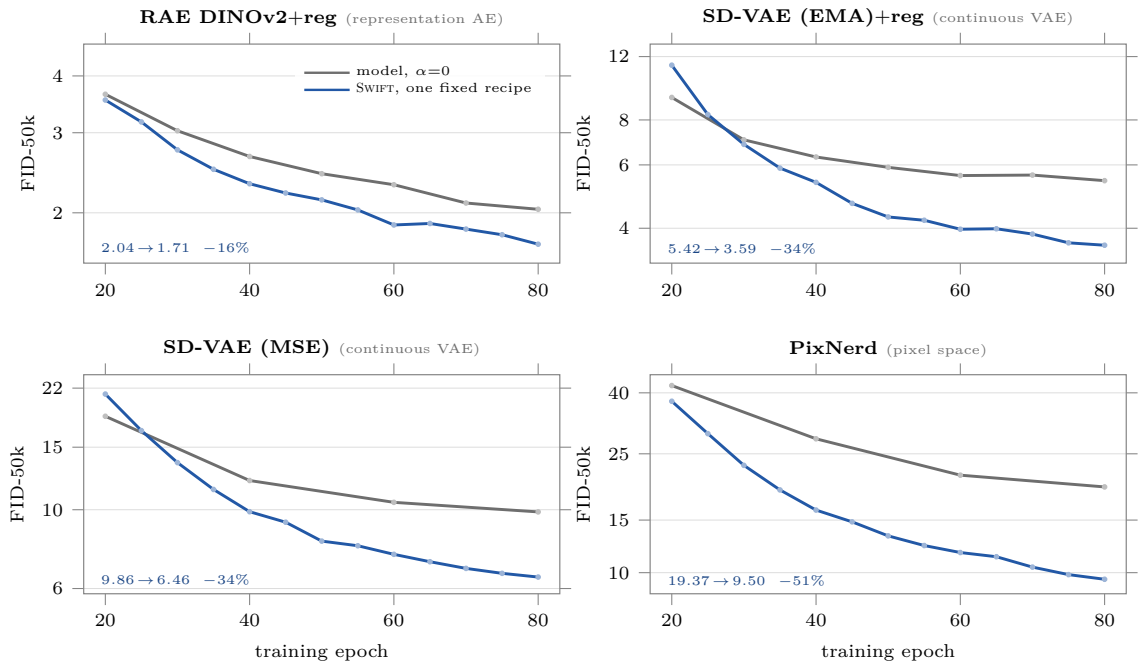
\begin{figure}[H]
\centering
\begin{tikzpicture}
\begin{groupplot}[group style={group size=2 by 2, horizontal sep=34pt, vertical sep=42pt}, width=0.475\linewidth, height=0.27\linewidth,
tick label style={font=\scriptsize, /pgf/number format/assume math mode=true},
  label style={font=\scriptsize}, title style={font=\scriptsize\bfseries, yshift=-3pt},
  ymajorgrids, grid style={cGrid, line width=0.3pt},
  axis line style={cAxis!65, line width=0.45pt},
  tick style={cAxis!65, line width=0.45pt}, tick align=outside,
  xtick style={line width=0.4pt}, ytick style={line width=0.4pt},
  every axis plot/.append style={line width=0.9pt, mark size=1.25pt,
      mark options={line width=0.5pt}},
  legend style={font=\tiny, draw=none, fill=white, fill opacity=0.9, text opacity=1,
      row sep=-2.5pt, inner xsep=3pt, inner ysep=2pt},
  legend cell align=left, legend image post style={line width=0.9pt}]
\nextgroupplot[title={RAE DINOv2$+$reg~\textnormal{\tiny\color{cBase}(representation AE)}}, ymode=log, ylabel={FID-50k}, ymin=1.55, ymax=4.7, ytick={2,3,4}, yticklabels={2,3,4}, xmin=17, xmax=83, xtick={20,40,60,80}, legend style={at={(0.97,0.95)}, anchor=north east}]
\addplot[only marks, mark=*, mark size=0.8pt, cBase!45, forget plot] coordinates {(20,3.6454) (30,3.0304) (40,2.6600) (50,2.4378) (60,2.3053) (70,2.1016) (80,2.0358)};
\addplot[only marks, mark=*, mark size=0.8pt, cAGT!45, forget plot] coordinates {(20,3.5411) (25,3.1687) (30,2.7508) (35,2.4926) (40,2.3169) (45,2.2109) (50,2.1351) (55,2.0292) (60,1.8806) (65,1.8947) (70,1.8417) (75,1.7897) (80,1.7060)};
\addplot[cBase, line width=1.1pt] coordinates {(20,3.6454) (30,3.0304) (40,2.6600) (50,2.4378) (60,2.3053) (70,2.1016) (80,2.0358)};
\addlegendentry{model, $\alpha{=}0$}
\addplot[cAGT, line width=1.1pt] coordinates {(20,3.5411) (25,3.1687) (30,2.7508) (35,2.4926) (40,2.3169) (45,2.2109) (50,2.1351) (55,2.0292) (60,1.8806) (65,1.8947) (70,1.8417) (75,1.7897) (80,1.7060)};
\addlegendentry{\swift{}, one fixed recipe}
\node[font=\tiny, cAGT!80!black, anchor=south west, inner sep=1pt] at (axis cs:19,1.5965) {$2.04\!\to\!1.71$\ \ \textbf{$-16\%$}};
\nextgroupplot[title={SD-VAE (EMA)$+$reg~\textnormal{\tiny\color{cBase}(continuous VAE)}}, ymode=log, ylabel={FID-50k}, ymin=3.20, ymax=13.0, ytick={4,6,8,12}, yticklabels={4,6,8,12}, xmin=17, xmax=83, xtick={20,40,60,80},]
\addplot[only marks, mark=*, mark size=0.8pt, cBase!45, forget plot] coordinates {(20,9.2390) (30,7.0438) (40,6.3118) (50,5.9075) (60,5.6040) (70,5.6218) (80,5.4245)};
\addplot[only marks, mark=*, mark size=0.8pt, cAGT!45, forget plot] coordinates {(20,11.3580) (25,8.2808) (30,6.8428) (35,5.8782) (40,5.3656) (45,4.6906) (50,4.3000) (55,4.2097) (60,3.9757) (65,3.9879) (70,3.8540) (75,3.6444) (80,3.5884)};
\addplot[cBase, line width=1.1pt] coordinates {(20,9.2390) (30,7.0438) (40,6.3118) (50,5.9075) (60,5.6040) (70,5.6218) (80,5.4245)};
\addplot[cAGT, line width=1.1pt] coordinates {(20,11.3580) (25,8.2808) (30,6.8428) (35,5.8782) (40,5.3656) (45,4.6906) (50,4.3000) (55,4.2097) (60,3.9757) (65,3.9879) (70,3.8540) (75,3.6444) (80,3.5884)};
\node[font=\tiny, cAGT!80!black, anchor=south west, inner sep=1pt] at (axis cs:19,3.2960) {$5.42\!\to\!3.59$\ \ \textbf{$-34\%$}};
\nextgroupplot[title={SD-VAE (MSE)~\textnormal{\tiny\color{cBase}(continuous VAE)}}, ymode=log, xlabel={training epoch}, ylabel={FID-50k}, ymin=5.80, ymax=24.0, ytick={6,10,15,22}, yticklabels={6,10,15,22}, xmin=17, xmax=83, xtick={20,40,60,80},]
\addplot[only marks, mark=*, mark size=0.8pt, cBase!45, forget plot] coordinates {(20,18.3302) (40,12.0844) (60,10.4915) (80,9.8606)};
\addplot[only marks, mark=*, mark size=0.8pt, cAGT!45, forget plot] coordinates {(20,21.1688) (25,16.7019) (30,13.5733) (35,11.4043) (40,9.8731) (45,9.2228) (50,8.1594) (55,7.9190) (60,7.4891) (65,7.1391) (70,6.8368) (75,6.6215) (80,6.4639)};
\addplot[cBase, line width=1.1pt] coordinates {(20,18.3302) (40,12.0844) (60,10.4915) (80,9.8606)};
\addplot[cAGT, line width=1.1pt] coordinates {(20,21.1688) (25,16.7019) (30,13.5733) (35,11.4043) (40,9.8731) (45,9.2228) (50,8.1594) (55,7.9190) (60,7.4891) (65,7.1391) (70,6.8368) (75,6.6215) (80,6.4639)};
\node[font=\tiny, cAGT!80!black, anchor=south west, inner sep=1pt] at (axis cs:19,5.9740) {$9.86\!\to\!6.46$\ \ \textbf{$-34\%$}};
\nextgroupplot[title={PixNerd~\textnormal{\tiny\color{cBase}(pixel space)}}, ymode=log, xlabel={training epoch}, ylabel={FID-50k}, ymin=8.50, ymax=46.0, ytick={10,15,25,40}, yticklabels={10,15,25,40}, xmin=17, xmax=83, xtick={20,40,60,80},]
\addplot[only marks, mark=*, mark size=0.8pt, cBase!45, forget plot] coordinates {(20,42.3084) (40,28.0651) (60,21.2054) (80,19.3687)};
\addplot[only marks, mark=*, mark size=0.8pt, cAGT!45, forget plot] coordinates {(20,37.4927) (25,29.2127) (30,22.8707) (35,18.9277) (40,16.2039) (45,14.7951) (50,13.2787) (55,12.3263) (60,11.6798) (65,11.3080) (70,10.4389) (75,9.8514) (80,9.5039)};
\addplot[cBase, line width=1.1pt] coordinates {(20,42.3084) (40,28.0651) (60,21.2054) (80,19.3687)};
\addplot[cAGT, line width=1.1pt] coordinates {(20,37.4927) (25,29.2127) (30,22.8707) (35,18.9277) (40,16.2039) (45,14.7951) (50,13.2787) (55,12.3263) (60,11.6798) (65,11.3080) (70,10.4389) (75,9.8514) (80,9.5039)};
\node[font=\tiny, cAGT!80!black, anchor=south west, inner sep=1pt] at (axis cs:19,8.7550) {$19.37\!\to\!9.50$\ \ \textbf{$-51\%$}};
\end{groupplot}
\end{tikzpicture}
\caption{
\textbf{The same fixed \swift{} recipe improves four different ImageNet-256
training trajectories.}
Each panel shows a different DiffusionBench stage-2 model and representation.
Grey is the matched $\alpha{=}0$ baseline and blue is \swift{} with the same fixed
recipe ($\alpha{=}0.4$, half-width guide pretrained for half the standard
main-model duration and then frozen), with no
per-model tuning.
\swift{} is ahead by epoch 40 in all four cases and remains ahead through the end
of training, across baselines spanning roughly an order of magnitude in FID.
The $y$-axis is logarithmic; faint markers denote individual FID evaluations.
The trajectory baselines reproduce the corresponding DiffusionBench endpoints
within $0.6$--$6\%$; the comparisons in \Cref{tab:main} instead re-evaluate
both sides with 50k samples.
}
\label{fig:four}
\end{figure}

\paragraph{Guide maturity on DINOv2-B.}
In the matched 80-epoch ablation, the default guide is pretrained for half
the standard main-model duration. Its $\alpha{=}0.4$ and $0.8$ arms lag
the baseline at epoch 20 (by $0.11$ and $0.05$ FID), but lead from
epoch 40 onward (by $0.28$--$0.48$). An earlier guide with
$\alpha{=}0.4$ leads throughout (by $0.47$--$0.59$ FID).

\paragraph{Persistence through 240 epochs.}
In a separate, hardware-matched DINOv2-B pair using the fixed
$\alpha{=}0.4$, default frozen-guide recipe, we train both arms to 240 epochs.
Across 36 matched evaluations from epoch 20 onward, \swift{} has
lower FID at every checkpoint. The gap remains $0.26$ FID at epoch 240,
so the gain does not vanish on this longer schedule. \Cref{tab:long}
reports FID values every 20 epochs, and \Cref{fig:long-run}
shows the corresponding training trajectories. A separate $\alpha{=}0$ hardware replicate
is an additional reference only and is excluded from the differences.

\begin{table}[H]
\centering
\caption{\textbf{DINOv2-B over 240 matched training epochs.}
Unguided FID from the training-time evaluator, sampled every 20 epochs.
The two blocks list epochs in order; $\Delta$ is \swift{} minus baseline,
computed before rounding the FIDs displayed here.
\Cref{fig:long-run} plots the same 20-epoch checkpoints.}
\label{tab:long}
\small
\setlength{\tabcolsep}{5pt}
\renewcommand{\arraystretch}{1.08}
\begin{tabular}{@{}rrrr@{\qquad}rrrr@{}}
\toprule
Epoch & Baseline & \swift{} & $\Delta$
& Epoch & Baseline & \swift{} & $\Delta$ \\
\midrule
20 & 3.63 & \textbf{3.52} & $-0.12$ & 140 & 1.81 & \textbf{1.51} & $-0.30$ \\
40 & 2.67 & \textbf{2.32} & $-0.36$ & 160 & 1.80 & \textbf{1.46} & $-0.34$ \\
60 & 2.36 & \textbf{2.00} & $-0.36$ & 180 & 1.75 & \textbf{1.49} & $-0.26$ \\
80 & 2.13 & \textbf{1.77} & $-0.37$ & 200 & 1.72 & \textbf{1.44} & $-0.28$ \\
100 & 1.99 & \textbf{1.60} & $-0.39$ & 220 & 1.67 & \textbf{1.41} & $-0.26$ \\
120 & 1.91 & \textbf{1.50} & $-0.41$ & 240 & 1.63 & \textbf{1.38} & $-0.26$ \\
\bottomrule
\end{tabular}
\end{table}

\begin{figure}[H]
\centering
\begin{tikzpicture}
\begin{axis}[
  width=0.84\linewidth, height=0.35\linewidth,
  tick label style={font=\scriptsize, /pgf/number format/assume math mode=true},
  label style={font=\scriptsize},
  ymajorgrids, grid style={cGrid, line width=0.3pt},
  axis line style={cAxis!65, line width=0.45pt},
  tick style={cAxis!65, line width=0.45pt}, tick align=outside,
  every axis plot/.append style={line width=1pt, mark size=1.25pt},
  xlabel={training epoch}, ylabel={unguided FID-50k},
  xmin=15, xmax=248, xtick={20,60,100,140,180,220,240},
  ymin=1.2, ymax=3.9, ytick={1.5,2,2.5,3,3.5},
  legend style={font=\scriptsize, draw=none, fill=white,
    at={(0.98,0.96)}, anchor=north east}, legend cell align=left]
\addplot[cBase, mark=*] coordinates {
  (20,3.63) (40,2.67) (60,2.36) (80,2.13) (100,1.99) (120,1.91)
  (140,1.81) (160,1.80) (180,1.75) (200,1.72) (220,1.67) (240,1.63)};
\addlegendentry{matched baseline}
\addplot[cAGT, mark=*] coordinates {
  (20,3.52) (40,2.32) (60,2.00) (80,1.77) (100,1.60) (120,1.50)
  (140,1.51) (160,1.46) (180,1.49) (200,1.44) (220,1.41) (240,1.38)};
\addlegendentry{\swift{}}
\end{axis}
\end{tikzpicture}
\caption{\textbf{DINOv2-B training trajectories through epoch 240.}
The matched unguided FID curves are plotted every 20 epochs, using the
rounded values in \Cref{tab:long}. \swift{} remains below the baseline
throughout the displayed schedule.}
\label{fig:long-run}
\end{figure}
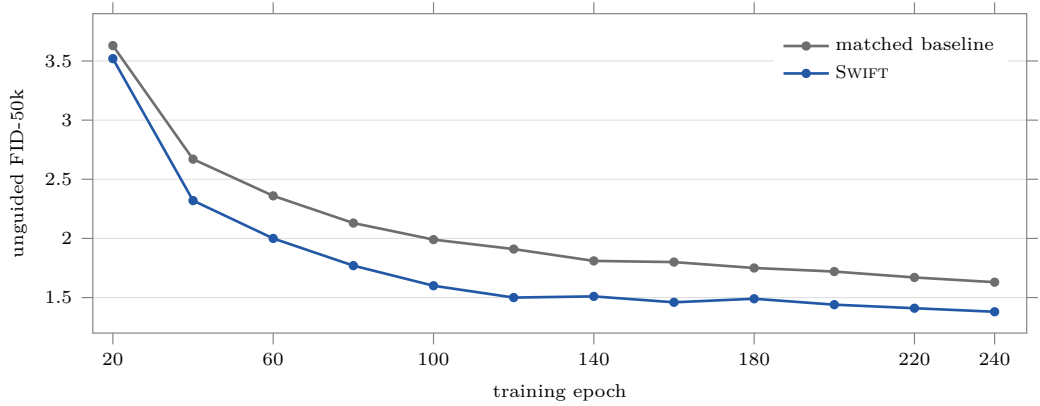

\FloatBarrier

\clearpage
\section{Additional CIFAR-10 experiments}
\label{app:cifar}
\label{app:cifar-full}
\label{app:campaign}
\label{app:cifar-width}

\paragraph{Experimental setting.}
Beyond the coefficient, guide-width, and guide-maturity ablations in
\Cref{sec:selection}, we tested the additional methods below.
We use class-conditional CIFAR-10 at $32\times32$, a DDPM$++$ student
with channel base 128, batch size 512, and 200k kimg of training.
Unless varied by the method, the guide is a separately trained, frozen
DDPM$++$ model with channel base 64 at 200k.
Evaluation uses FID-50k with the 18-step EDM sampler.

\begin{table}[H]
\centering\small
\setlength{\tabcolsep}{6pt}
\caption{\textbf{Additional methods and their best FIDs.} Each entry is the lowest FID at the common 200k endpoint across all completed settings and runs of that method, rather than a run average. Online adaptive rules are preliminary.}
\label{tab:cifar-summary}
\begin{tabular}{>{\raggedright\arraybackslash}p{0.75\linewidth}r}
\toprule
Method & Best FID \\
\midrule
Single-stage guide-capacity switch & 1.501 \\
Time profile: late peak & 1.507 \\
Combined time/noise coefficients & 1.549 \\
Recorded guide ahead of student training age & 1.557 \\
Time profile: late hold & 1.559 \\
Multi-stage guide-capacity schedule & 1.561 \\
Terminal-$\alpha$ changes (profiles or shared-prefix tails) & 1.562 \\
Normalized noise-dependent $\alpha$ tilts & 1.568 \\
NCSN$++$ guides (width 64/128) & 1.570 \\
Negative-$\alpha$ warm starts & 1.582 \\
Recorded same-age guide & 1.602 \\
Co-trained guide & 1.618 \\
Time profile: early/middle/tail & 1.619 \\
Recorded guide behind student training age & 1.621 \\
Noise-dependent $\alpha$ ramps & 1.622 \\
Random $\alpha\sim U[0,1]$ & 1.625 \\
Time profile: piecewise & 1.630 \\
ADM guides (width 96) & 1.639 \\
Time profile: middle only & 1.827 \\
Online adaptive $\alpha$ rules (preliminary) & 2.416 \\

\bottomrule
\end{tabular}
\end{table}

\paragraph{Guide variants.}
\label{app:cifar-guide-changes}
We replaced the DDPM$++$ guide with NCSN$++$ or ADM, co-trained the guide
alongside the student, or replayed separately trained guide checkpoints.
The recorded checkpoints either matched the student's training age or
were ahead/behind it; ahead variants require pre-trained future checkpoints.
Capacity schedules switched guide widths during training, either once
or through several stages, including $16\to32\to64$ and a further
width-96 stage.

\paragraph{Coefficient schedules.}
\label{app:negative}
We varied $\alpha$ over training using early/middle/tail and piecewise
profiles, a middle-only pulse, or late-onset peak/hold profiles.
We also tested a short negative-$\alpha$ prefix and different terminal
coefficients. Noise-dependent variants used ramps or normalized tilts
over noise level; other trials sampled $\alpha$ uniformly from $[0,1]$
or adjusted it online using short loss probes.

\paragraph{Combined methods.}
We multiplied a training-time coefficient profile by a noise-dependent
factor, allowing the extrapolation strength to vary jointly with training
progress and noise level. This tests whether the two adjustments can
provide complementary improvements.
\label{app:arms}
Full settings and all 156 run labels (152 scored), including extended-budget
and exploratory runs, remain in the accompanying \texttt{DATA} files.

\clearpage
\begingroup
\setlength{\intextsep}{5pt}
\captionsetup{skip=4pt}
\section{Additional samples}
\label{app:samples}

\Cref{fig:samples-animals,fig:samples-marine,fig:samples-objects,fig:samples-food}
show samples from \swift{} on RAEv2/DINOv3-L, the model with guided FID
$0.99$ in \Cref{tab:main-cfg}, for 32 ImageNet classes. As in
\Cref{fig:swift-imagenet-samples}, samples are generated with retuned internal
guidance and a 50-step Euler ODE sampler. Each block shows 28 samples of one class;
the class index and name are given below the block.

\begin{figure}[H]
\centering
\includegraphics[width=\linewidth]{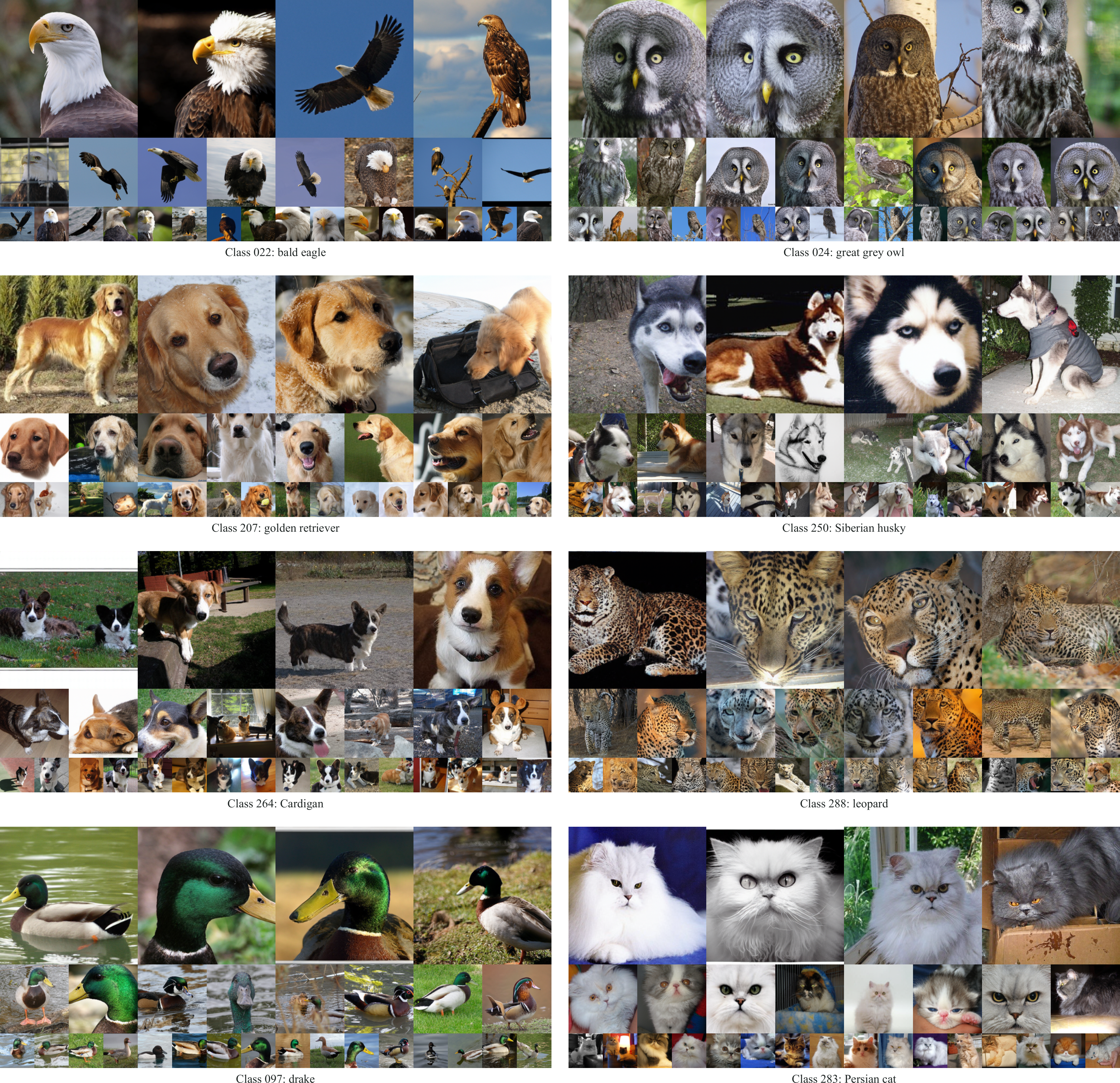}
\caption{\textbf{Samples from \swift{} on RAEv2/DINOv3-L: birds and mammals.}
Internal guidance, 50-step Euler ODE sampler; 28 samples per class.}
\label{fig:samples-animals}
\end{figure}

\clearpage
\begin{figure}[H]
\centering
\includegraphics[width=\linewidth]{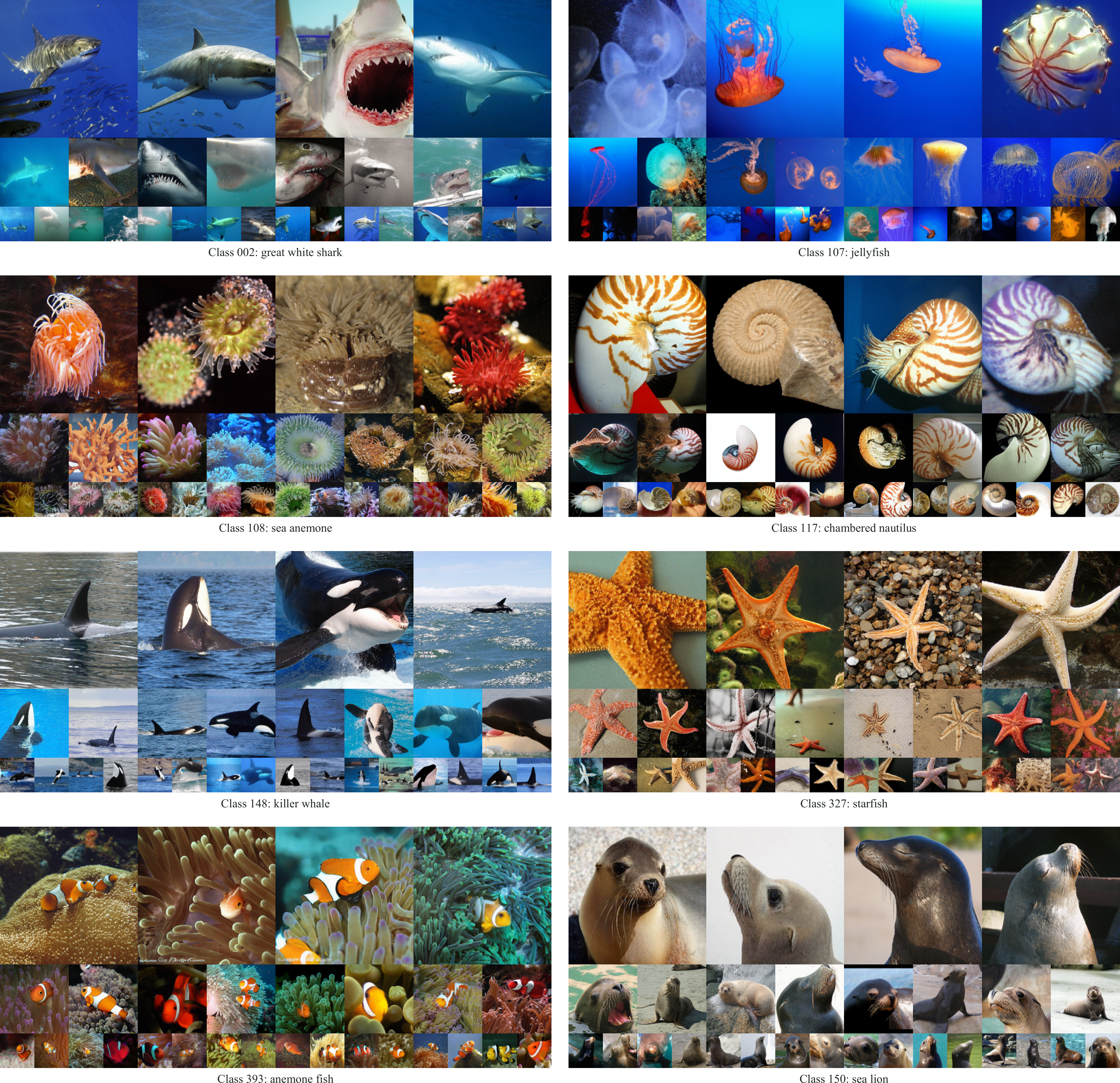}
\caption{\textbf{Samples from \swift{} on RAEv2/DINOv3-L: marine life.}
Internal guidance, 50-step Euler ODE sampler; 28 samples per class.}
\label{fig:samples-marine}
\end{figure}

\clearpage
\begin{figure}[H]
\centering
\includegraphics[width=\linewidth]{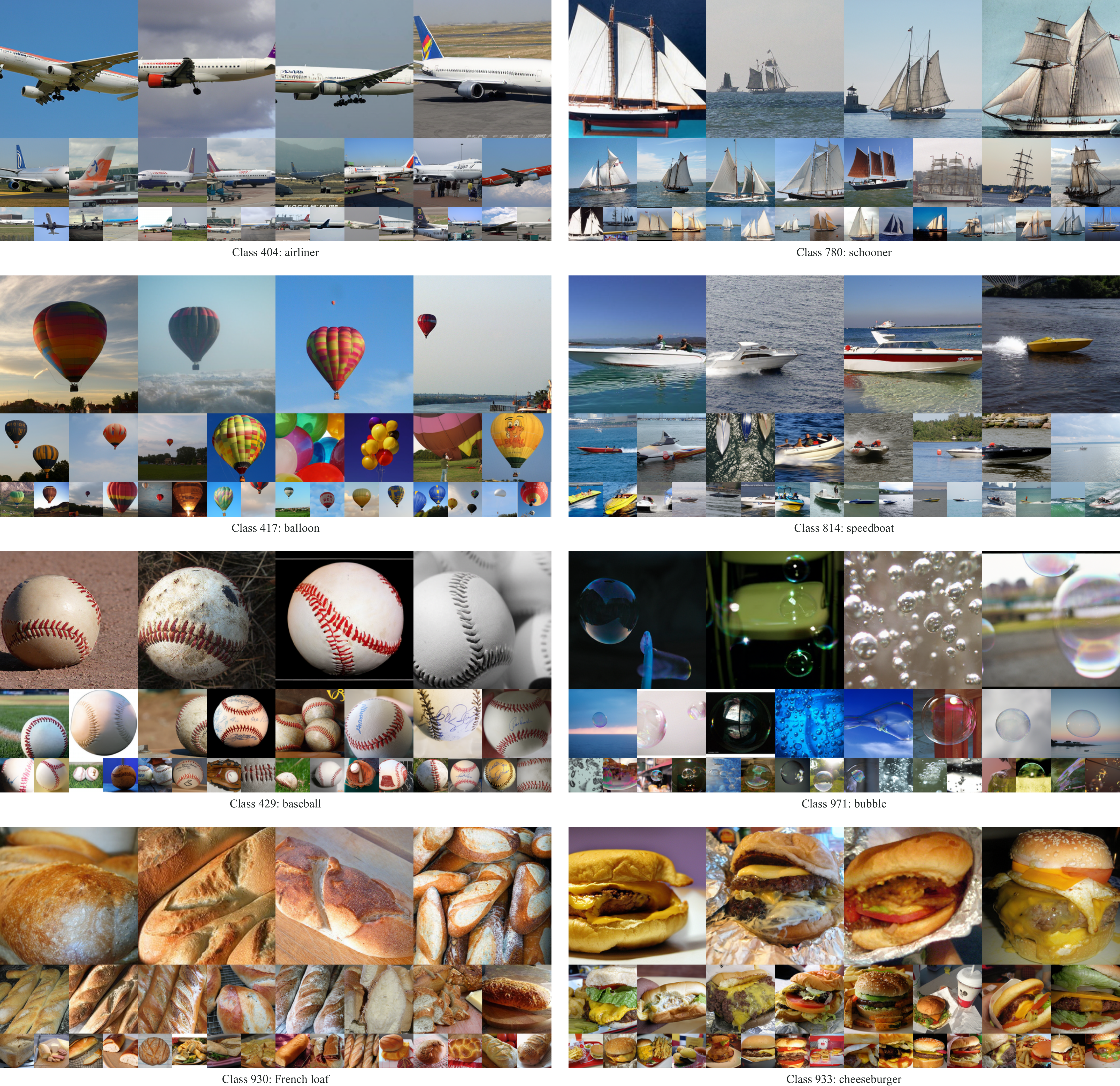}
\caption{\textbf{Samples from \swift{} on RAEv2/DINOv3-L: vehicles, objects, and bread.}
Internal guidance, 50-step Euler ODE sampler; 28 samples per class.}
\label{fig:samples-objects}
\end{figure}

\clearpage
\begin{figure}[H]
\centering
\includegraphics[width=\linewidth]{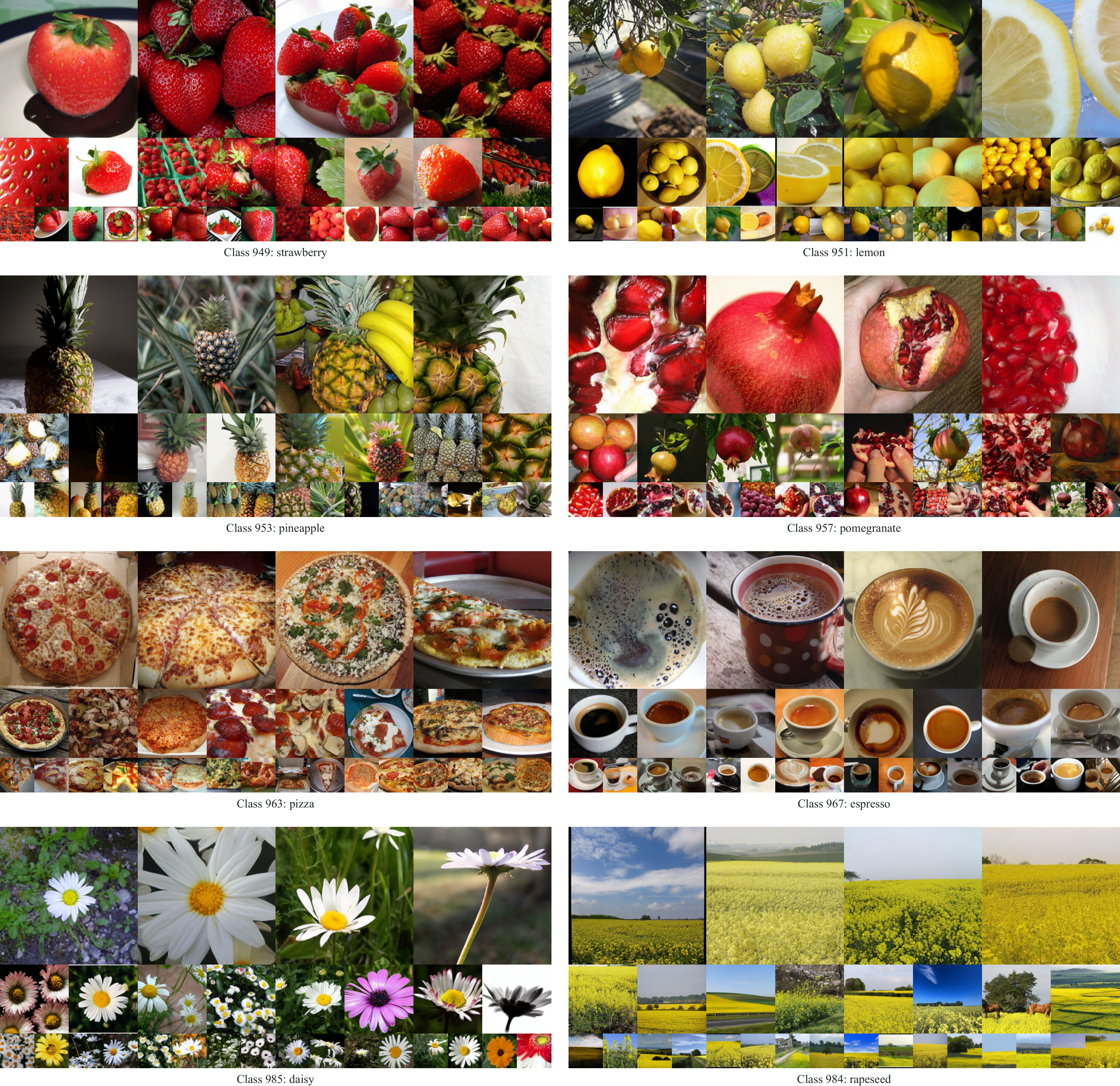}
\caption{\textbf{Samples from \swift{} on RAEv2/DINOv3-L: fruit, food, and plants.}
Internal guidance, 50-step Euler ODE sampler; 28 samples per class.}
\label{fig:samples-food}
\end{figure}
\endgroup

\end{document}